\documentclass{article}
\usepackage{iclr2027_conference,times}

\iclrfinalcopy 

\usepackage{amsmath,amsfonts,bm}

\def\eqref#1{equation~\ref{#1}}

\def\1{\bm{1}}

\DeclareMathAlphabet{\mathsfit}{\encodingdefault}{\sfdefault}{m}{sl}
\SetMathAlphabet{\mathsfit}{bold}{\encodingdefault}{\sfdefault}{bx}{n}

\newcommand{\E}{\mathbb{E}}

\newcommand{\R}{\mathbb{R}}

\newcommand{\Cov}{\mathrm{Cov}}

\usepackage[utf8]{inputenc}
\usepackage[T1]{fontenc}
\usepackage{xcolor}
\usepackage{url}
\usepackage{graphicx}

\usepackage[normalem]{ulem}
\usepackage{amsthm}
\usepackage{amssymb}
\usepackage{booktabs}
\usepackage{array}
\usepackage{colortbl}
\usepackage{tabularx}
\usepackage{algorithm,algorithmic}
\usepackage{hyperref}
\hypersetup{hidelinks}
\newtheorem{proposition}{Proposition}

\newtheorem{corollary}{Corollary}
\newtheorem{assumption}{Assumption}
\theoremstyle{remark}
\newtheorem{remark}{Remark}
\newtheorem{innerprop}{Proposition}
\newenvironment{restateprop}[1]{\renewcommand{\theinnerprop}{#1}\innerprop}{\endinnerprop}
\newcommand{\tr}{\operatorname{tr}}
\newcommand{\T}{\mathsf{T}}

\definecolor{PALOriginalColor}{HTML}{A34854}
\definecolor{PALReplacementColor}{HTML}{00796B}
\newcommand{\PALoriginal}[1]{{\color{PALOriginalColor}#1}}
\newcommand{\PALalt}[1]{\item[\textcolor{PALReplacementColor}{\scriptsize\bfseries UAL}]\begingroup\normalfont\color{PALReplacementColor}#1\par\endgroup}
\newcommand{\PALlegend}{\par\noindent{\normalfont\footnotesize
\textcolor{PALOriginalColor}{Original operation}\enspace|\enspace
\textcolor{PALReplacementColor}{UAL replacement}\enspace|\enspace
\textcolor{black}{Shared steps}}\par}

\title{Diffusion Posterior Sampling with Uncertain Forward Models}

\author{Wentao Shangguan \& Vivek K Goyal  \\
Department of Electrical and Computer Engineering\\
Boston University\\
Boston, MA 02215, USA \\
\texttt{wentaos@bu.edu, v.goyal@ieee.org}
}

\begin{document}

\maketitle

\lhead{} 
\renewcommand{\headrulewidth}{0pt}

\begin{abstract}
Diffusion posterior samplers reconstruct signals by combining a pretrained diffusion prior with a measurement likelihood specified by a forward model and a noise model.
Calibration and reference data can characterize uncertainty in the forward model's parameters, such as blur kernels or detector responses.
Fixing these parameters at estimated values leaves their uncertainty out of the likelihood, and parameter errors are thus absorbed into the reconstruction.
We introduce an uncertainty-aware likelihood (UAL), a Gaussian approximation to the marginal likelihood obtained by integrating over the uncertain parameters.
For each candidate signal, UAL propagates parameter uncertainty into a structured measurement covariance, reducing the quadratic penalty  along residual directions that parameter variation can explain.
Existing samplers can use UAL in their data updates while
retaining their pretrained diffusion priors and sampling schedules.
Under a fixed Gaussian nuisance model, we bound likelihood error from linearization and average divergence between the exact marginal and UAL posteriors.
We compare samplers with and without UAL across six inverse problems spanning Gaussian and motion deblurring, CT, MRI, ultrasound, and fluorescence microscopy, and find improvements in reconstruction and ensemble prediction.
UAL reduces geometric mean MSE by 9\% to 20\% across samplers in the main Gaussian-deblurring setting and by 51\% to 92\% in motion deblurring when non-blind samplers use the population-mean kernel.
In motion deblurring, DAPS and PnP-DM with UAL also give 6.2 to 6.4 dB higher ensemble-mean PSNR than published BlindDPS and Fast Diffusion EM.
On the ablation set, UAL outperforms trace-matched isotropic inflation for PnP-DM in both deblurring tasks; in motion deblurring, it also exceeds the best-tested fixed scalar data weighting by 3.3 dB for PnP-DM and 4.2 dB for DAPS in average PSNR.
\end{abstract}

\edef\UALsavedTextFloatSep{\the\textfloatsep}
\setlength{\textfloatsep}{12pt plus 2pt minus 2pt}
\begin{figure}[t]
\setlength{\abovecaptionskip}{4pt}
\centering
\includegraphics[width=\linewidth]{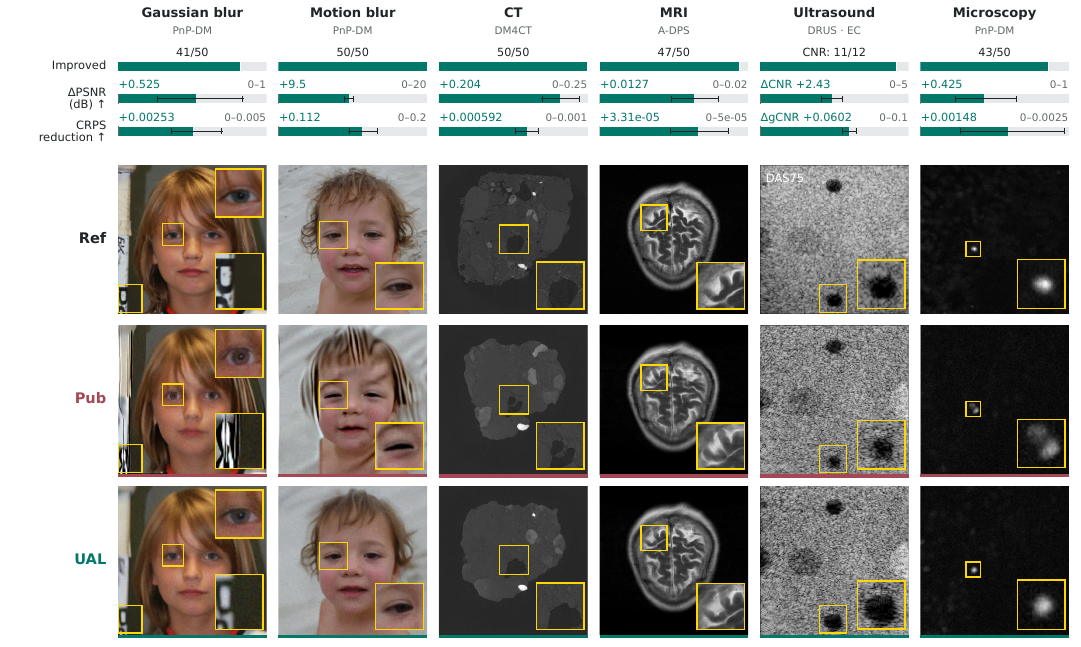}
\caption{Reconstruction examples and quantitative effects of UAL across six inverse problems. Reconstruction rows show selected individual trajectories; bars summarize each evaluation set on task-specific scales.
 Ultrasound reports CNR and gCNR computed within each reconstruction on the contrast phantom, with one SE across its 12 paired trajectories; the delay-and-sum reference (DAS75) is shown. Other tasks report PSNR and CRPS with paired-bootstrap whiskers.}
\label{fig:teaser} 
\end{figure}

\section{Introduction}
Diffusion posterior samplers reconstruct signals by combining a pretrained diffusion prior~\citep{ho2020ddpm,song2021sde} with a measurement likelihood specified by a forward model and an observation noise model~\citep{chung2023dps,wu2024pnpdm}.
 Parameters in the forward model, such as a blur kernel or
detector response, can remain uncertain after calibration.
We treat these quantities as nuisance parameters:
they affect the measurements but estimating them is not an explicit goal. Calibration and reference data provide nuisance parameter estimates and can also describe their uncertainty. Fixing these quantities at one estimate, however, leaves that uncertainty out of the reconstruction likelihood.
A reconstruction process can then distort the signal to fit components of residuals caused by nuisance parameter error.

In motion deblurring, for example, a kernel bank describes variation in blur direction and extent across
exposures. For a candidate image and a fixed kernel, the residual between the observed and predicted
measurements combines image error, kernel error, and observation noise.
A point likelihood, evaluated at one fixed kernel, penalizes this residual using only the
observation noise model. For each candidate image, however, the kernel statistics and forward model predict correlated measurement changes caused by kernel variation.
Residuals along these predicted directions provide weaker evidence of image error than residuals with the same observation-noise penalty that kernel variation cannot explain, so they should receive smaller quadratic penalties.
The same effect arises in CT, where uncertainty in the flat and dark reference fields used for calibration changes the corrected projections even for fixed acquired data~\citep{shi2026dm4ct}.

Non-blind diffusion posterior samplers condition on a fixed measurement
model~\citep{chung2023dps,kawar2022ddrm,zhang2025daps}, whereas blind methods sample or estimate the nuisance parameters together with the signal~\citep{chung2023blinddps,murata2023gibbsddrm,wang2024dmplug}. Some blind samplers, such as BlindDPS, train an additional diffusion prior for the nuisance parameters~\citep{chung2023blinddps}.

We instead marginalize forward-model parameters using an uncertainty model specified by calibration and reference information. Let $x$ denote the
unknown signal, $y$ the measurements, $\psi$ the nuisance parameters, and $q$ a fixed uncertainty
distribution for $\psi$.
The marginal likelihood is
\[
  p_{\mathrm{marg}}(y\mid x)=\E_{\psi\sim q}\!\left[p(y\mid x,\psi)\right].
\]
We introduce an uncertainty-aware likelihood (UAL), a Gaussian approximation to the nuisance-marginal likelihood that allows this direction-dependent weighting to be incorporated into existing diffusion posterior samplers.
At a candidate signal, UAL linearizes the forward model in the nuisance parameters and uses their mean and covariance to construct the measurement likelihood. UAL is exact for affine nuisance dependence with Gaussian nuisance parameters and independent Gaussian observation noise; for non-Gaussian $q$, UAL uses a Gaussian with the mean and covariance of the linearized measurement model (Section~\ref{sec:marginal}).

For a fixed Gaussian nuisance model, we bound likelihood error from linearization and the resulting posterior divergence averaged over measurements
from the model (Corollary~\ref{cor:posterior-error}).
In a fixed-sensitivity linear model with known signal covariance and matched nuisance moments, the linear estimator using UAL
minimizes mean squared reconstruction error among linear estimators, averaged over signals,
nuisance parameters, and observation noise.
 At a fixed nuisance realization,
the point-likelihood estimator can instead have lower error.
For a single nuisance parameter and distinct estimators, we give the crossover distance
from the nuisance mean in closed form (Appendix~\ref{app:linear-risk}).

UAL incorporates calibration and reference statistics into existing samplers\textquoteright{} data updates while retaining their pretrained priors and sampling schedules.
 We evaluate it with non-blind and blind diffusion samplers on synthetic and acquired measurements: Gaussian and motion deblurring, CT, MRI, ultrasound, and fluorescence microscopy.
A simpler alternative inflates the observation-noise variance, which in effect weakens the data term.
For PnP-DM, the best tested scalar inflation gives mean per-trajectory PSNR of 27.03 dB in Gaussian deblurring and 23.37 dB in motion deblurring. UAL adds variance only along directions that parameter variation can explain and gives 27.48 and 26.64 dB, respectively (Section~\ref{sec:exp-ablations}).
UAL also improves ensemble prediction: it lowers the mean CRPS, energy, and local variogram scores for every sampler and application in Table~\ref{tab:exp-deblurring}.  In Gaussian deblurring, UAL more closely matches the numerically marginalized likelihood and image score than the point and isotropic alternatives; covariance controls show the reconstruction benefit of structured, candidate-dependent weighting (Section~\ref{sec:exp-ablations}).
Applying UAL requires nuisance means and covariances and a forward model for propagating their variation; no additional diffusion-model training is needed.
Covariance factorization reduces evaluation to forming and solving a system in nuisance coordinates.
In the configurations of Table~\ref{tab:runtime}, UAL takes at most three times the published sampler's wall time per trajectory, except DAPS in motion deblurring at 7.1 times (Appendix~\ref{app:runtime}).

\section{Background and Related Work}
\label{sec:background}

\paragraph{Diffusion priors for inverse problems.}

For fixed nuisance parameters $\psi$ and a signal prior $p_0(x)$ that does not depend on $\psi$, Bayes' rule
gives $p(x\mid y,\psi)\propto p(y\mid x,\psi)p_0(x)$, where $p(y\mid x,\psi)$ is the measurement
likelihood. A pretrained diffusion model provides the prior through its denoisers or
scores~\citep{ho2020ddpm,song2021sde}.
 The likelihood enters diffusion reconstruction through three common data-update forms: gradient guidance, data-fidelity optimization, and sampling-based updates. The likelihood at a noisy diffusion state $x_t$,
$p(y\mid x_t,\psi)$, averages over possible clean signals and is
generally intractable.
DPS-style gradient guidance approximates it using a
denoised estimate $\hat x_0(x_t)$ and differentiates through that estimate~\citep{chung2023dps}.
$\Pi$GDM propagates Gaussian approximations of clean-signal uncertainty through a fixed measurement operator to approximate the noisy-state likelihood~\citep{song2023pseudoinverse}.
Data-fidelity optimization solves a subproblem between denoising steps~\citep{zhu2023diffpir}.
Sampling-based updates use the measurement in spectral coordinates~\citep{kawar2022ddrm}, in a split Gibbs
sampler~\citep{wu2024pnpdm}, or in annealed conditionals~\citep{zhang2025daps}. Diffusion priors with fixed-operator measurement models are also used in
MRI~\citep{aali2025adps}, ultrasound~\citep{10.1007/978-3-031-53767-7_19},
and CT~\citep{shi2026dm4ct}.
In these fixed-operator formulations, the likelihood conditions on a single nuisance-parameter value. For plug-and-play unadjusted Langevin sampling, \citet{renaud2024mismatched} bound posterior error under mismatched measurement models and denoisers.
UAL instead accounts for uncertainty in the forward-model parameters and enters each sampler through its existing data update (Section~\ref{sec:seam}).

\paragraph{Blind inverse problems.}

Blind reconstruction treats the nuisance parameters as additional inference variables. Joint
samplers place a prior on them and sample them with the signal: BlindDPS couples an image diffusion
model and an operator diffusion model through a shared measurement
likelihood~\citep{chung2023blinddps}; GibbsDDRM alternates approximate posterior updates under a
generic operator prior~\citep{murata2023gibbsddrm}; and PRISM pairs a measurement-conditioned kernel
prior with an unconditional image prior in a split Gibbs sampler~\citep{hu2026prism}. Estimation
methods fit them instead: Fast Diffusion EM updates the operator in a maximization step from image
samples, with denoiser-based kernel regularization~\citep{laroche2024fastem}; and DMPlug optimizes the
input to a pretrained diffusion generator jointly with the operator
parameters~\citep{wang2024dmplug}. Kernel Diffusion samples a conditional kernel distribution and
uses a non-blind method to reconstruct the image~\citep{sanghvi2024kerneldiff}. Under the same joint
model, exact joint inference and exact marginalization give the same signal posterior; the methods
differ in how nuisance uncertainty reaches the signal update. UAL approximates the marginal
measurement likelihood from a specified uncertainty model, without training an additional diffusion
model for the nuisance parameters.
 Appendix~\ref{app:seams} gives the image and operator updates for the evaluated samplers.

\paragraph{\texorpdfstring{Calibration and reference information.}{Calibration and reference information.}}

We use calibration data and reference samples to specify the means and covariances of uncertain forward-model parameters.
Repeated CT flat and dark frames characterize variation in calibration fields~\citep{shi2026dm4ct}, and principal-component representations of flat fields support dynamic projection normalization~\citep{vannieuwenhove2015flatfields}.
A motion-kernel bank describes variation in blur across exposures.
ESPIRiT derives coil sensitivity maps from calibration data and permits multiple maps when a single-map model is insufficient~\citep{uecker2014espirit}.
Differences between the sound speeds provided with CUBDL acquisitions and their optimized values set the uncertainty scale for our PICMUS phantom reconstructions~\citep{hyun2021cubdl,liebgott2016picmus}.
We represent the parameter variation by a distribution $q$ and use its mean and covariance to construct UAL (Section~\ref{sec:setting}).

\paragraph{Marginalization and UAL.}

Errors-in-variables models account for measurement error in the model's input variables~\citep{fuller2009measurement}.
\citet{wiesel2008gaussian} derive maximum-likelihood estimators for linear regression with Gaussian uncertainty in the model matrix, and \citet{beck2010structured} extend them to structured perturbations.
\citet{kumar2013location} linearize sampling-location errors to obtain a Gaussian likelihood
with signal-dependent variance, then update the variance weights in an iterative linear estimator.
Using Bayesian approximation-error modeling~\citep{kaipio2007statistical,kaipio2013approximate,nissinen2008approximation}, \citet{mozumder2016} approximately marginalize
uncertain optical coefficients with an enhanced Gaussian error model whose moments are estimated
from prior samples and held fixed during reconstruction.
For partially known random blurs, covariance weighting gives the classical LMMSE restoration filter~\citep{galatsanos2000hierarchical}. UAL adapts this Gaussian treatment of forward-model uncertainty to diffusion posterior sampling through existing samplers' data updates, and we test it against numerical marginalization and scalar variance inflation.

\section{Method}

\subsection{\texorpdfstring{Forward model and UAL construction}{Forward model and UAL construction}}
\label{sec:setting}
\label{sec:marginal}

Let $x\in\R^n$ denote the unknown signal (an image in our applications),
$y\in\R^m$ the measurement, and
$\psi\in\R^P$ the nuisance parameters. We write the measurement model as
\begin{equation}
  y=\mathcal A(x,\psi)+e,\qquad e\sim\mathcal N(0,{\Sigma_e}),\qquad {\Sigma_e}\succ0,
  \label{eq:model}
\end{equation}
where $\mathcal A$ is the forward model and $e$ is zero-mean Gaussian observation noise with covariance $\Sigma_e\in\R^{m\times m}$.
We represent complex measurements by stacking their real and imaginary parts.

Let $q$ be the nuisance uncertainty model used for reconstruction, with mean $\mu=\E_q\psi$ and covariance $\Sigma_\psi=\Cov_q\psi\succeq0$.
This distribution can describe variation in a reference population or uncertainty in calibrated parameters; in controlled experiments, it can be specified directly.
We hold the nuisance model fixed and model $x$ and $\psi$ as independent; the measurement noise is independent of both.

The marginal likelihood and its signal posterior are
\begin{equation}
  p_{\mathrm{marg}}(y\mid x)
  =\E_{\psi\sim q}\!\left[\mathcal N\!\left(y;\mathcal A(x,\psi),{\Sigma_e}\right)\right],
  \qquad
  p_{\mathrm{marg}}(x\mid y)\propto p_0(x)\,p_{\mathrm{marg}}(y\mid x),
  \label{eq:physical-target}
\end{equation}
where $p_0$ is the signal prior. The expectation also covers a discrete
reference population. By comparison, a point likelihood evaluates
$\mathcal N(y;\mathcal A(x,\psi_0),{\Sigma_e})$ at a single parameter value $\psi_0$.
UAL approximates the expectation in (\ref{eq:physical-target}), while keeping
the signal prior fixed.

For each candidate signal $x$, we choose a reference $\psi_0$ and linearize the forward model in its nuisance parameters. With sensitivity $J_\psi(x)$ and remainder $r_x(\psi)$, we write
\begin{equation}
  \mathcal A(x,\psi)
  =\mathcal A(x,\psi_0)+J_\psi(x)(\psi-\psi_0)+r_x(\psi),
  \qquad
  J_\psi(x)=D_\psi\mathcal A(x,\psi_0).
  \label{eq:expansion}
\end{equation}
Dropping $r_x(\psi)$ and taking the mean and covariance over $\psi\sim q$ and the independent noise $e$ gives
\begin{equation}
  \bar y(x)=\mathcal A(x,\psi_0)+J_\psi(x)(\mu-\psi_0),
  \qquad
  S(x)={\Sigma_e}+J_\psi(x){\Sigma_\psi} J_\psi(x)^\T .
  \label{eq:S}
\end{equation}
Nuisance uncertainty thus induces structured measurement uncertainty that may depend on $x$.

UAL is the Gaussian density
$p_{\mathrm{UAL}}(y\mid x;\psi_0,q)=\mathcal N(y;\bar y(x),S(x))$
for this fixed local model.
The mean correction $J_\psi(x)(\mu-\psi_0)$ accounts for the difference
between the reference and the nuisance mean (Remark~\ref{rem:offset}).
It vanishes at $\psi_0=\mu$.

\begin{proposition}[Exact affine-Gaussian marginal]
\label{prop:marginal}
At fixed $x$, suppose $\mathcal A(x,\psi)$ is affine in $\psi$, $q=\mathcal N(\mu,\Sigma_\psi)$, and $e$ is independent of $\psi$. Then $p_{\rm marg}(y\mid x)=\mathcal N(y;\bar y(x),S(x))$, with moments given by (\ref{eq:S}).
For nonlinear dependence on $\psi$, the same moments are exact for the linearized forward model in (\ref{eq:expansion}).
\end{proposition}
For non-Gaussian $q$, UAL approximates the measurement likelihood by a Gaussian with mean $\bar y(x)$ and covariance $S(x)$ from (\ref{eq:S}).
Appendix~\ref{app:proofs} gives the proofs and bounds on moment errors
from linearization and on covariance truncation.

\subsection{\texorpdfstring{UAL evaluation and sampler integration}{UAL evaluation and sampler integration}}
\label{sec:pal-evaluation}
\label{sec:seam}

\paragraph{\texorpdfstring{Energy and score.}{Energy and score.}}

Define the measurement residual $d(x)=y-\bar y(x)$, the precision matrix $S(x)^{-1}$, and the precision-weighted residual $w(x)=S(x)^{-1}d(x)$.
The normalized negative log likelihood is
\begin{equation}
  \ell_{\mathrm{UAL}}(x;y)
  =\frac12 d(x)^\T S(x)^{-1}d(x)
   +\frac12\log\det S(x)+\frac m2\log(2\pi).
  \label{eq:pal-energy}
\end{equation}
Increasing measurement variance along a residual direction lowers its quadratic penalty, while the log determinant grows with the corresponding increase in covariance volume.

For the whitened residual $\Sigma_e^{-1/2}d$, the quadratic weight along a left singular direction of $\Sigma_e^{-1/2}J_\psi F$ is $1/(1+s_i^2)$, where $FF^\T=\Sigma_\psi$ and $s_i$ is the corresponding singular value. Here $s_i^2$ is the added variance along that whitened direction.
Directions orthogonal to the whitened columns keep unit weight.
Residuals with the same observation-noise penalty can therefore receive different
UAL quadratic penalties, according to their compatibility with the predicted nuisance
variation. Scalar covariance inflation, by contrast, assigns equal weights in noise-whitened coordinates; with $x$ and $\bar y(x)$ fixed, it has a 
positive minimum KL divergence from UAL unless the nuisance-induced covariance is
proportional to the observation-noise covariance
(Remark~\ref{rem:scalar-inflation}).

For differentiable $\bar y(x)$ and $S(x)$, with the nuisance model and reference fixed, the likelihood score is given componentwise by
\begin{equation}
  \partial_{x_i}\log p_{\mathrm{UAL}}(y\mid x)
  =(\partial_{x_i}\bar y)^\T w
   +\frac12 w^\T(\partial_{x_i}S)w
   -\frac12\tr\!\left(S^{-1}\partial_{x_i}S\right).
  \label{eq:pal-score}
\end{equation}
The first term differentiates the predicted mean, and the remaining terms account for covariance weighting and normalization.

\paragraph{\texorpdfstring{Efficient evaluation.}{Efficient evaluation.}}
Factoring the nuisance covariance reduces UAL evaluation to an $r\times r$ system, avoiding an explicit inverse of the $m\times m$ measurement covariance.
Write ${\Sigma_\psi}=FF^\T$, where $r$ is the number of columns of
$F\in\R^{P\times r}$ and $\Sigma_\psi$ may be singular, and set $U(x)=J_\psi(x)F\in\R^{m\times r}$.
Define the $r\times r$ matrix $G(x)=I_r+U(x)^\T\Sigma_e^{-1}U(x)$.
The reduced solve gives
\[
 Ga_{\rm red}=U^\T\Sigma_e^{-1}d,\qquad
 w=\Sigma_e^{-1}(d-Ua_{\rm red}).
\]
Under the local Gaussian model at fixed $x$, $Ua_{\rm red}$ is the mean nuisance contribution conditional on the residual $d$.
The remaining residual is then weighted by the observation-noise precision $\Sigma_e^{-1}$.
The normalization follows from $\log\det S=\log\det\Sigma_e+\log\det G$.
Appendix~\ref{app:schur} derives the factorized energy and gives Algorithm~\ref{alg:pal}.
The cost also includes evaluating the forward model and its parameter sensitivities and forming $G$.  
Keeping only leading covariance modes reduces the solve dimension $r$. With the predicted mean $\bar y(x)$ unchanged, Corollary~\ref{cor:covariance-truncation} bounds the resulting likelihood error.
Appendix~\ref{app:runtime} reports recorded reconstruction times.

\paragraph{Use in diffusion samplers.}
Across applications, UAL treats forward-model uncertainty in the same way:
it propagates parameter variation into a measurement covariance.
We incorporate UAL through each sampler\textquoteright s published data update~\citep{zhang2025daps,wu2024pnpdm}, retaining its diffusion prior and sampling schedule.
Appendix Table~\ref{tab:point-pal} compares the point likelihood and UAL\@.
Appendix~\ref{app:seams} specifies the parameter models, reference values, and sampler updates.

\subsection{\texorpdfstring{Approximation error and reconstruction risk}{Approximation error and reconstruction risk}}
\label{sec:approximation-scope}

We bound the likelihood error introduced by nuisance linearization and the resulting average divergence between the exact marginal and UAL posteriors.
For nonlinear nuisance dependence, likelihood error is controlled by the remainder \(r_x(\psi)\) in (\ref{eq:expansion}), measured in observation-noise units.

\begin{proposition}[Likelihood error from linearization]
\label{prop:kl}
Fix $x$ and $\psi_0$, let $q=\mathcal N(\mu,{\Sigma_\psi})$, and assume
$\E_q[r_x(\psi)^\T{\Sigma_e^{-1}}r_x(\psi)]<\infty$. Then
\begin{equation}
  D_{\mathrm{KL}}\big(p_{\mathrm{marg}}(\cdot\mid x)\,\big\Vert\,p_{\mathrm{UAL}}(\cdot\mid x;\psi_0,q)\big)
  \le\tfrac12\E_q\!\left[r_x^\T{\Sigma_e^{-1}}r_x\right].
  \label{eq:kl}
\end{equation}
\end{proposition}

Appendix~\ref{app:kl-proof} gives the proof.
Under the conditions of Corollary~\ref{cor:posterior-error}, the prior-averaged bound in (\ref{eq:kl}) also bounds posterior divergence averaged over marginal-model measurements. These bounds cover the likelihood approximation and its induced posterior, excluding the sampler updates.

\paragraph{\texorpdfstring{Reconstruction-risk analysis.}{Reconstruction-risk analysis.}}
In the linear model of Appendix~\ref{app:linear-risk}, with nuisance sensitivity held fixed as the signal varies, the linear estimator based on UAL minimizes mean squared reconstruction error among linear estimators for any nuisance distribution whose mean and covariance match the assumed ones (Proposition~\ref{prop:risk}).
For a fixed parameter realization near the mean, the point estimator can have lower error; Appendix~\ref{app:linear-risk} characterizes the crossover and the conditions for positive average gain when these moments are inaccurate.
\section{Experiments}
\label{sec:experiments}

\subsection{Experimental protocol}

\label{sec:evaluated}
We evaluate UAL at two levels: reconstruction and ensemble prediction in paired sampler comparisons (Section~\ref{sec:exp-deblurring}), and, at shared candidate images, likelihood and image-score accuracy against numerical nuisance marginalization (Section~\ref{sec:exp-ablations}).
Additional controls examine covariance structure, recomputation and differentiation, uncertainty level, and retained rank.

Published denotes the released sampler and UAL its uncertainty-aware version; each pair uses the same measurement data, diffusion checkpoint, sampling schedule, step sizes, and reconstruction seeds.
Deblurring, CT, MRI, and microscopy use 50 cases, and ultrasound uses the acquired PICMUS contrast and resolution phantoms provided with DRUS; each case has 12 reconstruction seeds. The fixed deblurring ablation set contains 16 images, with 12 paired reconstruction seeds per image.
Appendix~\ref{app:seams} gives forward models, nuisance models, and sampler updates.

We report PSNR and SSIM both as the mean over trajectories and for the ensemble-mean reconstruction.
 CRPS, energy, and local variogram scores (lower is better) use the fair finite-ensemble estimators of Appendix~\ref{app:metrics}, which also gives the scoring coordinates and interval construction; ultrasound uses native contrast and resolution metrics.

\subsection{\texorpdfstring{Reconstruction across samplers and applications}{Reconstruction across samplers and applications}}
\label{sec:exp-deblurring}
\label{sec:exp-applications}
\label{sec:alternatives}
We compare the Published and UAL versions of DAPS, PnP-DM, DiffPIR, and BlindDPS for Gaussian and motion deblurring on FFHQ images~\citep{karras2019ffhq} (Table~\ref{tab:exp-deblurring}).
The Gaussian setting uses kernel widths with mean $\bar s=3$ pixels and standard deviation $\tau=0.5$ pixels. 
In motion deblurring, each non-blind pair uses the population-mean kernel, giving the two likelihoods the same mean response at a common candidate image.
Motion deblurring also includes Fast Diffusion EM~\citep{laroche2024fastem}, a blind sampler that estimates a kernel for each image. Its UAL image update uses the population-mean kernel instead, with covariance accounting for image and kernel uncertainty and their interaction.
We estimate the kernel covariance and retain its leading 256 eigenmodes (\(P=256\)), with this rank fixed before the covariance-rank sweep (Section~\ref{sec:cov-struc}).

\begin{table}[!t]
\setlength{\abovecaptionskip}{4pt}
\centering
\small
\color{black}
\setlength{\tabcolsep}{7pt}
\renewcommand{\arraystretch}{0.96}
\caption[Reconstruction and ensemble prediction]{Reconstruction and ensemble prediction. T/M: mean per-trajectory score / score of the ensemble-mean reconstruction. PSNR is in dB; CRPS, energy, and LVS are fair ensemble scores, with LVS scaled by $10^3$. Bold marks better UAL values.}
\label{tab:exp-deblurring}
\label{tab:exp-applications}
\label{tab:ensemble-scores}
\begin{tabular}{llrrrrr}
\toprule
Sampler / task & Version & PSNR T/M $\uparrow$ & SSIM T/M $\uparrow$
& CRPS $\downarrow$ & Energy $\downarrow$ & LVS $\downarrow$ \\
\midrule
\multicolumn{7}{l}{Gaussian blur} \\
 & Published & 29.49 / 30.60 & .7416 / .8143 & .02871 & 20.128 & 11.240 \\
\rowcolor{black!6}
\cellcolor{white}\raisebox{0.48\normalbaselineskip}[0pt][0pt]{DAPS} & UAL & \textbf{29.88} / \textbf{30.94} & \textbf{.7468} / \textbf{.8171} & \textbf{.02796} & \textbf{19.449} & \textbf{11.021} \\
\addlinespace[1pt]
 & Published & 26.97 / 29.53 & .5603 / .7616 & .03209 & 20.874 & 15.286 \\
\rowcolor{black!6}
\cellcolor{white}\raisebox{0.48\normalbaselineskip}[0pt][0pt]{PnP-DM} & UAL & \textbf{27.50} / \textbf{30.35} & \textbf{.5721} / \textbf{.7776} & \textbf{.02956} & \textbf{18.803} & \textbf{14.252} \\
\addlinespace[1pt]
 & Published & 29.13 / 30.28 & .7687 / .8180 & .02948 & 21.263 & 10.788 \\
\rowcolor{black!6}
\cellcolor{white}\raisebox{0.48\normalbaselineskip}[0pt][0pt]{DiffPIR} & UAL & \textbf{30.12} / \textbf{31.30} & \textbf{.7860} / \textbf{.8331} & \textbf{.02583} & \textbf{18.325} & \textbf{9.353} \\
\addlinespace[1pt]
 & Published & 28.31 / 30.30 & .7343 / .8002 & .02698 & 19.322 & 9.739 \\
\rowcolor{black!6}
\cellcolor{white}\raisebox{0.48\normalbaselineskip}[0pt][0pt]{BlindDPS} & UAL & \textbf{29.21} / \textbf{31.08} & \textbf{.7708} / \textbf{.8278} & \textbf{.02476} & \textbf{17.778} & \textbf{8.895} \\
\midrule
\multicolumn{7}{l}{Motion blur} \\
 & Published & 17.68 / 18.10 & .3303 / .3779 & .15043 & 96.013 & 57.076 \\
\rowcolor{black!6}
\cellcolor{white}\raisebox{0.48\normalbaselineskip}[0pt][0pt]{DAPS} & UAL & \textbf{28.84} / \textbf{29.58} & \textbf{.7103} / \textbf{.7726} & \textbf{.03502} & \textbf{23.388} & \textbf{12.774} \\
\addlinespace[1pt]
 & Published & 17.04 / 17.90 & .2766 / .3897 & .14485 & 91.587 & 61.032 \\
\rowcolor{black!6}
\cellcolor{white}\raisebox{0.48\normalbaselineskip}[0pt][0pt]{PnP-DM} & UAL & \textbf{26.54} / \textbf{29.43} & \textbf{.5552} / \textbf{.7628} & \textbf{.03236} & \textbf{20.943} & \textbf{14.959} \\
\addlinespace[1pt]
 & Published & 11.44 / 12.17 & .0701 / .0877 & .34406 & 172.270 & 179.800 \\
\rowcolor{black!6}
\cellcolor{white}\raisebox{0.48\normalbaselineskip}[0pt][0pt]{DiffPIR} & UAL & \textbf{14.55} / \textbf{15.19} & \textbf{.1440} / \textbf{.1666} & \textbf{.24688} & \textbf{130.857} & \textbf{100.688} \\
\addlinespace[1pt]
 & Published & 22.12 / 23.21 & .5559 / .6263 & .06392 & 46.633 & 18.291 \\
\rowcolor{black!6}
\cellcolor{white}\raisebox{0.48\normalbaselineskip}[0pt][0pt]{BlindDPS} & UAL & 21.95 / \textbf{23.33} & .5449 / \textbf{.6278} & \textbf{.06218} & \textbf{44.589} & \textbf{17.339} \\
\addlinespace[1pt]
 & Published & 22.76 / 23.27 & .5236 / .5525 & .07789 & 51.865 & 24.271 \\
\rowcolor{black!6}
\cellcolor{white}\raisebox{0.48\normalbaselineskip}[0pt][0pt]{Fast Diffusion EM} & UAL & \textbf{29.10} / \textbf{30.72} & \textbf{.7635} / \textbf{.8108} & \textbf{.02708} & \textbf{18.332} & \textbf{9.633} \\
\midrule
\multicolumn{7}{l}{Imaging applications} \\
 & Published & 25.94 / 27.33 & .4248 / .5195 & .02194 & 23.464 & 5.361 \\
\rowcolor{black!6}
\cellcolor{white}\raisebox{0.48\normalbaselineskip}[0pt][0pt]{CT} & UAL & \textbf{26.15} / \textbf{27.54} & \textbf{.4321} / \textbf{.5299} & \textbf{.02135} & \textbf{22.898} & \textbf{5.355} \\
\addlinespace[1pt]
 & Published & 30.77 / 31.88 & .8435 / .8179 & .02007 & 11.562 & 5.817 \\
\rowcolor{black!6}
\cellcolor{white}\raisebox{0.48\normalbaselineskip}[0pt][0pt]{MRI} & UAL & \textbf{30.78} / \textbf{31.89} & \textbf{.8438} / \textbf{.8180} & \textbf{.02004} & \textbf{11.540} & \textbf{5.806} \\
\addlinespace[1pt]
 & Published & 29.33 / 33.67 & .5352 / .7912 & .02265 & 8.383 & 10.621 \\
\rowcolor{black!6}
\cellcolor{white}\raisebox{0.48\normalbaselineskip}[0pt][0pt]{Microscopy} & UAL & \textbf{29.76} / \textbf{34.22} & \textbf{.5505} / \textbf{.7971} & \textbf{.02117} & \textbf{7.482} & \textbf{9.740} \\
\bottomrule
\end{tabular}
\end{table}

In Table~\ref{tab:exp-deblurring}, UAL improves ensemble-mean PSNR and SSIM and lowers the mean CRPS, energy, and local variogram scores in every deblurring comparison.
In motion deblurring, DAPS and PnP-DM with UAL exceed the published blind samplers BlindDPS and Fast Diffusion EM, which sample or estimate a kernel for each image, by 6.2 to 6.4 dB in ensemble-mean PSNR, with paired 95\% intervals excluding zero (Table~\ref{tab:ual-vs-blind}). For mean per-trajectory PSNR, UAL gives 28.84 dB for DAPS and 26.54 dB for PnP-DM, compared with 30.65 and 27.65 dB when the published samplers use the true kernel (Table~\ref{tab:pnpdm-motion-oracle}).

UAL also improves reconstruction and ensemble prediction in CT, MRI, and fluorescence microscopy (Figure~\ref{fig:teaser}). Mean per-trajectory PSNR increases on all 50 CT detector rows, 47 of 50 MRI scans, and 43 of 50 microscopy sections.
For CT, the gain persists when both versions use the same mean flat and dark reference fields (Appendix~\ref{app:additional-results}).
For MRI under the Gaussian coil-sensitivity model, UAL gives the exact marginal measurement likelihood.
The MRI reconstruction gains are small but consistent: all seven reported metrics improve, with pointwise 95\% paired-bootstrap intervals excluding zero.
On the acquired PICMUS ultrasound contrast phantom, UAL increases CNR from 9.929 to 12.363 dB and gCNR from 0.9141 to 0.9743\@.

\subsection{\texorpdfstring{Likelihood validation and controls}{Likelihood validation and controls}}
\label{sec:exp-ablations}
\edef\UALsavedInTextSep{\the\intextsep}
\setlength{\textfloatsep}{12pt}
\setlength{\intextsep}{8pt}

\paragraph{Likelihood approximation.}
For Gaussian deblurring, we evaluate the published PnP-DM point likelihood, UAL, and isotropic inflation at common final reconstructions produced by the published sampler. Numerical nuisance marginalization provides the reference.
 The isotropic control retains the UAL mean and scales the observation-noise covariance to $v_\star(x)\Sigma_e$, where $v_\star(x)=\tr(\Sigma_e^{-1}S(x))/m$ is the average measurement variance under UAL after noise whitening.
 The isotropic and UAL covariances therefore have the same whitened trace.
  Among Gaussian likelihoods with the same mean and covariance $v\Sigma_e$, this trace match minimizes the KL divergence from UAL (Remark~\ref{rem:scalar-inflation}).
At final reconstructions, UAL's mean relative image-score error is below 1\% at every tested uncertainty level, compared with 6\% to 15\% for the point likelihood. Its mean absolute negative-log-likelihood error is also the lowest of the three likelihoods (Figure~\ref{fig:likelihood-approximation}a,b).
During sampling, UAL also has the lowest mean relative image-score error at each of the five saved states and each uncertainty level. 
At later states, its mean error is at most 2\%, compared with approximately 6\% to 19\% for the point likelihood (Figure~\ref{fig:likelihood-approximation}c).
Across these common candidates, UAL more closely approximates both the marginalized log likelihood and its image score than the two alternatives.
\begin{figure}[!t]
\setlength{\abovecaptionskip}{4pt}
\centering
\includegraphics[width=\linewidth,trim=0 9bp 0 4bp,clip]{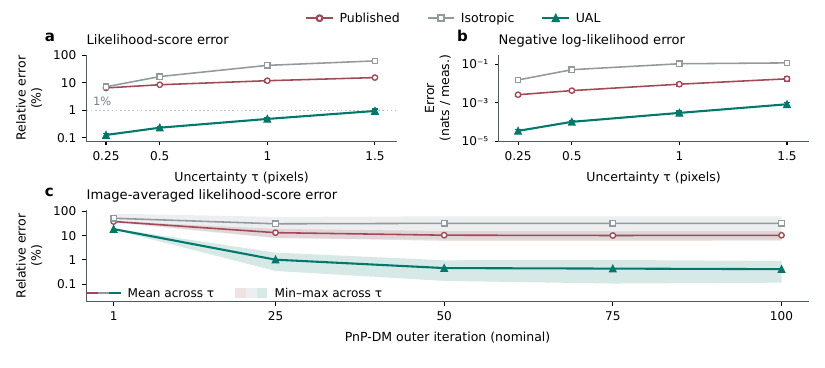}
\caption[Likelihood accuracy]{Likelihood accuracy in Gaussian deblurring against numerical marginalization. (a,b) Final-state relative image-score error and absolute negative-log-likelihood error per measurement: means $\pm$ one SE across the ablation set. (c) Score error at five saved states, using one trajectory per image. Errors are averaged over images within each uncertainty level. Lines and bands show the mean and range across $\tau=0.25,0.5,1,1.5$ pixels. Generating and assumed width uncertainty vary together. Appendix~\ref{app:numerical-reference} gives the stored indices corresponding to panel (c)'s nominal iterations.}
\label{fig:likelihood-approximation} 
\end{figure}

\paragraph{Covariance structure.}
\label{sec:cov-struc}
We next test whether covariance structure improves reconstruction and ensemble prediction beyond increasing total variance, using the three likelihoods in deblurring with PnP-DM\@.
 UAL and the isotropic control both include Gaussian normalization and differentiate their signal-dependent covariance.
 UAL's structured covariance increases per-trajectory PSNR and reduces CRPS relative to trace-matched isotropic inflation in both tasks (Table~\ref{tab:exp-ablations}).
 We also test fixed noise-variance multipliers and, for each task, report the one with the highest mean per-trajectory PSNR on the same ablation set used for comparison.
The best tested values give 27.029 dB ($v=3$) in Gaussian and 23.369 dB ($v=48$) in motion deblurring, compared with 27.475 and 26.641 dB for UAL (Appendix~\ref{app:covariance-control}). For DAPS in motion deblurring, the best tested multiplier of its data-term variance ($v=48$) gives 24.413 dB mean per-trajectory PSNR, compared with 28.631 dB for UAL on the ablation set.
At the final UAL reconstructions in Gaussian deblurring with PnP-DM, nuisance variation adds about 20\% to the average whitened measurement variance, carried by the two width-sensitivity responses (Appendix~\ref{app:covariance-diagnostics}). Trace-matched isotropic inflation distributes the same added variance uniformly in noise-whitened coordinates.

\begin{table}[!ht]
\centering
\small
\setlength{\tabcolsep}{5pt}
\setlength{\abovecaptionskip}{4pt}
\setlength{\belowcaptionskip}{4pt}
\caption{Covariance structure at matched whitened trace on the ablation set. Appendix~\ref{app:covariance-control} reports paired differences and standard errors.
}
\label{tab:exp-ablations}
\begin{tabular*}{\linewidth}{@{\extracolsep{\fill}}lrrrrrr@{}}
\toprule
 & \multicolumn{3}{c}{Gaussian blur} & \multicolumn{3}{c}{Motion blur} \\
\cmidrule(lr){2-4}\cmidrule(lr){5-7}
Metric & Published & Isotropic & UAL & Published & Isotropic & UAL \\
\midrule
PSNR T (dB) $\uparrow$ & 26.944 & 27.020 & \textbf{27.475} & 17.872 & 24.084 & \textbf{26.641} \\
CRPS ($\times10^3$) $\downarrow$ & 32.53 & 32.42 & \textbf{29.81} & 131.53 & 58.55 & \textbf{31.49} \\
\bottomrule
\end{tabular*}
\end{table}
\paragraph{Covariance updates and derivatives.}
In Gaussian deblurring with PnP-DM, two further controls separate recomputing the covariance from differentiating it (Table~\ref{tab:covariance-updates}). Fixing the covariance at the published sampler's ensemble mean removes nearly all of the PSNR gain and raises CRPS above the published sampler's.
Recomputing it at each candidate keeps the gain whether or not its derivatives are propagated. By contrast, the motion-deblurring controls for Fast Diffusion EM keep the population-mean kernel fixed: differentiating the covariance raises mean per-trajectory PSNR from 23.80 to 28.49 dB, and including the log determinant gives 29.10 dB (Appendix~\ref{app:covariance-control}).

\begin{table}[!ht]
\centering
\small
\caption{Gaussian PnP-DM covariance-update controls on the ablation set.
The two recomputed treatments agree within two standard errors in per-trajectory PSNR.}
\label{tab:covariance-updates}
\begin{tabular}{@{}lrrrr@{}}
\toprule
Metric & Published & Fixed covariance & Derivatives stopped & UAL \\
\midrule
PSNR T (dB) $\uparrow$ & 26.944 & 26.964 & 27.506 & 27.475 \\
CRPS ($\times10^3$) $\downarrow$ & 32.53 & 34.84 & 29.76 & 29.81 \\
\bottomrule
\end{tabular}
\end{table}

\paragraph{Matched uncertainty.}
In Gaussian deblurring with PnP-DM, UAL's reconstruction advantage grows as the generating and modeled width uncertainty increase together (Figure~\ref{fig:matched-uncertainty}).
At $\tau=1.5$ pixels, UAL exceeds Published by 1.75 dB and
Isotropic by 0.47 dB in mean per-trajectory PSNR
(Table~\ref{tab:matched-uncertainty}).
In a separate comparison of three combinations of mean width and width uncertainty, UAL gives higher mean per-trajectory PSNR and lower CRPS for all four samplers (Table~\ref{tab:gaussian-widths}).

\paragraph{Covariance rank.}
\label{sec:cov_rank}
The rank sweep with PnP-DM shows that a compact covariance representation preserves reconstruction quality in motion deblurring (Figure~\ref{fig:rank-tradeoff}).
The 256-mode representation retains 96\% of the kernel variance. On the ablation set, it gives 26.641 dB mean per-trajectory PSNR, compared with 26.251 dB for 2048 modes. The corresponding mean wall times are 89.5 and 2171.3 seconds per trajectory, respectively.
\begin{figure}[!t]
\setlength{\abovecaptionskip}{4pt}
\begin{minipage}[t]{0.485\linewidth}
\vspace{0pt}
\centering
\includegraphics[width=\linewidth,trim=0 12bp 0 12bp,clip]{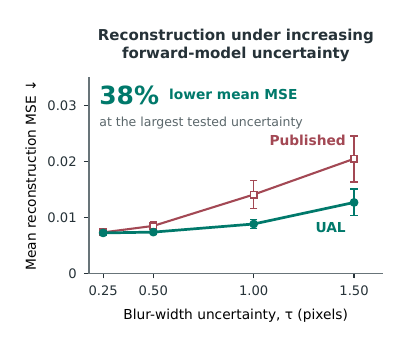}
\caption[Matched uncertainty]{Reconstruction error under matched Gaussian-width uncertainty with PnP-DM at mean width 3 pixels. Generating and modeled width standard deviations increase together. Mean per-trajectory MSE $\pm$ one SE across the ablation set.}
\label{fig:matched-uncertainty}
\end{minipage}\hfill
\begin{minipage}[t]{0.485\linewidth}
\vspace{0pt}
\centering
\includegraphics[width=\linewidth,trim=0 12bp 0 12bp,clip]{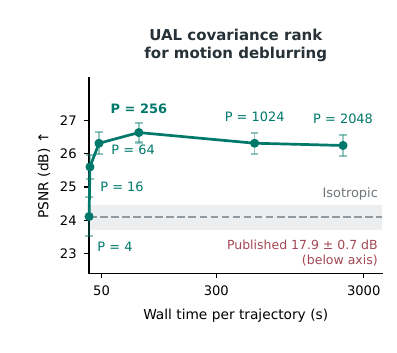}
\caption[Covariance rank and runtime]{Covariance rank and runtime in motion deblurring with PnP-DM, labeled by the number $P$ of retained kernel-covariance modes. Mean per-trajectory PSNR $\pm$ one SE across the ablation set; Isotropic is trace-matched at $P=256$.}
\label{fig:rank-tradeoff}
\end{minipage}
\end{figure}
\paragraph{\texorpdfstring{Sensitivity to nuisance moments.}{Sensitivity to nuisance moments.}}
With Gaussian-deblurring observations fixed, PnP-DM with UAL gives mean per-trajectory PSNR between 27.408 and 27.475 dB across the tested changes to either the parameter mean or covariance scale, up to approximately 20\%. The published sampler with the correct mean width gives 26.944 dB, below every tested UAL setting (Figure~\ref{fig:covariance-scale}).
Underestimating the mean width increases CRPS more than any of the tested covariance rescalings.
Motion deblurring is similarly insensitive to 20\% changes in covariance scale (Table~\ref{tab:covariance-scale}).

\section{Conclusion and limitations}

UAL uses forward-model uncertainty to weight measurement-residual directions according to how well parameter variation can explain them. It enters existing samplers through their data updates without additional diffusion-model training, and it improves reconstruction in CT, MRI, and fluorescence microscopy and contrast in ultrasound.
In controlled Gaussian and motion deblurring comparisons with PnP-DM, retaining covariance structure improves reconstruction beyond trace-matched isotropic inflation.
Accounting for forward-model uncertainty improves ensemble prediction: UAL lowers the mean CRPS, energy, and local variogram scores in all 12 comparisons in Table~\ref{tab:exp-deblurring}.
For PnP-DM in Gaussian deblurring, UAL has lower mean relative image-score error than the published and isotropic likelihoods at the final reconstructions and every evaluated sampling state, using numerical marginalization as the reference.
 UAL's accuracy and reconstruction benefit depend on the nuisance model and, for nonlinear nuisance dependence, on the local forward-model approximation. UAL adds computation that grows with the covariance rank and depends on the sampler (Appendix~\ref{app:runtime}). In deblurring, UAL keeps its gain under tested mean or covariance-scale errors of up to about 20\% (Figure~\ref{fig:covariance-scale}, Table~\ref{tab:covariance-scale}); a different motion-kernel population was not tested.

\setlength{\textfloatsep}{\UALsavedTextFloatSep}
\setlength{\intextsep}{\UALsavedInTextSep}
\bibliography{iclr2027_conference}
\bibliographystyle{iclr2027_conference}

\clearpage
\appendix
\makeatletter
\setlength{\@fptop}{0pt}
\setlength{\@fpsep}{14pt}
\setlength{\@fpbot}{0pt plus 1fil}
\makeatother
\section{UAL likelihood: derivations and approximation theory}
\label{app:proofs}

UAL is exact for an affine nuisance response with Gaussian nuisance
parameters and independent Gaussian observation noise.
We derive the likelihood score and factorized evaluation. We then bound likelihood, averaged posterior, and moment errors from nuisance linearization, followed by the likelihood error from covariance truncation.

\begin{table}[H]
\setlength{\abovecaptionskip}{4pt}
\centering
\color{black}
\small
\setlength{\tabcolsep}{4pt}
\caption[Point Gaussian likelihood and UAL]{Point likelihood and UAL at a fixed reference $\psi_0$.
Here the residuals are $d_0=y-\mathcal A(x,\psi_0)$ and $d=y-\bar y(x)$.
Both energies omit $m\log(2\pi)/2$; $\log\det\Sigma_e$ is constant in $x$.}
\label{tab:point-pal}
\renewcommand{\arraystretch}{1.15}
\begin{tabular*}{\linewidth}{@{\extracolsep{\fill}}lll@{}}
\toprule
Quantity & Point Gaussian likelihood & UAL \\
\midrule
Mean &
$\mathcal A(x,\psi_0)$ &
$\bar y(x)=\mathcal A(x,\psi_0)+J_\psi(x)(\mu-\psi_0)$ \\
\addlinespace[4pt]
Covariance &
$\Sigma_e$ &
$S(x)=\Sigma_e+J_\psi(x)\Sigma_\psi J_\psi(x)^\T$ \\
\addlinespace[4pt]
Negative log likelihood &
$\tfrac12 d_0^\T\Sigma_e^{-1}d_0+\tfrac12\log\det\Sigma_e$ &
$\tfrac12 d^\T S(x)^{-1}d+\tfrac12\log\det S(x)$ \\
\bottomrule
\end{tabular*}
\end{table}

\subsection{Affine Gaussian marginal and likelihood score}

\begin{restateprop}{\ref{prop:marginal}}[exact affine-Gaussian marginal, restated]
At fixed $x$, suppose $\mathcal A(x,\psi)$ is affine in $\psi$, $q=\mathcal N(\mu,\Sigma_\psi)$, and $e\sim\mathcal N(0,\Sigma_e)$ is independent of $\psi$. Then $p_{\rm marg}(y\mid x)=\mathcal N(y;\bar y(x),S(x))$, with moments given by (\ref{eq:S}). For nonlinear dependence on $\psi$, (\ref{eq:S}) gives the exact mean and covariance of the linearized forward model obtained by dropping $r_x$ in (\ref{eq:expansion}). 
\end{restateprop}
\begin{proof}
At fixed $x$, the measurement is an affine function of the independent Gaussian parameters and noise, hence Gaussian. Its mean and covariance are $\bar y(x)$ and $S(x)$ in (\ref{eq:S}), which gives the stated marginal.

For nonlinear dependence on $\psi$, dropping the remainder in
(\ref{eq:expansion}) gives
\begin{equation}
 y=\mathcal A(x,\psi_0)+J_\psi(x)(\psi-\psi_0)+e.
 \label{eq:affine-surrogate}
\end{equation}
This response is affine in $\psi$, with offset
$\mathcal A(x,\psi_0)-J_\psi(x)\psi_0$.
The same moment calculation gives (\ref{eq:S}).
At $\psi_0=\mu$, Proposition~\ref{prop:linearization} bounds the mean
and covariance omitted by dropping the remainder.
\end{proof}

\paragraph{\texorpdfstring{Likelihood score.}{Likelihood score.}}
Differentiating the normalized UAL energy in
(\ref{eq:pal-energy}) and negating its derivative gives the likelihood score, with $q$ and $\psi_0$ held fixed.
Suppose that $\bar y(x)$ and $S(x)$ are differentiable. Write $\partial_i=\partial/\partial x_i$ and
$w=S^{-1}d$, with $d=y-\bar y$. The required identities are
\begin{equation}
 \begin{aligned}
 \partial_i d&=-\partial_i\bar y,\\
 \partial_i S^{-1}&=-S^{-1}(\partial_i S)S^{-1},\\
 \partial_i\log\det S&=\tr(S^{-1}\partial_i S).
 \end{aligned}
 \label{eq:score-identities}
\end{equation}
Applying them to (\ref{eq:pal-energy}) gives
\begin{equation}
 \partial_i\ell_{\mathrm{UAL}}
 =-(\partial_i\bar y)^\T w
   -\frac12w^\T(\partial_iS)w
   +\frac12\tr(S^{-1}\partial_iS).
 \label{eq:pal-energy-derivative}
\end{equation}
Negating this expression gives (\ref{eq:pal-score}).
When the covariance depends on the signal, the score accounts for changes
in both residual weighting and Gaussian normalization. Holding the covariance
fixed leaves the mean-response term.

\subsection{Factorized evaluation of UAL}
\label{app:schur}
The UAL covariance is $S={\Sigma_e}+UU^\T$, with $U$ and $G$
defined in Section~\ref{sec:pal-evaluation}.
A single factorization of the $r\times r$ matrix $G$ provides both $S^{-1}d$ and the determinant contribution $\log\det G$.

The Woodbury identity and matrix determinant lemma give
\begin{align}
  S^{-1}
  &={\Sigma_e^{-1}}-{\Sigma_e^{-1}}UG^{-1}U^\T{\Sigma_e^{-1}},
  \label{eq:woodbury}\\
  \log\det S
  &=\log\det{\Sigma_e}+\log\det G.
  \label{eq:pal-logdet}
\end{align}
For $d=y-\bar y$, solve the reduced system
\begin{equation}
 Ga_{\rm red}=U^\T{\Sigma_e^{-1}}d,\qquad
 w={\Sigma_e^{-1}}(d-Ua_{\rm red})=S^{-1}d.
 \label{eq:pal-reduced-solve}
\end{equation}
In the equivalent local Gaussian model $d=Ua+e$, with $a\sim\mathcal N(0,I_r)$ independent of $e$, the solution $a_{\rm red}$ is the conditional mean of $a$ given $x$ and $y$.
Thus $Ua_{\rm red}$ is the conditional mean of the nuisance contribution to the residual.

With $\|v\|_{{\Sigma_e^{-1}}}^2=v^\T{\Sigma_e^{-1}}v$, the quadratic can
be evaluated as a sum of nonnegative terms, avoiding cancellation between large quadratic terms:
\begin{equation}
\begin{aligned}
 \|d-Ua_{\rm red}\|_{{\Sigma_e^{-1}}}^2+\|a_{\rm red}\|^2
 &=d^\T{\Sigma_e^{-1}}d-2a_{\rm red}^\T U^\T{\Sigma_e^{-1}}d+a_{\rm red}^\T Ga_{\rm red}\\
 &=d^\T{\Sigma_e^{-1}}d-a_{\rm red}^\T Ga_{\rm red}
 =d^\T S^{-1}d.
\end{aligned}
 \label{eq:schur-square}
\end{equation}
The second equality uses (\ref{eq:pal-reduced-solve}); the last uses
the Woodbury identity. Substitution into (\ref{eq:pal-energy}) gives
\begin{equation}
\begin{aligned}
 \ell_{\mathrm{UAL}}
 &=\tfrac12\|d-Ua_{\rm red}\|_{{\Sigma_e^{-1}}}^2+\tfrac12\|a_{\rm red}\|^2\\
 &\quad+\tfrac12\log\det{\Sigma_e}+\tfrac12\log\det G
       +\tfrac m2\log(2\pi).
\end{aligned}
 \label{eq:pal-factorized-energy}
\end{equation}
For fixed $x$, $a_{\rm red}$ minimizes $\tfrac12\|d-Ua\|_{\Sigma_e^{-1}}^2+\tfrac12\|a\|^2$, the sum of the observation-noise penalty and the standard Gaussian prior penalty. Equation~(\ref{eq:pal-factorized-energy}) therefore expresses the marginal energy as this penalized profile energy plus $\tfrac12\log\det G$ and terms independent of $x$. The log determinant distinguishes marginalizing the local nuisance coordinates from optimizing them. Table~\ref{tab:normalization-motion} compares this penalized profile energy with UAL under the same sampler and nuisance model.

\paragraph{Computational implementation.}
Since ${\Sigma_e}\succ0$, $G\succ0$ even when ${\Sigma_\psi}$ is singular.
A Cholesky factorization ${G=L_GL_G^\T}$ evaluates
(\ref{eq:pal-reduced-solve}) by triangular solves and gives
${\log\det G=2\sum_{i=1}^r\log (L_G)_{ii}}$.
For $r=1$, the solve is scalar division.
In our experiments, $\Sigma_e$ is a scaled identity, so applying $\Sigma_e^{-1}$ is a scalar division. For a general $\Sigma_e$, solve $\Sigma_e z=v$. For motion deblurring, DAPS and PnP-DM compute $G$ and $S^{-1}d$ in Fourier coordinates without storing $U$.
Explicit columnwise construction is a general alternative requiring $r$
applications of $J_\psi$.

Algorithm~\ref{alg:pal} evaluates the same UAL Gaussian likelihood through its factorized energy. Covariance truncation introduces a separate approximation; Corollary~\ref{cor:covariance-truncation} bounds its likelihood error when the mean is unchanged and the discarded covariance is positive semidefinite.

\begin{algorithm}[htbp]
\color{black}
\caption[UAL data term]{UAL data term}
\label{alg:pal}

\begin{algorithmic}[1]
\REQUIRE signal argument $x$, measurement $y$, reference $\psi_0$, mean $\mu$ of $q$,
factor $F$ with $\Sigma_\psi=FF^\T$, noise covariance $\Sigma_e$, forward model $\mathcal A$
and its parameter Jacobian $J_\psi$ at $\psi_0$
\STATE $\bar y \leftarrow \mathcal A(x,\psi_0)+J_\psi(x)(\mu-\psi_0)$; \quad $d \leftarrow y-\bar y$; \quad $U \leftarrow J_\psi(x)F$
\STATE Form $G\leftarrow I_r+U^\T\Sigma_e^{-1}U$ and compute its Cholesky factor $G=L_GL_G^\T$.
\STATE Solve for $a_{\rm red}$ in $L_GL_G^\T a_{\rm red}=U^\T\Sigma_e^{-1}d$ by forward and backward substitution.

\STATE $w \leftarrow \Sigma_e^{-1}(d-Ua_{\rm red})$ \COMMENT{$w=S^{-1}d$}
\STATE $\ell_{\mathrm{UAL}} \leftarrow \tfrac12\|d-Ua_{\rm red}\|_{\Sigma_e^{-1}}^2+\tfrac12\|a_{\rm red}\|^2
+\tfrac12\log\det\Sigma_e+\tfrac12\log\det G+\tfrac m2\log(2\pi)$
\ENSURE energy $\ell_{\mathrm{UAL}}$ and likelihood score by (\ref{eq:pal-score}),
differentiating the mean and covariance with $\mu$, $F$, and $\psi_0$ fixed
\end{algorithmic}

\end{algorithm}

\begin{remark}[Scalar covariance inflation]
\label{rem:scalar-inflation}

Fix $x$ and compare Gaussian likelihoods with the same mean $\bar y(x)$. For $v>0$, the 
equal-mean Gaussian KL divergence is
\begin{equation}
\begin{aligned}
 K(v)&=D_{\mathrm{KL}}\!\left(
 \mathcal N(\bar y,S)\,\Vert\,\mathcal N(\bar y,v\Sigma_e)\right)\\
 &=\tfrac12\left[
 \frac{m+\tr G-r}{v}-m+m\log v-\log\det G\right],
\end{aligned}
 \label{eq:scalar-covariance-kl}
\end{equation}
where $S=\Sigma_e+UU^\T$ and $G=I_r+U^\T\Sigma_e^{-1}U$.
Differentiation gives the unique minimizer and its remaining divergence:
\begin{equation}
\begin{aligned}
 v_\star&=1+\frac{\tr G-r}{m},\\
 K(v_\star)&=\tfrac12\left[m\log v_\star-\log\det G\right]\ge0.
\end{aligned}
 \label{eq:optimal-scalar-inflation}
\end{equation}
The optimal scalar matches the average variance after whitening by the observation-noise
covariance. At fixed whitened covariance trace and unchanged $\bar y(x)$, anisotropy in the whitened covariance leaves a positive minimum divergence; scalar inflation cannot reproduce its directional weighting.
Concavity of the logarithm, applied to the eigenvalues of
$I_m+\Sigma_e^{-1/2}UU^\T\Sigma_e^{-1/2}$, gives the inequality,
with equality exactly when $UU^\T=\beta\Sigma_e$ for some $\beta\ge0$.
Under the affine-Gaussian assumptions of Proposition~\ref{prop:marginal},
UAL is the exact marginalized likelihood in this comparison.
\end{remark}

\subsection{Likelihood and posterior approximation error}
\label{app:kl-proof}

The error from nuisance linearization is controlled by the response
it omits, measured in observation-noise units.
\begin{proof}[Proof of Proposition~\ref{prop:kl}]
Fix $x$, and denote the affine response in (\ref{eq:expansion}) by
$\widetilde{\mathcal A}_x(\psi)$.
Let $P_x$ and $Q_x$ denote the joint distributions of $(\psi,y)$
under the original and linearized forward models, respectively: 
\begin{equation}
 \begin{aligned}
 P_x(d\psi,dy)&=q(d\psi)\,\mathcal N(dy;\mathcal A(x,\psi),{\Sigma_e}),\\
 Q_x(d\psi,dy)&=q(d\psi)\,\mathcal N(dy;\widetilde{\mathcal A}_x(\psi),{\Sigma_e}).
 \end{aligned}
 \label{eq:coupled-models}
\end{equation}

They share the nuisance distribution $q$ and the conditional noise covariance.
The equal-covariance Gaussian KL formula gives
\begin{equation}
 D_{\mathrm{KL}}(P_x\Vert Q_x)
 =\frac12\E_q\!\left[r_x(\psi)^\T{\Sigma_e^{-1}}r_x(\psi)\right].
 \label{eq:coupled-kl}
\end{equation}
Let $p_{\mathrm{aff}}$ be the measurement marginal of $Q_x$.
Marginalizing $\psi$ cannot increase divergence, so
\begin{equation}
 D_{\mathrm{KL}}\!\left(p_{\mathrm{marg}}(\cdot\mid x)
                      \Vert p_{\mathrm{aff}}(\cdot\mid x)\right)
 \le D_{\mathrm{KL}}(P_x\Vert Q_x).
 \label{eq:kl-marginalization}
\end{equation}
For Gaussian $q$, Proposition~\ref{prop:marginal} identifies
$p_{\mathrm{aff}}$ with $p_{\mathrm{UAL}}$.
For non-Gaussian $q$, this argument bounds the divergence to the
affine mixture, before Gaussian moment matching.
\end{proof}
A small expected squared remainder in observation-noise units
therefore gives a small likelihood divergence.

Proposition~\ref{prop:kl} compares measurement likelihoods at a fixed
signal. The next result connects that approximation to the inference
target: with the same signal prior, the prior-averaged likelihood error
also bounds the posterior error averaged over measurements.

\begin{corollary}[Averaged posterior error]
\label{cor:posterior-error}
Use the same normalized signal prior $p_0$ in the exact marginal and UAL models.
Fix the nuisance model $q=\mathcal N(\mu,\Sigma_\psi)$ and reference $\psi_0$ independently of the evaluated observation.
 The UAL posterior is
$p_{\mathrm{UAL}}(x\mid y)\propto
p_0(x)p_{\mathrm{UAL}}(y\mid x;\psi_0,q)$.
If $\E_{x\sim p_0,\psi\sim q}
[r_x(\psi)^\T{\Sigma_e^{-1}}r_x(\psi)]<\infty$, then
\begin{equation}
\begin{aligned}
 &\E_{y\sim p_{\mathrm{marg}}(y)}
 D_{\mathrm{KL}}\!\left(p_{\mathrm{marg}}(\cdot\mid y)
 \,\Vert\,p_{\mathrm{UAL}}(\cdot\mid y)\right)\\
 &\qquad\le \tfrac12\E_{x\sim p_0,\psi\sim q}
 \!\left[r_x(\psi)^\T{\Sigma_e^{-1}}r_x(\psi)\right].
\end{aligned}
 \label{eq:posterior-kl}
\end{equation}
Here $p_{\mathrm{marg}}(y)$ is the measurement distribution obtained
by drawing $x\sim p_0$ and then $y\sim p_{\mathrm{marg}}(\cdot\mid x)$.

\end{corollary}

\begin{proof}
Let $P(x,y)=p_0(x)p_{\mathrm{marg}}(y\mid x)$ and
$Q(x,y)=p_0(x)p_{\mathrm{UAL}}(y\mid x)$.
Because the prior is common, the KL chain rule first gives
\begin{equation}
 D_{\mathrm{KL}}(P\Vert Q)
 =\E_{x\sim p_0}D_{\mathrm{KL}}\!\left(
 p_{\mathrm{marg}}(\cdot\mid x)\Vert p_{\mathrm{UAL}}(\cdot\mid x)\right).
 \label{eq:joint-likelihood-kl}
\end{equation}
Let $p_{\rm UAL}(y)$ be the measurement marginal of $Q$. Factoring the same joint laws in the other order gives
\begin{equation}
\begin{aligned}
 D_{\mathrm{KL}}(P\Vert Q)
 &=D_{\mathrm{KL}}\!\left(p_{\mathrm{marg}}(y)\Vert p_{\mathrm{UAL}}(y)\right)\\
 &\quad+\E_{y\sim p_{\mathrm{marg}}(y)}
 D_{\mathrm{KL}}\!\left(p_{\mathrm{marg}}(\cdot\mid y)
                    \Vert p_{\mathrm{UAL}}(\cdot\mid y)\right).
\end{aligned}
 \label{eq:joint-posterior-kl}
\end{equation}
The first term is nonnegative. Thus average posterior divergence is
at most joint divergence, which (\ref{eq:joint-likelihood-kl}) and
Proposition~\ref{prop:kl} bound by the right-hand side of
(\ref{eq:posterior-kl}).

\end{proof}

\paragraph{A quadratic remainder envelope.}

Under the hypotheses of Proposition~\ref{prop:kl}, suppose $L\ge0$ and $\|r_x(\psi)\|\le L\|\psi-\psi_0\|^2/2$ for $q$-almost every $\psi$. This curvature bound controls the likelihood error through the fourth nuisance moment:
\begin{equation}
 \frac12\E_q[r_x^\T{\Sigma_e^{-1}}r_x]
 \le \frac{L^2}{8\lambda_{\min}({\Sigma_e})}
       \E_q\|\psi-\psi_0\|^4.
 \label{eq:kl-envelope}
\end{equation}
This follows from $r_x(\psi)^\T{\Sigma_e^{-1}}r_x(\psi)\le\|r_x(\psi)\|^2/\lambda_{\min}({\Sigma_e})$
and the remainder envelope.
The fourth moment is finite for Gaussian $q$.
At $\psi_0=\mu$ with scalar ${\Sigma_\psi}=c$ and ${\Sigma_e}=\sigma^2I$,
$\E_q(\psi-\mu)^4=3c^2$ gives the bound $3L^2c^2/(8\sigma^2)$.
At fixed curvature bound and observation-noise variance, this bound grows
quadratically with the nuisance variance.

For the corresponding averaged posterior bound, allow the curvature bound to depend on the signal. Suppose additionally that
$\|r_x(\psi)\|\le L(x)\|\psi-\psi_0\|^2/2$
for $p_0\otimes q$-almost every $(x,\psi)$, where $L(x)\ge0$ is
measurable and
$\E_{x\sim p_0}[L(x)^2\E_q\|\psi-\psi_0\|^4]<\infty$.
Then the right-hand side of (\ref{eq:posterior-kl}) is at most
\begin{equation}
 \frac{\E_{x\sim p_0}\!\left[L(x)^2\E_q\|\psi-\psi_0\|^4\right]}
      {8\lambda_{\min}({\Sigma_e})}.
 \label{eq:posterior-envelope}
\end{equation}
Applying (\ref{eq:kl-envelope}) inside the prior expectation gives (\ref{eq:posterior-envelope}).

\subsection{Moment errors from nuisance linearization}
\label{sec:linearization}
Nuisance curvature can change both the marginal mean and covariance,
even when the first-order mean correction vanishes.
\begin{proposition}[Moment errors from linearization]
\label{prop:linearization}
At fixed $x$ and $\psi_0=\mu$, let $J=J_\psi(x)$ and suppose
$L\ge0$ and $\|r_x(\psi)\|\le L\|\psi-\mu\|^2/2$ on the support of $q$,
and the moments below are finite.
Writing $\delta=\psi-\mu$, the mean omitted by (\ref{eq:S}) is
$b_r=\E r_x$ with $\|b_r\|\le L\tr {\Sigma_\psi}/2$.
Let $\Delta S=\Cov(y\mid x)-S(x)$, where the conditional covariance retains the nonlinear remainder. Then
\begin{equation}
  \max\{|\tr\Delta S|,\|\Delta S\|_{\mathrm{op}}\}
  \le\frac{L^2}{4}\E\|\delta\|^4
       +L\,\E\!\left[\|\delta\|^2\|J\delta\|\right].
  \label{eq:linearization}
\end{equation}
For scalar Gaussian $\delta$ with variance $c$, writing $J=j$, the right-hand side is
$3L^2c^2/4+\sqrt{8/\pi}\,L\|j\|c^{3/2}$.
\end{proposition}

\begin{proof}
Write $\delta=\psi-\mu$. 
 The exact response is
$y=\mathcal A(x,\mu)+J\delta+r_x(\psi)+e$.
Its omitted mean is $b_r=\E r_x(\psi)$, so
$\|b_r\|\le\E\|r_x(\psi)\|\le L\E\|\delta\|^2/2=L\tr {\Sigma_\psi}/2$.
Since $\E\delta=0$ and the observation noise is independent,
\begin{equation}
 \Delta S=\Cov(r_x(\psi))+\E[J\delta r_x(\psi)^\T]+\E[r_x(\psi)\delta^\T J^\T].
 \label{eq:omitted-covariance}
\end{equation}
The covariance term is positive semidefinite, with
$\|\Cov(r_x(\psi))\|_{\mathrm{op}}\le\tr\Cov(r_x(\psi))\le\E\|r_x(\psi)\|^2$.
Each cross term has both absolute trace and operator norm at most
$\E[\|J\delta\|\|r_x(\psi)\|]$. Thus $|\tr\Delta S|$ and
$\|\Delta S\|_{\mathrm{op}}$ are each bounded by
$\E\|r_x(\psi)\|^2+2\E[\|J\delta\|\|r_x(\psi)\|]$.
Substituting the remainder bound gives (\ref{eq:linearization}).
For scalar Gaussian $\delta$, $\E\delta^4=3c^2$ and
$\E|\delta|^3=\sqrt{8/\pi}\,c^{3/2}$ give the stated constants.
\end{proof}
For scalar Gaussian variance $c$, the omitted mean is $O(c)$ and the covariance-error bound is $O(c^{3/2}+c^2)$ at fixed curvature and sensitivity. The nonlinear remainder can increase or decrease variance along a measurement direction through its cross terms with the linear response.

\begin{remark}[Phase responses: pulse-echo ultrasound]
\label{rem:ultrasound}
A phase response can have a nonlinear mean shift even when the first-order mean correction vanishes. For a complex sample $s$ perturbed to $se^{i\theta}$ with $\theta\sim\mathcal N(0,\sigma_\theta^2)$,
\begin{equation}
 \E[se^{i\theta}]=s e^{-\sigma_\theta^2/2}.
 \label{eq:phase-mean}
\end{equation}
Zero-mean Gaussian phase uncertainty attenuates the mean response by $e^{-\sigma_\theta^2/2}$; the first-order mean $s$ omits this attenuation.
\end{remark}

\subsection{Likelihood error from covariance truncation}

We next keep the UAL mean fixed and truncate only its covariance. For a positive-semidefinite discarded covariance, the following corollary bounds the divergence from UAL to the Gaussian likelihood with retained covariance $S_K$. The error is small when the discarded covariance is small after whitening by $S_K$.

\begin{corollary}[Covariance-truncation divergence]
\label{cor:covariance-truncation}
Fix $x$ and retain the UAL mean $\bar y(x)$.
Let $C_{y,K}\in\R^{m\times m}$ be the retained nuisance-induced
measurement covariance, satisfying
$0\preceq C_{y,K}\preceq J_\psi(x)\Sigma_\psi J_\psi(x)^\T$.
The truncated likelihood is
$p_{\mathrm{UAL},K}=\mathcal N(\bar y,S_K)$, with
$S_K=\Sigma_e+C_{y,K}$.
Its discarded covariance is $S-S_K$.
Writing this discarded covariance in whitened coordinates gives
$E_K=S_K^{-1/2}(S-S_K)S_K^{-1/2}$. Then

\begin{equation}
 D_{\mathrm{KL}}(p_{\mathrm{UAL}}\Vert p_{\mathrm{UAL},K})
 =\frac12[\tr E_K-\log\det(I+E_K)]
 \le\frac14\|E_K\|_F^2 .
 \label{eq:covariance-truncation-kl}
\end{equation}
This includes truncating the covariance factor, provided the discarded
covariance is positive semidefinite and the mean is unchanged.
\end{corollary}

\begin{proof}
The equal-mean Gaussian KL formula gives
$\tfrac12[\tr(S_K^{-1}S)-m-\log\det(S_K^{-1}S)]$.
The matrix $S_K^{-1/2}SS_K^{-1/2}=I+E_K$ has the same eigenvalues
as $S_K^{-1}S$, which gives the identity.
Let $\epsilon_i\ge0$ be the eigenvalues of $E_K$.
Applying $0\le a-\log(1+a)\le a^2/2$ to each eigenvalue gives
\begin{equation}
 \frac12\sum_i[\epsilon_i-\log(1+\epsilon_i)]
 \le\frac14\sum_i\epsilon_i^2
 =\frac14\|E_K\|_F^2.
 \label{eq:truncation-eigenvalues}
\end{equation}

\end{proof}
For an unchanged mean and positive-semidefinite covariance discard, the retained covariance determines which discarded directions contribute most to likelihood divergence.

\section{Fixed-sensitivity reconstruction analysis}
\label{app:linear-risk}
\label{sec:guarantees}

Under matched nuisance moments, the UAL linear estimator minimizes average
squared error among linear estimators; at a fixed nuisance realization its error
can exceed the point estimator's. We quantify this tradeoff and the reconstruction
cost of covariance truncation.

\subsection{Fixed-sensitivity model and estimator definitions}
\label{app:setting}

A fixed-sensitivity linear model lets us compute both filters' risks exactly and identify when covariance weighting helps. Holding the forward sensitivity and signal moments fixed isolates the reconstruction effect of changing the measurement-error covariance.

\begin{assumption}[Fixed-sensitivity linear model]
\label{as:frozen}
Let $A_0\in\R^{m\times n}$ be a fixed reference operator and
$J\in\R^{m\times P}$ a fixed nuisance sensitivity.
Set $\psi_0=\mu$ and consider
\[
 y=A_0x+J(\psi-\mu)+e.
\]
The signal is centered, with $\E x=0$ and $\Cov x=\Sigma_x$, and is
independent of $\psi$ and $e$. The uncertainty model $q$ has mean $\mu$
and covariance ${\Sigma_\psi}$. The noise is centered, independent of $\psi$, and
has covariance ${\Sigma_e}$.
Thus the effective noise $\varepsilon=J(\psi-\mu)+e$ has mean zero and covariance
$S={\Sigma_e}+J{\Sigma_\psi}J^\T$. For a response linear in the signal, $\mathcal A(x,\psi)=A(\psi)x$, take $A_0=A(\mu)$ and hold the local nuisance sensitivity $J$ fixed as $x$ varies.

\end{assumption}

The two estimators are $\widehat x=W_0y$ and $\widehat x=W_Sy$, with
\begin{equation}
  M_0=A_0\Sigma_x A_0^\T+{\Sigma_e},\quad M_S=M_0+J{\Sigma_\psi}J^\T,\qquad
  W_0=\Sigma_x A_0^\T M_0^{-1},\quad W_S=\Sigma_x A_0^\T M_S^{-1}.
  \label{eq:arms}
\end{equation}
The matrices $M_0$ and $M_S$ are the measurement covariances under the point likelihood and UAL, respectively; both estimators use the same signal covariance $\Sigma_x$. At $\psi_0=\mu$, the mean correction vanishes.
Define $\mathcal R(W)=\E\|Wy-x\|^2$, averaging over the signal,
nuisance parameters, and observation noise.

These moment assumptions suffice for the linear-risk comparisons.
Gaussianity is needed only where the linear estimator is identified
with the posterior mean. Conditional risk fixes the realized nuisance value $\psi^\star$ and averages over $x$ and $e$.

\begin{remark}[Nonzero prior mean or reference offset]
\label{rem:offset}
Let $m_x=\E x$, $\delta_\mu=\mu-\psi_0$, and consider the fixed-matrix model $y=A_{\rm ref}x+J(\psi-\psi_0)+e$, where $A_{\rm ref}=A(\psi_0)$. The estimator $Wy$ has mean error $a_W=(WA_{\rm ref}-I)m_x+WJ\delta_\mu$. Relative to the centered model with the same matrices, the average risk increases by $\|a_W\|^2$. Conditional on $\delta=\psi^\star-\mu$, the risk also includes the cross term $2a_W^\T WJ\delta$. The centered affine estimator
\[
 \widehat x=m_x+W[y-A_{\rm ref}m_x-J\delta_\mu]
\]
removes these terms. The shift $J\delta_\mu$ is the mean correction in (\ref{eq:S}). The results below use fixed matrices with $m_x=0$ and $\delta_\mu=0$.
\end{remark}

\subsection{Average risk and the optimal linear estimator}

Under the matched-moment model, the UAL filter uniquely minimizes average squared
error among linear filters. The squared-norm term in (\ref{eq:square}) gives the excess
risk of using a different filter.

\begin{proposition}[Average reconstruction error (LMMSE optimality)]
\label{prop:risk}
Under Assumption~\ref{as:frozen}, $W_S$ minimizes $\mathcal R(W)$ among
linear estimators, so $\mathcal R(W_S)\le\mathcal R(W_0)$.
When $x$, $\psi$, and $e$ are Gaussian, $W_Sy$ is also the posterior mean.
\end{proposition}

\begin{proof}
Under Assumption~\ref{as:frozen}, write $Wy-x=(WA_0-I)x+W\varepsilon$. Independence and zero means eliminate the cross term. Expanding the remaining traces gives, for any linear filter $W$,
\begin{equation}
  \mathcal R(W) = \tr\Sigma_x - 2\tr(WA_0\Sigma_x) + \tr\big(WM_SW^\T\big).
  \label{eq:risk}
\end{equation}
Since $W_SM_S=\Sigma_xA_0^\T$, completing the square in (\ref{eq:risk}) gives
\begin{equation}
 \begin{aligned}
 \mathcal R(W)&=\mathcal R(W_S)+\|(W-W_S)M_S^{1/2}\|_F^2,\\
 \mathcal R(W_S)&=\tr\Sigma_x
   -\tr(\Sigma_xA_0^\T M_S^{-1}A_0\Sigma_x).
 \end{aligned}
 \label{eq:square}
\end{equation}
The squared norm is the excess risk above $W_S$.
Because $M_S\succ0$, it vanishes only when $W=W_S$.
The minimizer is therefore unique even if $\Sigma_x$ is singular.

When $x$, $\psi$, and $e$ are Gaussian, $(x,y)$ is jointly Gaussian
with $\Cov(x,y)=\Sigma_xA_0^\T$ and $\Cov(y)=M_S$. Hence
\[
 \E[x\mid y]=\Sigma_xA_0^\T M_S^{-1}y=W_Sy.
\]
\end{proof}
The results below characterize the gain of $W_S$ and the crossover at a fixed nuisance realization.

\subsection{Average gain and its normalization}

The gain is the point estimator's nuisance-error contribution
multiplied by the fraction removed by covariance weighting.
\begin{proposition}[Average gain and its normalization]
\label{prop:excess}

Under Assumption~\ref{as:frozen}, use ${\Sigma_\psi}=FF^\T$ and $U=JF$
as in Section~\ref{sec:pal-evaluation}, with $J$ now fixed.
Define the dimensionless Gram matrix $\Lambda=U^\T M_0^{-1}U$.
The matrix $\Lambda$ uses the predictive covariance $M_0$; the likelihood-evaluation matrix $G$ uses the observation-noise covariance $\Sigma_e$. Then
\begin{equation}
  \mathcal R(W_0)-\mathcal R(W_S)
  =\tr\!\left(W_0{U}\Lambda(I+\Lambda)^{-1}{U^\T} W_0^\T\right)\ge0.
  \label{eq:excess}
\end{equation}
Let $H_0=\tr(W_0J{\Sigma_\psi}J^\T W_0^\T)$ be the point estimator\textquoteright s nuisance-error contribution. When $H_0>0$, define
$\kappa_0=[\mathcal R(W_0)-\mathcal R(W_S)]/H_0$.
For one nuisance parameter, ${\Sigma_\psi}=c$ and $J=j$,
\begin{equation}
  \kappa_0=\frac{\lambda}{1+\lambda},\qquad
  \lambda=c\,j^\T M_0^{-1}j.
  \label{eq:capture}
\end{equation}
The net average gain recovers a small fraction of $H_0$ at small $\lambda$
and nearly all of $H_0$ at large $\lambda$.
For multiple parameters, $\kappa_0$ is a nonnegative weighted average of
$\lambda_i/(1+\lambda_i)$ over the eigenvalues of $\Lambda$.
\end{proposition}

Scaling up the nuisance covariance by a common positive factor increases the fraction of the point estimator's nuisance error removed by covariance weighting. The absolute gain is $H_0\kappa_0$. 

If $H_0=0$, the estimators coincide and $\kappa_0$ is undefined.

\paragraph{Eigenvalue form.}

Let $\Lambda v_i=\lambda_i v_i$ with orthonormal $v_i$, and define
${\omega_i}=\|W_0{U}v_i\|^2$. For $H_0>0$,
\[
 \kappa_0=\frac{\sum_i{\omega_i}\,\lambda_i/(1+\lambda_i)}{\sum_i{\omega_i}},
 \qquad
 \frac{\lambda_{\min}}{1+\lambda_{\min}}
 \le\kappa_0\le
 \frac{\lambda_{\max}}{1+\lambda_{\max}}.
\]
The weights measure how much each factor direction contributes to
the point estimator's nuisance error.

\begin{proof}
Write $\mathcal R_0(W)=\E_{x,e}\|W(A_0x+e)-x\|^2$ for the
risk when the nuisance displacement is zero. For the estimator $W_0$,
the average risk separates into this zero-displacement risk and its
nuisance contribution:
\[
 {\mathcal R_0(W_0)}
 =\tr\Sigma_x-\tr(\Sigma_xA_0^\T M_0^{-1}A_0\Sigma_x),
 \qquad
 \mathcal R(W_0)={\mathcal R_0(W_0)}+H_0,
\]
where $H_0=\|W_0{U}\|_F^2$.
The Woodbury identity, expressed through $F$ and valid for singular ${\Sigma_\psi}$, gives
\[
 M_S^{-1}=M_0^{-1}
 -M_0^{-1}{U}(I+\Lambda)^{-1}{U^\T} M_0^{-1}.
\]
Substitution into (\ref{eq:square}) yields
\[
 \mathcal R(W_S)={\mathcal R_0(W_0)}
 +\tr\!\left(W_0{U}(I+\Lambda)^{-1}{U^\T} W_0^\T\right).
\]
Subtracting this expression from $\mathcal R(W_0)$ and using
$I-(I+\Lambda)^{-1}=\Lambda(I+\Lambda)^{-1}\succeq0$
proves (\ref{eq:excess}).

For one parameter, divide by $H_0=c\|W_0j\|^2>0$ to obtain
(\ref{eq:capture}). For multiple parameters, express both traces in
the eigenbasis of $\Lambda$: their terms are
${\omega_i}\lambda_i/(1+\lambda_i)$ and ${\omega_i}$.
If $H_0=0$, then $W_0{U}=0$ and the Woodbury formula gives $W_S=W_0$.

\end{proof}

\subsection{Conditional reconstruction error and the crossover}

At a fixed nuisance displacement, covariance weighting changes the nuisance-error contribution and incurs a filter-change cost.
\begin{proposition}[Multivariate conditional reconstruction error]
\label{prop:perdraw-general}
Under Assumption~\ref{as:frozen}, fix $\psi^\star=\mu+\delta$ and write
$b=J\delta$ for the realized nuisance term.
Let $\mathcal R_\delta(W)$ be the mean squared error averaged over $x$
and $e$ with $\delta$ fixed. For any linear $W$,
\begin{equation}
\begin{aligned}
  \mathcal R_\delta(W)
  &=\tr\!\left((WA_0-I)\Sigma_x(WA_0-I)^\T\right)
    +\tr(W{\Sigma_e} W^\T)+\|Wb\|^2,\\
  \mathcal R_\delta(W_0)-\mathcal R_\delta(W_S)
  &=\underbrace{\|W_0J\delta\|^2-\|W_SJ\delta\|^2}_{\text{difference in nuisance-error contributions}}-\underbrace{\tau_0({\Sigma_\psi})}_{\text{filter-change cost}},\\
  \tau_0({\Sigma_\psi})
  &=\tr\!\left((W_S-W_0)M_0(W_S-W_0)^\T\right)\ge0.
\end{aligned}
  \label{eq:perdraw}
\end{equation}
Moreover, $\tau_0({\Sigma_\psi})>0$ if and only if $W_0J{\Sigma_\psi}\ne0$.
\end{proposition}

\begin{proof}
For a fixed $b=J\delta$, write
\[
 Wy-x=(WA_0-I)x+We+Wb.
\]
The first two terms are centered and independent, while $Wb$ is
deterministic. Their squared norms add in expectation, giving the first
line of (\ref{eq:perdraw}). Completing the square in the part independent
of $b$ gives
\[
 \mathcal R_\delta(W)={\mathcal R_0(W_0)}
 +\|(W-W_0)M_0^{1/2}\|_F^2+\|Wb\|^2.
\]
Subtract the expressions for $W_0$ and $W_S$ to obtain the conditional
gain. The term $\tau_0({\Sigma_\psi})$ is the cost of changing the filter when the
nuisance contribution is absent. Since $M_0\succ0$, it vanishes exactly
when $W_S=W_0$.

Using $J{\Sigma_\psi}J^\T=UU^\T$, the resolvent identity gives
\[
 W_S-W_0=-W_0{U}{U}^\T M_S^{-1}.
\]
Thus $W_S=W_0$ exactly when $W_0{U}{U}^\T=0$, equivalently $W_0{U}=0$,
because $W_0{U}{U}^\T W_0^\T=(W_0{U})(W_0{U})^\T$.
The condition $W_0{U}=0$ implies $W_0J{\Sigma_\psi}=W_0{U} F^\T=0$;
conversely, $W_0J{\Sigma_\psi}=0$ implies $W_0{U}{U}^\T=0$.
This proves the equality condition even for singular ${\Sigma_\psi}$.

\end{proof}

For one nuisance parameter, write $\Sigma_\psi=c$ and $J=j$,
with $\lambda$ as defined in (\ref{eq:capture}).
\begin{corollary}[Error for a fixed nuisance parameter value]
\label{prop:perdraw}
Under Assumption~\ref{as:frozen} with one nuisance parameter,
fix $\psi^\star=\mu+\delta$.
Let $\mathcal R_\delta(W)$ be the mean squared error averaged over $x$
and $e$ with $\delta$ held fixed.
The scalar identities
$\|W_Sj\|^2=\|W_0j\|^2/(1+\lambda)^2$ and
$\tau_0(c)=c\lambda\|W_0j\|^2/(1+\lambda)^2$
give the conditional difference
\begin{equation}
 \mathcal R_\delta(W_0)-\mathcal R_\delta(W_S)
 =\frac{\lambda\|W_0j\|^2}{(1+\lambda)^2}
   \big[(\lambda+2)\delta^2-c\big].
 \label{eq:scalar-sign}
\end{equation}
For distinct estimators the prefactor is positive. Their conditional
errors agree at the threshold
\begin{equation}
  |\delta^\star|^2=\frac{c}{\lambda+2}.  
  \label{eq:crossover}
\end{equation}
\end{corollary}

\begin{proof}
For one parameter, Sherman--Morrison gives
\[
 W_S-W_0=-\frac{c}{1+\lambda}W_0j\,j^\T M_0^{-1},
 \qquad W_Sj=\frac{W_0j}{1+\lambda}.
\]
Consequently,
\[
 \tau_0(c)=\frac{c\lambda}{(1+\lambda)^2}\|W_0j\|^2,\qquad
 \|W_0j\delta\|^2-\|W_Sj\delta\|^2
 =\frac{\lambda(\lambda+2)\delta^2}{(1+\lambda)^2}\|W_0j\|^2.
\]
Subtracting these terms proves (\ref{eq:scalar-sign}).
For distinct estimators the common prefactor is positive; setting the
gain to zero gives $|\delta^\star|^2=c/(\lambda+2)$.
\end{proof}
For one nuisance parameter and distinct estimators, the point-likelihood estimator has lower conditional error when $|\delta|<|\delta^\star|$, and UAL when $|\delta|>|\delta^\star|$. Relative to the nuisance standard deviation $\sqrt c$, the crossover distance is $(\lambda+2)^{-1/2}$, which decreases as $\lambda$ increases.

\begin{remark}[Multivariate conditional gain]
\label{rem:perdraw}
For a fixed direction $v$, let
$a_v=\|W_0Jv\|^2-\|W_SJv\|^2$. Along $\delta=sv$, the gain is
\[
 \mathcal R_{sv}(W_0)-\mathcal R_{sv}(W_S)=s^2a_v-\tau_0({\Sigma_\psi}).
\]
If $a_v>0$, the crossover occurs at $|s|=\sqrt{\tau_0({\Sigma_\psi})/a_v}$.
If $a_v\le0$, increasing $|s|$ does not produce a positive gain.
The direction of the nuisance error therefore determines whether increasing its
magnitude can favor UAL\@.
\end{remark}

\paragraph{Mismatch between assumed and generating moments.}

The scalar comparison also quantifies how inaccurate nuisance moments affect average reconstruction error.
Keep $A_0,J,\Sigma_x,\Sigma_e$ fixed, and draw $\psi^\star$ independently of $x$ and $e$ from a distribution that differs from $q$.
For one nuisance parameter, averaging (\ref{eq:scalar-sign}) replaces
$\delta^2$ by its second moment about the mean $\mu$,
\[
 \E[(\psi^\star-\mu)^2]
 =\operatorname{Var}(\psi^\star)+(\E\psi^\star-\mu)^2.
\]
For assumed variance $c>0$ and distinct estimators, average gain is
positive exactly when
\begin{equation}
 \operatorname{Var}(\psi^\star)+(\E\psi^\star-\mu)^2
 >\frac{c}{\lambda+2}.
 \label{eq:misspecified-moments}
\end{equation}
Below this threshold, the filter-change cost exceeds the reduction in
nuisance-induced error, so the point estimator has lower average risk. The risks
agree at the threshold.
A sufficient condition for positive average gain is $0<c\le2\E[(\psi^\star-\mu)^2]$: distinct filters imply $\lambda>0$, so $c/(\lambda+2)<c/2$. Thus the assumed variance $c$ can be as large as twice the generating second moment about the mean $\mu$ while preserving the average advantage in this scalar fixed-sensitivity model.

Define $\lambda^\star=\E[(\psi^\star-\mu)^2]j^\T M_0^{-1}j$. Compared with $\lambda$, $\lambda^\star$ uses the generating second moment about $\mu$. The average gain is negative exactly when
\begin{equation}
 \lambda^\star<1
 \quad\text{and}\quad
 c>\frac{2\E[(\psi^\star-\mu)^2]}{1-\lambda^\star}.
 \label{eq:variance-overstatement}
\end{equation}
For one nuisance parameter and distinct filters, using too large a variance $c$ can reverse the average gain only when $\lambda^\star<1$.
 For $\lambda^\star\ge1$, the gain remains positive for every $c>0$. 

\subsection{Reconstruction error from covariance truncation}

Corollary~\ref{cor:covariance-truncation} bounds Gaussian-likelihood divergence after covariance truncation. The next result bounds its reconstruction cost in the fixed-sensitivity linear model, using the discarded measurement-covariance modes.

\begin{proposition}[Covariance truncation]
\label{prop:rank}
Under Assumption~\ref{as:frozen}, let $\Sigma_e=\sigma^2I$ and
choose a retained covariance $\sigma^2I\preceq S_K\preceq S$.
Write the discarded covariance as
$D=S-S_K=\sum_i\nu_i u_i u_i^\T$, with orthonormal $u_i$ and
$\nu_i\ge0$, and set
$W_{S_K}=\Sigma_x A_0^\T(A_0\Sigma_x A_0^\T+S_K)^{-1}$.
Evaluate both filters under the full covariance $S$. Then
\begin{equation}
  0\le\mathcal R(W_{S_K})-\mathcal R(W_S)
  \le\|W_{S_K}\|_{\mathrm{op}}^2\sum_i\frac{\nu_i^2}{\sigma^2+\nu_i}
  \le\frac{\|\Sigma_x A_0^\T\|_{\mathrm{op}}^2}{\sigma^4}\sum_i\nu_i.
  \label{eq:rank}
\end{equation}
When $S_K=S$, both sides are zero.
\end{proposition}

In the upper bound, a discarded variance contributes approximately
$\nu_i^2/\sigma^2$ when $\nu_i\ll\sigma^2$ and approximately $\nu_i$ when
$\nu_i\gg\sigma^2$. The factor $\|W_{S_K}\|_{\mathrm{op}}^2$ accounts for
amplification by the reconstruction filter.
The bound also applies when the nuisance covariance is truncated:
$S_K=\sigma^2I+J\Sigma_{\psi,K}J^\T$ with
$0\preceq\Sigma_{\psi,K}\preceq\Sigma_\psi$.

\begin{proof}
Write $M_{S_K}=A_0\Sigma_xA_0^\T+S_K$, so
$M_S=M_{S_K}+D$.
The completed square (\ref{eq:square}) gives
\[
 \mathcal R(W_{S_K})-\mathcal R(W_S)
 =\|(W_{S_K}-W_S)M_S^{1/2}\|_F^2\ge0.
\]
By the resolvent identity,
\[
 W_{S_K}-W_S
 =\Sigma_xA_0^\T(M_{S_K}^{-1}-M_S^{-1})
 =W_{S_K}DM_S^{-1}.
\]
Hence the excess risk is
\[
 \tr\!\left(W_{S_K}DM_S^{-1}DW_{S_K}^\T\right).
\]
Since $M_{S_K}\succeq\sigma^2I$, we have
$M_S\succeq\sigma^2I+D$. Inversion reverses the ordering and
congruence by $D$ preserves it:
\[
 DM_S^{-1}D\preceq D(\sigma^2I+D)^{-1}D
 =\sum_i\frac{\nu_i^2}{\sigma^2+\nu_i}u_i u_i^\T
 \preceq D.
\]
For $X\succeq0$,
$\tr(WXW^\T)\le\|W\|_{\mathrm{op}}^2\tr X$.
Applying this inequality gives the middle bound in (\ref{eq:rank}).
Finally,
\[
 \|W_{S_K}\|_{\mathrm{op}}
 \le\frac{\|\Sigma_xA_0^\T\|_{\mathrm{op}}}{\sigma^2},
 \qquad \frac{\nu_i^2}{\sigma^2+\nu_i}\le\nu_i,
\]
which gives the last bound.
\end{proof}

\section{Task models and UAL implementations}
\label{app:seams}
\label{app:data}

We use the notation of Section~\ref{sec:setting}:
$\mathcal A$ is the forward model, ${\Sigma_e}$ is the observation-noise
covariance, and $\mu,{\Sigma_\psi}$ are the mean and covariance of the specified nuisance model $q$.
 For each application, we construct a measurement likelihood from the nuisance statistics and incorporate its data term into the published sampler update. Each UAL version retains the published data-fitting scales and step sizes. So that these step sizes apply unchanged, BlindDPS rescales the UAL image guidance to the published image-guidance norm before the denoiser Jacobian, and CT rescales the precision-weighted residual to the unweighted residual norm before backprojection. DiffPIR evaluates the covariance at the current prediction and holds it fixed within the data step, which keeps that step a quadratic subproblem.

For fixed $\Sigma_e\succ0$, let $d=y-\bar y$ and define
\begin{equation}
 Q(x)=d(x)^\T S(x)^{-1}d(x)+\log\det S(x)-\log\det\Sigma_e.
 \label{eq:pal-relative-energy}
\end{equation}
For measurement dimension $m$,
$Q=2\ell_{\rm UAL}-\log\det\Sigma_e-m\log(2\pi)$.
Removing the noise-only constant from $2\ell_{\rm UAL}$ leaves its gradient unchanged. The square-root fidelity used below puts this energy on the scale of the sampler's residual-norm data term. Since $S\succeq\Sigma_e$,
$Q\ge0$, with the square-root derivative defined for $Q>0$.
The determinant identity also gives $Q=d^\T S^{-1}d+\log\det G$, evaluated with the reduced factorization in Section~\ref{sec:pal-evaluation}.

\paragraph{DPS-style guidance.}
The UAL score acts on a candidate clean signal.
DPS-style samplers update the noisy diffusion state, so the denoiser Jacobian maps
that score into their guidance direction~\citep{chung2023dps,aali2025adps}:
\begin{equation}
 g_t=
 [D_{x_t}\widehat x_0(x_t)]^\T
 \left.\nabla_x\log p_{\mathrm{UAL}}(y\mid x)\right|_{x=\widehat x_0(x_t)}.
 \label{eq:pal-guidance}
\end{equation}
Here $\widehat x_0(x_t)$ is the denoised estimate from state $x_t$.

Algorithms pair the published operation in burgundy with its UAL replacement in teal; black steps are shared.
Only one of the two alternatives is used.
The listings retain the published notation and label code transcriptions separately.
Table~\ref{tab:task-nuisance-map} summarizes the nuisance statistics and their use in each task.

\begin{table}[htbp]
\centering
\small
\caption[Nuisance statistics and UAL constructions by task]{Nuisance
statistics and UAL constructions by task. The rows summarize the
specifications below; the final column states how each task forms the UAL mean and covariance.}
\label{tab:task-nuisance-map}
\begin{tabular}{@{}
 >{\raggedright\arraybackslash}p{0.14\linewidth}
 @{\hspace{0.02\linewidth}}
 >{\raggedright\arraybackslash}p{0.20\linewidth}
 @{\hspace{0.02\linewidth}}
 >{\raggedright\arraybackslash}p{0.27\linewidth}
 @{\hspace{0.02\linewidth}}
 >{\raggedright\arraybackslash}p{0.33\linewidth}@{}}
\toprule
Task & Nuisance coordinates & Nuisance statistics & UAL construction \\
\midrule
Gaussian blur &
Kernel widths $(s_x,s_y)$ &
Width mean and covariance $\mu$ and ${\Sigma_\psi}=\tau^2I_2$ &
Width-sensitivity covariance (Sections~\ref{app:deblur-models}--\ref{app:deblur-samplers}). \\
\addlinespace
Motion blur &
Coefficients in the retained kernel basis &
Mean and covariance from $131{,}072$ kernels&
Gaussian moment matching in retained coordinates (Section~\ref{app:deblur-models}). \\
\addlinespace
CT &
Flat and dark reference fields &
Mean and covariance across 20 flat and 20 dark frames &
Linearized flat--dark calibration map
(Section~\ref{app:ct}). \\
\addlinespace
MRI &
Real coil-map coefficients &
32 calibration replicates; calibrated maps for the mean prediction &
Exact Gaussian marginal for a forward model affine in Gaussian coil-map coefficients
(Section~\ref{app:mri}). \\
\addlinespace
Ultrasound &
Sound speed $c_s$ &
Gaussian scale from CUBDL reference spread; truncated quadrature &
Per-mode variances from nonlinear sound-speed perturbations
(Section~\ref{app:ultrasound}). \\
\addlinespace
Microscopy &
Five pupil-phase coefficients (Noll modes 5, 6, 7, 8, 11) &
Prescribed Gaussian phase law, also used to generate the synthetic observations &
Derivatives of the normalized PSF; energy with candidate-dependent covariance
(Section~\ref{app:microscopy}). \\
\bottomrule
\end{tabular}
\end{table}

\subsection{Deblurring forward models and nuisance statistics}
\label{app:deblur-models}

Gaussian deblurring uses two uncertain kernel widths; motion deblurring uses
coefficients in a retained kernel basis. The specified width moments or kernel-population
moments determine the uncertainty propagated into each sampler's data term.

\paragraph{Measurements.}
We use $256\times256$ FFHQ images~\citep{karras2019ffhq}.
The 50 candidate images are drawn randomly from the FFHQ validation partition, IDs 60000--69999\@.
The ablation set is a fixed 16-image subset of these 50 images.
The same image identifiers are retained across the deblurring samplers.
For deblurring, (\ref{eq:model}) specializes to
\begin{equation}
 y=\mathcal A(x,\psi)+e,\qquad
 \mathcal A(x,\psi)=A(\psi)x,\qquad
 e\sim\mathcal N(0,{\Sigma_e}),\qquad {\Sigma_e}=\sigma_y^2I.
 \label{eq:deblur-model}
\end{equation}
Here $A(\psi)$ applies the kernel $k(\psi)$ to each image channel.
We write $A_k$ for this same linear action when specifying the kernel
directly, so $A(\psi)=A_{k(\psi)}$.
Images use $[-1,1]$ intensity coordinates; observations use the same scale,
with $\sigma_y=0.02$.

Each paired comparison uses the same observation.
The kernel origin and boundary convention are specified below.
FFHQ's dataset metadata and accompanying materials are distributed under
CC BY-NC-SA 4.0; individual images retain the licenses and author information
recorded in their metadata.\footnote{\url{https://github.com/NVlabs/ffhq-dataset\#licenses}}

\paragraph{Gaussian blur.}
The nuisance vector $\psi=(s_x,s_y)$ contains the horizontal and vertical
kernel widths in pixels. On the $64\times64$ kernel grid, indexed by
$u,v\in\{-32,\ldots,31\}$, we use
\begin{equation}
 k_\psi(u,v)=\frac{1}{Z_\psi}
 \exp\!\left(-\frac{u^2}{2s_y^2}-\frac{v^2}{2s_x^2}\right)
 \mathbf 1_{\{|u|,|v|\le31\}},
 \label{eq:gaussian-kernel}
\end{equation}
where $Z_\psi$ normalizes the kernel to unit sum.
The width mean and covariance are $\mu=(\bar s,\bar s)$
and ${\Sigma_\psi}=\tau^2I_2$.
Table~\ref{tab:gaussian-widths} gives the three reconstruction settings.
The width generator samples the two coordinates independently
from normal distributions truncated to $(0,\infty)$, with $\bar s$ and
$\tau$ specifying the post-truncation mean and standard deviation.
The kernel sensitivities include the derivative of the unit-sum normalization. Let $\langle f\rangle_{k_\psi}=\sum_{u,v}k_\psi(u,v)f(u,v)$. Then
\[
 \kappa_x(u,v)=k_\psi(u,v)\frac{v^2-\langle v^2\rangle_{k_\psi}}{s_x^3},\qquad
 \kappa_y(u,v)=k_\psi(u,v)\frac{u^2-\langle u^2\rangle_{k_\psi}}{s_y^3}.
\]
Each derivative kernel sums to zero. Evaluating them at the reference widths gives
\[
 J_\psi(x)=[A_{\kappa_x}x,A_{\kappa_y}x],\qquad U(x)=\tau J_\psi(x).
\] 
These columns propagate the two width variances into the measurement covariance in (\ref{eq:S}).
We write $A_\mu=A_{k_\mu}$ for the kernel rendered at the mean widths.

BlindDPS estimates a kernel instead of two widths~\citep{chung2023blinddps}.
We read the widths directly from the normalized fitted kernel $\widehat k$:
\begin{equation}
 \begin{aligned}
 \widehat s_x&=\frac{1}{|\omega|}
 \sqrt{-2\log\left|\sum_{u,v}\widehat k(u,v)e^{-\mathrm i\omega v}\right|},\\
 \widehat s_y&=\frac{1}{|\omega|}
 \sqrt{-2\log\left|\sum_{u,v}\widehat k(u,v)e^{-\mathrm i\omega u}\right|},
 \qquad \omega=\frac{4\pi}{64}.
 \end{aligned}
 \label{eq:width-readout}
\end{equation}
The Fourier magnitude of a continuous Gaussian of width $s$ is
$\exp(-s^2\omega^2/2)$; solving this relation gives the readout.
We apply it to the fitted discrete kernel when the Fourier magnitudes
lie strictly between zero and one. 

This Gaussian frequency-matching rule estimates the current reference widths.
It leaves the population mean and covariance unchanged.
The non-blind implementations use their reference widths directly.

\paragraph{Motion blur.}
We generate $131{,}072$ reference kernels per population with the
\texttt{motionblur} generator included in the PnP-DM release.\footnote{\url{https://github.com/LeviBorodenko/motionblur}}
The kernel size is $64\times64$, the generator intensity is $0.5$, and the
reference seeds are $0,\ldots,131071$. Observation kernels use seeds $133120,\ldots,133169$, disjoint from the reference seeds.

For motion blur, the nuisance coordinates describe variation in the
generated kernels themselves. We estimate a kernel mean and covariance
from each reference population and retain its leading covariance modes.

Identify each unit-sum kernel with its vectorized form $k_j\in\R^{4096}$,
and let $M=131{,}072$. We compute 

\begin{equation}
 \bar k=\frac1M\sum_{j=1}^M k_j,\qquad
 C_k=\frac1{M-1}\sum_{j=1}^M(k_j-\bar k)(k_j-\bar k)^\T.
 \label{eq:motion-statistics}
\end{equation}
The full covariance $C_k\in\mathbb R^{4096\times4096}$ is estimated
before dimensionality reduction. Write its eigendecomposition as
$C_k=B\operatorname{diag}(\lambda_1,\ldots,\lambda_{4096})B^\T$,
with eigenvalues in decreasing order.
We retain the leading $P=256$ covariance eigenvectors in $B_P\in\mathbb R^{4096\times P}$ and their variances in $\Lambda_P\in\mathbb R^{P\times P}$, giving $C_k\approx B_P\Lambda_PB_P^\T$. We fixed $P=256$ before running the covariance-rank sweep, which varies $P$. The choice balances running time against covariance truncation error: the retained modes carry 96\% of the kernel variance.
Equivalently, let
$X=[k_1-\bar k,\ldots,k_M-\bar k]/\sqrt{M-1}$.
Since $C_k=XX^\T$, its eigenvectors are the left singular vectors
of $X$, and its eigenvalues are the squared singular values.
We accumulate and eigendecompose $C_k$ directly. Unit-sum kernels have zero-sum centered deviations,
so the covariance has at most $4095$ nonzero modes.
We remove numerical components in the all-ones direction. 
In the notation of Section~\ref{sec:setting}, the retained nuisance
coordinates and specified statistics are
\begin{equation}
 k(\psi)=\bar k+B_P\psi,\qquad
 \mu=0,\qquad {\Sigma_\psi}=\Lambda_P,\qquad
 \widehat\psi=B_P^\T(\widehat k-\bar k).
 \label{eq:motion-coordinates}
\end{equation}
Here $C_k$ is the full kernel covariance; ${\Sigma_\psi}$ is the covariance
in the retained nuisance coordinates. The kernel still has $4096$ entries;
its uncertain displacement is represented by $P$ coefficients.
The projection $\widehat\psi=B_P^\T(\widehat k-\bar k)$ extracts the
retained component of a fitted kernel. Components orthogonal to $B_P$
are outside this representation.

Let $b_j=\operatorname{unvec}_{64\times64}((B_P)_{:,j})$ be the
kernel-shaped $j$th covariance eigenvector. Each $b_j$ is a signed,
zero-sum perturbation direction, whereas $\bar k$ is the unit-sum
mean kernel. By linearity of convolution in the kernel,
\[
 A_{k(\psi)}x=A_{\bar k}x+\sum_{j=1}^{P}\psi_jA_{b_j}x.
\]
Thus the derivative with respect to $\psi_j$ is $A_{b_j}x$. 

For the retained model, the UAL mean and covariance are
\begin{equation}
 \bar y(x)=\mathcal A(x,\mu)=A_{\bar k}x,\qquad
 S(x)={\Sigma_e}+\sum_{j=1}^P
 \lambda_j(A_{b_j}x)(A_{b_j}x)^\T.
 \label{eq:motion-pal}
\end{equation}
Projecting each centered reference kernel gives $\psi^{(j)}=B_P^\T(k_j-\bar k)$. These coefficients have sample mean zero and sample covariance
\[
 \frac1{M-1}\sum_{j=1}^{M}\psi^{(j)}(\psi^{(j)})^\T=B_P^\T C_kB_P=\Lambda_P.
\]
UAL uses the Gaussian coefficient model $\psi\sim\mathcal N(0,\Lambda_P)$. With $F=\Lambda_P^{1/2}$ and $\psi=Fa$, where $a\sim\mathcal N(0,I_P)$, the observation model is 
\[
 y=A_{\bar k}x+U(x)a+e,\qquad U_{:\ell}(x)=\sqrt{\lambda_\ell}A_{b_\ell}x.
\]
Integrating over $a$ gives (\ref{eq:motion-pal}). Convolution is affine in the retained coefficients; the approximations are the Gaussian coefficient law and the truncation of the kernel covariance.

\paragraph{Forward conventions.}
In Gaussian deblurring, DAPS, PnP-DM, and BlindDPS use zero-padded cross-correlation,
with padding of 31 pixels before and 32 after each image dimension.
DiffPIR's released closed-form Fourier deblurring step requires
circular convolution, so we retain that convention and generate its
observations with it. 
Both members of each sampler
pair receive the same observation.
Motion observations use circular convolution.
DAPS, PnP-DM, and DiffPIR use circular mean and mode operators.

The DiffPIR likelihood step uses $x_+=(x_-+1)/2$ and $y_+=(y_-+1)/2$,
converting the stored $[-1,1]$ coordinates to $[0,1]$.
For unit-sum circular convolution, $A\mathbf1=\mathbf1$, so
$y_+=Ax_++e_-/2$. Consequently $\sigma_+=0.01$ and ${\Sigma_{e,+}}={\Sigma_{e,-}}/4$.
Its data-step output is converted back by $x_-=2x_+-1$ before the
diffusion update. 

\subsection{UAL in deblurring samplers}
\label{app:deblur-samplers}

Each sampler incorporates UAL's measurement mean and covariance through its existing data update while retaining its pretrained diffusion prior and sampling schedule.
Two PnP-DM controls in Gaussian deblurring separate covariance recomputation from differentiation.
 One recomputes covariance at each candidate while stopping its derivatives; the other fixes covariance at the ensemble mean of the published sampler's 12 reconstructions. The latter reference is shared across that image's reconstruction seeds; all controls use the same observations and evaluation images. A third control multiplies the point likelihood's noise variance by $v\in\{1,1.5,2,3,5\}$ on the ablation set. In motion deblurring, the noise-variance multipliers are $v\in\{1,3,6,12,24,48,96,192,384\}$ on the ablation set. We compare the multipliers by mean per-trajectory PSNR, keeping the Langevin step size unchanged. Appendix~\ref{app:covariance-control} reports the scalar-control results.

\subsubsection{DAPS}
\label{app:daps}
DAPS alternates reverse diffusion, conditional sampling, and
re-noising~\citep{zhang2025daps}. We retain this outer loop and differentiate the UAL energy in (\ref{eq:pal-energy}) in the inner measurement update for both Gaussian and motion deblurring.
 At noise level $\sigma_t$, reverse diffusion
provides $\widehat x_0(\mathbf x_t)$ in the published notation.
 Approximating the inner prior by $\mathcal N(\widehat x_0(\mathbf x_t),\sigma_t^2I)$ gives the drift $(\widehat x_0-z)/\sigma_t^2$ at candidate $z$.

Let $\tau_{\mathrm D}$ be DAPS's data-fitting scale. For Gaussian and motion deblurring, $\tau_{\mathrm D}=0.01$.
We use $A_\mu=A_{k_\mu}$ for Gaussian blur and
$A_\mu=A_{\bar k}$ for motion blur.
With $d(z)=y-A_\mu z$, the data drift is the negative gradient of the scaled UAL energy
\[
 E_{\rm D}(z)=\frac{\sigma_y^2}{\tau_{\rm D}^2}
 \left[d(z)^\T S(z)^{-1}d(z)+\log\det S(z)-\log\det\Sigma_e\right].
\]
Thus
\begin{equation}
 g_{\mathrm D}(z)=-\nabla_z E_{\rm D}(z)
 =-\frac{2\sigma_y^2}{\tau_{\mathrm D}^2}
 \nabla_z\ell_{\mathrm{UAL}}(z;y).
 \label{eq:daps-data-drift}
\end{equation}
Since $\Sigma_e=\sigma_y^2I$, setting $\Sigma_\psi=0$ gives $E_{\rm D}(z)=\|y-A_\mu z\|^2/\tau_{\rm D}^2$, recovering the published DAPS squared-residual data term. 
The coefficient in (\ref{eq:daps-data-drift}) therefore preserves the data-fitting scale $\tau_{\rm D}$, distinct from the observation-noise scale $\sigma_y$.

Gaussian blur uses two width-sensitivity columns; motion blur uses the
retained kernel-mode responses. We evaluate the energy in (\ref{eq:pal-factorized-energy}). Algorithm~\ref{alg:pal-daps} shows the evaluated branch of the published DAPS algorithm.
The paired configurations share their inner step sizes, prior, and
re-noising schedule. Our comparison uses the Langevin branch. DAPS uses $N_A$ outer steps and $N$ Langevin steps of size $\eta_t$; sampling starts at $T=t_{N_A}$.

\begin{algorithm}[htbp]
\small
\color{black}
\caption{Decoupled Annealing Posterior Sampling (DAPS): the evaluated branch of published Algorithm 1~\citep{zhang2025daps}, with UAL replacement lines.}
\label{alg:pal-daps}

\PALlegend
\begin{algorithmic}[1]
\REQUIRE Score model $\mathbf s_\theta$, measurement $\mathbf y$, noise schedule $\sigma_t$, $(t_i)_{i\in\{0,\ldots,N_A\}}$.
\STATE Sample $\mathbf x_T\sim\mathcal N(\mathbf0,\sigma_T^2I)$.
\FOR{$i=N_A,N_A-1,\ldots,1$}
\STATE Compute $\widehat{\mathbf x}_0^{(0)}=\widehat{\mathbf x}_0(\mathbf x_{t_i})$ by solving the probability flow ODE in Eq.~(48) of the cited supplement with $\mathbf s_\theta$
\FOR{$j=0,\ldots,N-1$}
\STATE \PALoriginal{$\begin{aligned}
\widehat{\mathbf x}_0^{(j+1)}\leftarrow{}&
\widehat{\mathbf x}_0^{(j)}
+\eta_t\bigl(\nabla_{\widehat{\mathbf x}_0}
\log p(\widehat{\mathbf x}_0^{(j)}\mid\mathbf x_{t_i})
+\nabla_{\widehat{\mathbf x}_0}\log p(\mathbf y\mid\widehat{\mathbf x}_0^{(j)})\bigr)\\
&+\sqrt{2\eta_t}\,\epsilon_j,\qquad \epsilon_j\sim\mathcal N(\mathbf0,I).
\end{aligned}$}
\PALalt{$\begin{aligned}
\widehat{\mathbf x}_0^{(j+1)}\leftarrow{}&
\widehat{\mathbf x}_0^{(j)}
+\eta_t\bigl(\nabla_{\widehat{\mathbf x}_0}
\log p(\widehat{\mathbf x}_0^{(j)}\mid\mathbf x_{t_i})
+g_{\mathrm D}(\widehat{\mathbf x}_0^{(j)})\bigr)\\
&+\sqrt{2\eta_t}\,\epsilon_j,\qquad \epsilon_j\sim\mathcal N(\mathbf0,I).
\end{aligned}$ (UAL drift; \ref{eq:daps-data-drift})}
\ENDFOR
\STATE Sample $\mathbf x_{t_{i-1}}\sim
\mathcal N(\widehat{\mathbf x}_0^{(N)},\sigma_{t_{i-1}}^2I)$.
\ENDFOR
\RETURN $\mathbf x_0$
\end{algorithmic}
\end{algorithm}

\subsubsection{PnP-DM}
\label{app:pnpdm}

PnP-DM alternates a likelihood conditional with a diffusion-prior
conditional~\citep{wu2024pnpdm}. At outer iteration $k$, the likelihood
step starts from $\mathbf x^{(k)}$ and targets the conditional distribution of an auxiliary image $\mathbf u$,
with potential
$f(\mathbf u;\mathbf y)+\|\mathbf u-\mathbf x^{(k)}\|^2/(2\rho_k^2)$.
Here $f$ is the negative log likelihood potential in the published notation.
For UAL, we set $f(\mathbf u;\mathbf y)=\ell_{\rm UAL}(\mathbf u;\mathbf y)$, giving
\begin{equation}
 \pi_k(\mathbf u)\propto
 \exp\!\left[-\ell_{\rm UAL}(\mathbf u;\mathbf y)
 -\frac{\|\mathbf u-\mathbf x^{(k)}\|^2}{2\rho_k^2}\right].
 \label{eq:pal-inner-target}
\end{equation}
For Gaussian and motion deblurring, both versions use the published general Langevin likelihood routine with 100 inner iterations and step size $10^{-4}$.
The coupling scale starts at $10$, decreases by a factor of $0.9$,
and is bounded below by $0.3$. 

Each column of $U(x)$ is linear in $x$: $u_\ell(x)=L_\ell x$. For Gaussian blur, $L_1=\tau A_{\kappa_x}$ and $L_2=\tau A_{\kappa_y}$; for motion blur, $L_\ell=\sqrt{\lambda_\ell}A_{b_\ell}$. With $d=y-A_\mu x$ and $w=S(x)^{-1}d$, the likelihood step uses
\[
 \nabla_x\ell_{\rm UAL}=-A_\mu^\T w+
 \sum_\ell L_\ell^\T\bigl[S(x)^{-1}u_\ell-w(u_\ell^\T w)\bigr].
\]
The first term differentiates the predicted measurement mean. The two terms in the sum differentiate the log determinant and the covariance-weighted residual penalty. PnP-DM evaluates this gradient at each inner Langevin candidate.

Algorithms~\ref{alg:pal-pnpdm} and~\ref{alg:pnp-likelihood} show the published outer sampler and its likelihood subroutine.
The prior update is unchanged from the released method. $D_\theta$ is the denoiser, and $p(\cdot;\sigma)$ is the Gaussian-smoothed prior at noise level $\sigma$.
\begin{algorithm}[htbp]
\small
\color{black}
\caption{Plug-and-Play Diffusion Models (PnP-DM): published Algorithm 1~\citep{wu2024pnpdm}, with UAL replacement lines.}
\label{alg:pal-pnpdm}

\PALlegend
\begin{algorithmic}[1]
\item[] \textbf{Input:} initialization $\mathbf x_0\in\R^n$, total number of iterations $K>0$, coupling strength schedule $\{\rho_k>0\}_{k=0}^{K-1}$, likelihood potential $f(\,\cdot\,;\mathbf y)$ with measurements $\mathbf y\in\R^m$, pretrained model $D_\theta(\,\cdot\,;\,\cdot\,)$ that approximates $\nabla\log p(\mathbf x;\sigma)$ with $(D_\theta(\mathbf x;\sigma)-\mathbf x)/\sigma^2$.
\FOR{$k=0,\ldots,K-1$}
\STATE \PALoriginal{$\mathbf z^{(k)}\leftarrow
\texttt{LikelihoodStep}(\mathbf x^{(k)},\rho_k,f(\,\cdot\,;\mathbf y))$}
\PALalt{$\mathbf z^{(k)}\leftarrow
\texttt{LikelihoodStep}(\mathbf x^{(k)},\rho_k,\ell_{\mathrm{UAL}}(\,\cdot\,;\mathbf y))$ (\ref{eq:pal-energy})}
\STATE $\mathbf x^{(k+1)}\leftarrow
\texttt{PriorStep}(\mathbf z^{(k)},\rho_k,D_\theta(\,\cdot\,;\,\cdot\,))$
\ENDFOR
\RETURN $\mathbf x^{(k+1)}$
\end{algorithmic}
\end{algorithm}

\begin{algorithm}[!t]
\small
\color{black}
\caption{PnP-DM likelihood step: published Algorithm 2~\citep{wu2024pnpdm}, with UAL replacement lines.}
\label{alg:pnp-likelihood}

\PALlegend
\begin{algorithmic}[1]
\item[] \textbf{Input:} state $\mathbf x$, coupling strength $\rho>0$, likelihood potential $f(\,\cdot\,;\mathbf y)$ with measurements $\mathbf y$
\item[] \textbf{Hyperparameter:} step size $\gamma>0$, number of iterations $J>0$
\STATE $\mathbf u_0\leftarrow\mathbf x$
\FOR{$j=0,\ldots,J-1$}
\STATE \PALoriginal{$\begin{aligned}
\mathbf u_{j+1}\leftarrow{}&
 \mathbf u_j-\gamma\nabla f(\mathbf u_j;\mathbf y)
 -\frac{\gamma}{\rho^2}(\mathbf u_j-\mathbf x)
 +\sqrt{2\gamma}\,\epsilon_j,\quad
 \epsilon_j\sim\mathcal N(\mathbf0,I)
\end{aligned}$}
\PALalt{$\begin{aligned}
\mathbf u_{j+1}\leftarrow{}&
 \mathbf u_j-\gamma\nabla\ell_{\mathrm{UAL}}(\mathbf u_j;\mathbf y)
 -\frac{\gamma}{\rho^2}(\mathbf u_j-\mathbf x)
 +\sqrt{2\gamma}\,\epsilon_j,\quad
 \epsilon_j\sim\mathcal N(\mathbf0,I)
\end{aligned}$ (\ref{eq:pal-inner-target})}
\ENDFOR
\RETURN $\mathbf u_J$
\end{algorithmic}
\end{algorithm}

\subsubsection{DiffPIR}
\label{app:diffpir}

DiffPIR combines a diffusion prediction with a data-fidelity
subproblem~\citep{zhu2023diffpir}. Let $v_t$ be its prediction in the likelihood
coordinates and let $\rho_t^{\mathrm{pub}}>0$ be the released quadratic
coupling coefficient. Since the UAL precision already includes the noise
variance, we set $\widetilde\rho_t=\rho_t^{\mathrm{pub}}/\sigma_+^2$.
The data step uses image and measurement values in $[0,1]$. Define $d(z)=y-A_\mu z$.
At each diffusion step, the UAL update evaluates the covariance at the prediction, $S_t=S(v_t)$, and solves
\begin{equation}
 z_t=\arg\min_z\left\{
 \frac12(y-A_\mu z)^\T S_t^{-1}(y-A_\mu z)
 +\frac{\widetilde\rho_t}{2}\|z-v_t\|^2\right\}.
 \label{eq:diffpir-problem}
\end{equation}
The normal equation is
\[
 (A_\mu^\T S_t^{-1}A_\mu+\widetilde\rho_t I)z_t
 =A_\mu^\T S_t^{-1}y+\widetilde\rho_t v_t.
\]
Multiplying the objective by $2\sigma_+^2$ explains the coupling
$\widetilde\rho_t=\rho_t^{\rm pub}/\sigma_+^2$ and recovers the
published squared-residual units at zero nuisance covariance.
The inverse identity gives the solution
\begin{equation}
 z_t=v_t+\widetilde\rho_t^{-1}A_\mu^\T
 (S_t+\widetilde\rho_t^{-1}A_\mu A_\mu^\T)^{-1}(y-A_\mu v_t).
 \label{eq:diffpir-solution}
\end{equation}
We form $U_t=U(v_t)$ at the diffusion prediction. 
Set
\[
 D_t=\sigma_+^2I+\widetilde\rho_t^{-1}A_\mu A_\mu^\T,\qquad H_t=I+U_t^\T D_t^{-1}U_t.
\]
We evaluate the inverse in (\ref{eq:diffpir-solution}) with the Woodbury identity:
\[
 (D_t+U_tU_t^\T)^{-1}=D_t^{-1}-D_t^{-1}U_tH_t^{-1}U_t^\T D_t^{-1}.
\]
Circular convolution diagonalizes $D_t$ in Fourier coordinates, with diagonal $\sigma_+^2+|\widehat k_\mu|^2/\widetilde\rho_t$, where $\widehat k_\mu$ is the Fourier transform of the mean kernel. The remaining solve has the dimension of the retained nuisance representation. 

Algorithm~\ref{alg:pal-diffpir} retains the published reverse-diffusion
loop and shows the original data-proximal line and its UAL replacement. In Algorithm~\ref{alg:pal-diffpir}, $\mathcal H=A_\mu$, $\sigma_n=\sigma_+$, and $\rho_t=\rho_t^{\rm pub}$.
We follow the published DiffPIR update, using regularization weight $\lambda=7$, noise mixing $\zeta=0.3$, and DDIM noise parameter $\eta=0$, with 100 denoiser evaluations on the quadratic subsequence of a 1000-step grid. Here $\bar\alpha_t$ is the cumulative signal coefficient and $\bar\sigma_t=\sqrt{(1-\bar\alpha_t)/\bar\alpha_t}$ is the corresponding noise level; $t-1$ denotes the next selected time.
\begin{algorithm}[htbp]
\small
\color{black}
\caption{DiffPIR: published Algorithm 1~\citep{zhu2023diffpir}, with UAL replacement lines.}
\label{alg:pal-diffpir}

\PALlegend
\begin{algorithmic}[1]
\REQUIRE Score model $\mathbf s_\theta$, $T,\mathbf y,\sigma_n,\{\bar\sigma_t\}_{t=1}^T,\zeta,\lambda$
\STATE Initialize $\mathbf x_T\sim\mathcal N(\mathbf0,I)$, pre-calculate
 $\rho_t\triangleq\lambda\sigma_n^2/\bar\sigma_t^2$.
\FOR{$t=T$ to $1$}
\STATE $\mathbf x_0^{(t)}
 =\dfrac1{\sqrt{\bar\alpha_t}}
 \bigl(\mathbf x_t+(1-\bar\alpha_t)\mathbf s_\theta(\mathbf x_t,t)\bigr)$
\STATE \PALoriginal{$\widehat{\mathbf x}_0^{(t)}
 =\arg\min_{\mathbf x}\|\mathbf y-\mathcal H(\mathbf x)\|^2
 +\rho_t\|\mathbf x-\mathbf x_0^{(t)}\|^2$}
\PALalt{$\begin{aligned}
 \widehat{\mathbf x}_0^{(t)}
 ={}&2\left[\arg\min_{\mathbf z}
 \left\{\sigma_n^2d(\mathbf z)^\T S_t^{-1}d(\mathbf z)
       +\rho_t\|\mathbf z-v_t\|^2\right\}\right]-\mathbf1,\\
 &v_t=(\mathbf x_0^{(t)}+\mathbf1)/2,\qquad S_t=S(v_t).
 \end{aligned}$ (\ref{eq:diffpir-problem})}
\STATE $\widehat\epsilon
 =\dfrac1{\sqrt{1-\bar\alpha_t}}
 \bigl(\mathbf x_t-\sqrt{\bar\alpha_t}\widehat{\mathbf x}_0^{(t)}\bigr)$
\STATE $\epsilon_t\sim\mathcal N(\mathbf0,I)$
\STATE $\mathbf x_{t-1}
 =\sqrt{\bar\alpha_{t-1}}\widehat{\mathbf x}_0^{(t)}
 +\sqrt{1-\bar\alpha_{t-1}}
 \bigl(\sqrt{1-\zeta}\,\widehat\epsilon+\sqrt\zeta\,\epsilon_t\bigr)$
\ENDFOR
\RETURN $\mathbf x_0$
\end{algorithmic}
\end{algorithm}

\subsubsection{BlindDPS}
\label{app:blinddps}
BlindDPS evolves image and kernel diffusion states and couples them
through measurement guidance~\citep{chung2023blinddps}.
At each guidance query, the image branch differentiates the UAL energy in (\ref{eq:pal-energy}).
Published BlindDPS measures data fit by the residual norm
$\|y-\widehat k*\widehat x_0\|$. With $\Sigma_e=\sigma_y^2I$, this norm is
$\sigma_y$ times the square root of the point-likelihood energy
$\|y-\widehat k*\widehat x_0\|^2/\sigma_y^2$. We keep this form and replace
the point-likelihood energy by $Q$ from (\ref{eq:pal-relative-energy}):
\begin{equation}
 D_{\rm UAL}=\sigma_y\sqrt Q,\qquad
 \nabla D_{\rm UAL}=\frac{\sigma_y}{\sqrt Q}\,\nabla\ell_{\rm UAL}.
 \label{eq:blind-fidelity}
\end{equation}
The gradient is a positive multiple of $\nabla\ell_{\rm UAL}$, so the image
step follows the UAL score. The gradient requires $Q>0$, which holds unless
$d=0$ and $U=0$. Because $D_{\rm UAL}$ has the units of the residual norm, the kernel weight $\lambda$ is used unchanged.

With $\Sigma_\psi=0$, $D_{\rm UAL}=\|d\|$, the residual norm at the UAL mean.
The denoiser Jacobian maps the full image derivative to the diffusion state as in (\ref{eq:pal-guidance}).
For Gaussian and motion blur, we scale the UAL image guidance to have the same magnitude as the published image guidance before applying the denoiser Jacobian. Here $\widehat k_i$ is the projected kernel estimate at iteration $i$, denoted $\widehat{\mathbf k}_0$ in Algorithm~\ref{alg:pal-blinddps}. At the current denoised image, let
\[
 \begin{aligned}
 r_{{\rm pub},i}&=\nabla_z\|y-A_{\widehat k_i}z\|,\\
 r_{{\rm UAL},i}&=\nabla_zD_{\rm UAL}(z;y),\qquad
 c_i=\frac{\|r_{{\rm pub},i}\|}{\|r_{{\rm UAL},i}\|}.
 \end{aligned}
\]
Then $\|c_i r_{{\rm UAL},i}\|=\|r_{{\rm pub},i}\|$. The image update applies $-\alpha c_i[D_{x_i}\widehat x_0]^\T r_{{\rm UAL},i}$, where $\alpha$ is the guidance step size. We retain the published guidance step size $\alpha$. For motion blur, $r_{{\rm UAL},i}$ is the gradient of the motion fidelity defined below.

\paragraph{Gaussian blur.}
The fitted kernel gives the current reference widths $\widehat s_i$ through (\ref{eq:width-readout}). The population moments remain $\mu=(\bar s,\bar s)$ and $\Sigma_\psi=\tau^2I_2$. Define
\[
 J_i(z)=\left.D_s(A_{k(s)}z)\right|_{s=\widehat s_i}.
\]
In the model linearized at these reference widths,
\[
 \begin{aligned}
 y\mid z,a&\sim\mathcal N\!\left(
 \mathcal A(z,\widehat s_i)+J_i(z)(\mu-\widehat s_i)+\tau J_i(z)a,
 \sigma_y^2I\right),\\
 a&\sim\mathcal N(0,I_2).
 \end{aligned}
\]
Integrating over $a$ gives the measurement mean
\begin{equation}
 \bar y(z)=\mathcal A(z,\widehat s_i)+J_i(z)(\mu-\widehat s_i),
 \qquad \psi_0=\widehat s_i,
 \label{eq:blind-gaussian-mean}
\end{equation}
and covariance $S_i(z)=\sigma_y^2I+\tau^2J_i(z)J_i(z)^\T$. The correction maps $\mu-\widehat s_i$ into measurement space. The image derivative includes the dependence of this mean correction and covariance on $z$, and the guidance uses the square-root fidelity in (\ref{eq:blind-fidelity}).
In Gaussian blur, the fidelity depends on the fitted kernel through the reference widths $\widehat s_i$. The kernel diffusion and projection steps are retained.

\paragraph{Motion blur.}
For motion blur, let $\widehat k_i$ be the current fitted kernel and set
$J_i(z)=[A_{b_1}z,\ldots,A_{b_P}z]$.
The retained coordinates are $\widehat\psi_i=B_P^\T(\widehat k_i-\bar k)$, as in (\ref{eq:motion-coordinates}).
The residual and covariance are
\begin{equation}
 \begin{aligned}
 d_i(z)&=y-A_{\widehat k_i}z+J_i(z)\widehat\psi_i,\\
 S_i(z)&=\Sigma_e+J_i(z)\Lambda_PJ_i(z)^\T.
 \end{aligned}
 \label{eq:blind-motion-local}
\end{equation}
Both remain differentiable functions of the candidate image.
The mean correction gives the effective kernel
$k_{i,\rm eff}=\bar k+(I-B_PB_P^\T)(\widehat k_i-\bar k)$, so
$D_zd_i=-A_{k_{i,\rm eff}}$.
It corrects the retained kernel deviation toward the population mean and preserves the deviation outside the retained subspace.
The kernel update retains the released guidance.
Algorithm~\ref{alg:pal-blinddps} retains the published sampler
and identifies the two guidance substitutions. In this listing, $\widehat{\mathbf s}^{\,i}$ and $\widehat{\mathbf s}^{\,k}$ denote negative image and kernel noise predictions, respectively. The kernel projection $P_C$ enforces nonnegativity and unit sum. We retain the kernel regularizer $R_{\mathbf k}$ and weight $\lambda$ used by the published BlindDPS method.

\begin{algorithm}[htbp]
\small
\color{black}
\caption{BlindDPS for blind deblurring: published Algorithm~1~\citep{chung2023blinddps}, with UAL replacement lines.}
\label{alg:pal-blinddps}

\PALlegend
\begin{algorithmic}[1]
\REQUIRE $N,\mathbf y,\alpha,\{\widetilde\sigma_i\}_{i=1}^N,\lambda,R_{\mathbf k}(\,\cdot\,)$
\STATE $\mathbf x_N,\mathbf k_N\sim\mathcal N(\mathbf0,I)$
\FOR{$i=N-1$ to $0$}
\STATE $\widehat{\mathbf s}^{\,i}\leftarrow\mathbf s_{\theta^*}^{\,i}(\mathbf x_i,i)$
\STATE $\widehat{\mathbf s}^{\,k}\leftarrow\mathbf s_{\theta^*}^{\,k}(\mathbf k_i,i)$
\STATE $\widehat{\mathbf x}_0\leftarrow
 \dfrac1{\sqrt{\bar\alpha_i}}
 (\mathbf x_i+\sqrt{1-\bar\alpha_i}\widehat{\mathbf s}^{\,i})$
\STATE $\widehat{\mathbf k}_0\leftarrow
 \dfrac1{\sqrt{\bar\alpha_i}}
 (\mathbf k_i+\sqrt{1-\bar\alpha_i}\widehat{\mathbf s}^{\,k})$
\STATE $\widehat{\mathbf k}_0\leftarrow P_C(\widehat{\mathbf k}_0)$
\STATE $\mathbf z_i,\mathbf z_k\sim\mathcal N(\mathbf0,I)$
\STATE $\mathbf x'_{i-1}\leftarrow
 \dfrac{\sqrt{\alpha_i}(1-\bar\alpha_{i-1})}{1-\bar\alpha_i}\mathbf x_i
 +\dfrac{\sqrt{\bar\alpha_{i-1}}\beta_i}{1-\bar\alpha_i}\widehat{\mathbf x}_0
 +\widetilde\sigma_i\mathbf z_i$
\STATE $\mathbf k'_{i-1}\leftarrow
 \dfrac{\sqrt{\alpha_i}(1-\bar\alpha_{i-1})}{1-\bar\alpha_i}\mathbf k_i
 +\dfrac{\sqrt{\bar\alpha_{i-1}}\beta_i}{1-\bar\alpha_i}\widehat{\mathbf k}_0
 +\widetilde\sigma_i\mathbf z_k$
\STATE \PALoriginal{$\mathbf x_{i-1}\leftarrow\mathbf x'_{i-1}
 -\alpha\nabla_{\mathbf x_i}
 \|\mathbf y-\widehat{\mathbf k}_0*\widehat{\mathbf x}_0\|_2$}
\PALalt{$\mathbf x_{i-1}\leftarrow\mathbf x'_{i-1}
 -\alpha c_i\nabla_{\mathbf x_i}
 D_{\rm UAL}(\widehat{\mathbf x}_0;\mathbf y)$ (UAL gradient; \ref{eq:blind-fidelity})}
\STATE \PALoriginal{$L_{\mathbf k}\leftarrow
 \|\mathbf y-\widehat{\mathbf k}_0*\widehat{\mathbf x}_0\|_2
 +\lambda R_{\mathbf k}(\widehat{\mathbf k}_0)$}
\PALalt{$L_{\mathbf k}\leftarrow
 D_{\rm UAL}(\widehat{\mathbf x}_0;\mathbf y)
 +\lambda R_{\mathbf k}(\widehat{\mathbf k}_0)$ (Gaussian blur; for motion blur, use the original line.) (\ref{eq:blind-fidelity})}
\STATE $\mathbf k_{i-1}\leftarrow\mathbf k'_{i-1}
 -\alpha\nabla_{\mathbf k_i}L_{\mathbf k}$
\ENDFOR
\RETURN $\mathbf x_0,\mathbf k_0$
\end{algorithmic}
\end{algorithm}

\subsubsection{Fast Diffusion EM}
\label{app:fdem}
Fast Diffusion EM alternates diffusion posterior sampling of the image with a maximization step for the kernel~\citep{laroche2024fastem}; the released configuration we use guides the image with $\Pi$GDM~\citep{song2023pseudoinverse}. Let $r_t$ be the clean-image uncertainty scale of that guidance. The released schedule sets $r_t^2=v_t/(1+v_t)$, with $v_t=\frac{1-\bar\alpha_{t-1}}{1-\bar\alpha_t}(1-\frac{\bar\alpha_t}{\bar\alpha_{t-1}})$, where $\bar\alpha_t$ is the cumulative diffusion coefficient and $t-1$ denotes the next scheduled reverse step. At a denoised candidate $z$, UAL uses the mean response $A_{\bar k}z$ and combines clean-image and kernel-coefficient uncertainty in the measurement covariance:
\begin{equation}
 \begin{aligned}
 \bar y_t(z)&=A_{\bar k}z,\\
 S_t(z)&=\Sigma_e+r_t^2A_{\bar k}A_{\bar k}^\T+J_\psi(z)\Lambda_PJ_\psi(z)^\T\\
 &\quad+r_t^2\sum_{j=1}^P\lambda_jA_{b_j}A_{b_j}^\T.
 \end{aligned}
 \label{eq:fdem-moments}
\end{equation}
The image update uses the marginal mean $A_{\bar k}z$, obtained by averaging over the kernel-coefficient model in (\ref{eq:motion-coordinates}), in place of the fitted-kernel response; the kernel-maximization step is unchanged. The final covariance term accounts for the interaction between image and kernel uncertainty. These quantities use Fast Diffusion EM\textquoteright s guidance coordinates, with the $[0,1]$ noise-scale conversion described for DiffPIR. The negative gradient of the Gaussian energy with these moments, including the log determinant and the covariance derivative, replaces the published guidance vector at the denoised candidate; the mapping to the diffusion state and the step size are unchanged. Algorithm~\ref{alg:pal-fdem} gives the published loop in its source notation, with $H$ the kernel operator, and marks the replaced guidance line.

\begin{algorithm}[htbp]
\small
\color{black}
\caption[Fast Diffusion EM with UAL]{Fast Diffusion EM with $\Pi$GDM guidance: published Algorithm~2~\citep{laroche2024fastem}, with the UAL replacement line.}
\label{alg:pal-fdem}

\PALlegend
\begin{algorithmic}[1]
\REQUIRE $y,\sigma,H_T,T$
\STATE $\mathbf x_T\sim\mathcal N(\mathbf 0,I)$
\FOR{$t=T$ to $1$}
\STATE $\widehat{\boldsymbol\epsilon}\leftarrow\epsilon(\mathbf x_t,t)$
\STATE $\widehat{\mathbf x}_0\leftarrow\dfrac1{\sqrt{\bar\alpha_t}}(\mathbf x_t-\sqrt{1-\bar\alpha_t}\,\widehat{\boldsymbol\epsilon})$
\STATE $H_{t-1}\leftarrow\text{M-step}(y,\widehat{\mathbf x}_0,\sigma)$
\STATE \PALoriginal{$\mathbf g\leftarrow\nabla_{\mathbf x_t}\log p(y\mid\mathbf x_t,H_{t-1})$ ($\Pi$GDM)}
\PALalt{$\mathbf g\leftarrow\nabla_{\mathbf x_t}\log\mathcal N\!\left(y;A_{\bar k}\widehat{\mathbf x}_0,S_t(\widehat{\mathbf x}_0)\right)$ (\ref{eq:fdem-moments})}
\STATE $\mathbf s\leftarrow\zeta_t\mathbf g-\dfrac1{\sqrt{1-\bar\alpha_t}}\widehat{\boldsymbol\epsilon}$
\STATE $\mathbf z\sim\mathcal N(\mathbf 0,I)$
\STATE $\mathbf x_{t-1}\leftarrow\dfrac1{\sqrt{\alpha_t}}(\mathbf x_t+\beta_t\mathbf s)+\widetilde\sigma_t\mathbf z$
\ENDFOR
\RETURN $\mathbf x_0,H_0$
\end{algorithmic}
\end{algorithm}

\subsection{CT with flat and dark reference uncertainty}
\label{app:ct}
\paragraph{Data and calibration.}
We use the DM4CT synchrotron rock data~\citep{shi2026dm4ct}, distributed
through Zenodo.\footnote{\url{https://zenodo.org/records/15420527}} It contains $1200$ projections,
$20$ flat frames, and $20$ dark frames.

Our application uses 60 views and the released
DM4CT synchrotron diffusion prior with DAPS~\citep{zhang2025daps}.
We follow the published DM4CT preprocessing and reference reconstructions. Published uses one flat and one dark frame; UAL uses the mean reference fields. Appendix~\ref{app:additional-results} also compares both versions using those mean fields.
Each CT case is a two-dimensional reconstruction from the same detector row across the 60 projection views.

Here $\psi$ collects the flat and dark reference fields.
For raw projection $p$, flat field $f$, and dark field
$d_{\rm dark}$, the calibration map acts elementwise as
\begin{equation}
 z(p;f,d_{\rm dark})=-a_{\rm CT}
 \log\max\!\left\{\frac{p-d_{\rm dark}}{f-d_{\rm dark}},t_{\min}\right\}.
 \label{eq:ct-calibration}
\end{equation}
Here $a_{\rm CT}$ is the log-sinogram scale and $t_{\min}$ is the
transmission floor.
The nuisance vector is $\psi=(f^\T,d_{\rm dark}^\T)^\T$. From each set of $N=20$ reference frames, we compute
\[
 \bar f=\frac1N\sum_{j=1}^{N}f_j,\qquad
 \bar d_{\rm dark}=\frac1N\sum_{j=1}^{N}d_{{\rm dark},j},
\]
\[
 \begin{aligned}
 F_f&=\frac{[f_1-\bar f,\ldots,f_N-\bar f]}{\sqrt{N-1}},\\
 F_d&=\frac{[d_{{\rm dark},1}-\bar d_{\rm dark},\ldots,d_{{\rm dark},N}-\bar d_{\rm dark}]}{\sqrt{N-1}}.
 \end{aligned}
\]
The Gaussian reference-field model has
\[
 \mu=\begin{bmatrix}\bar f\\\bar d_{\rm dark}\end{bmatrix},\qquad
 \Sigma_\psi=\begin{bmatrix}F_fF_f^\T&0\\0&F_dF_d^\T\end{bmatrix}.
\]
These statistics are estimated once and shared across the reconstructions. We use frame-to-frame variability to model the uncertain flat and dark fields, with independent Gaussian fluctuations.

\paragraph{Calibration statistics and propagated covariance.}
The calibrated measurement is $y=z(p;\bar f,\bar d_{\rm dark})$, and $A_{\rm CT}$ is the selected parallel-beam projection operator. Conditioning on the raw projection $p$, the calibration-error response is
\[
 \mathcal A_p(x,\psi)=A_{\rm CT}x+z(p;\mu)-z(p;\psi).
\]
At unclipped entries, the calibration sensitivities are
\begin{equation}
 \frac{\partial z}{\partial f}=\frac{a_{\rm CT}}{f-d_{\rm dark}},\qquad
 \frac{\partial z}{\partial d_{\rm dark}}=a_{\rm CT}\left(\frac1{p-d_{\rm dark}}-\frac1{f-d_{\rm dark}}\right).
 \label{eq:ct-sensitivity}
\end{equation}
Clipped entries have zero derivative. Let $R$ broadcast a detector field over the selected angles, and let $D_f$ and $D_d$ be diagonal matrices of these sensitivities evaluated at $(\bar f,\bar d_{\rm dark})$. Linearization gives
\[
 \bar y(x)=A_{\rm CT}x,\qquad
 J_\psi=-[D_fR,D_dR],\qquad
 U_{\rm CT}=-[D_fRF_f,D_dRF_d].
\]
A reference-field fluctuation affects all selected angles through the same field, so the propagated covariance retains the resulting measurement correlations. In CT, the observation-noise scale is $\sigma_{\rm CT}=1.5$, and the DAPS data-fitting scale is $\tau_{{\rm D},{\rm CT}}=1.5$. The covariance is
\begin{equation}
 S_{\rm CT}=\sigma_{\rm CT}^2I+U_{\rm CT}U_{\rm CT}^\T.
 \label{eq:ct-covariance}
\end{equation} 

\paragraph{Reconstruction update.} 
 
The propagated calibration model gives
\[
 p_{\rm UAL}(y\mid x)=\mathcal N(y;A_{\rm CT}x,S_{\rm CT}),\qquad
 \nabla_x\log p_{\rm UAL}=A_{\rm CT}^\T S_{\rm CT}^{-1}(y-A_{\rm CT}x).
\]
With $d=y-A_{\rm CT}x$ and $w=S_{\rm CT}^{-1}d$, the DAPS data drift is
\begin{equation}
 g_{\rm CT}(x)=\frac{2}{\tau_{{\rm D},{\rm CT}}^2}A_{\rm CT}^\T w\,\frac{\|d\|}{\|w\|}.
 \label{eq:ct-rescaling}
\end{equation}
When $U_{\rm CT}=0$, $w=d/\sigma_{\rm CT}^2$ and $w\|d\|/\|w\|=d$, recovering the published point-likelihood drift $2A_{\rm CT}^\T d/\tau_{{\rm D},{\rm CT}}^2$. The drift is a positive multiple of the UAL score. The norm ratio preserves the unweighted residual norm before backprojection, so the published step sizes and data-fitting scale are used unchanged while covariance weighting changes the residual direction.

CT uses Algorithm~\ref{alg:pal-daps} with the calibrated measurement
specified above. Its Langevin likelihood-score term is replaced by
$g_{\rm CT}(\widehat{\mathbf x}_0^{(j)})$
from (\ref{eq:ct-rescaling}).

\paragraph{Dataset terms.}
The DM4CT dataset is distributed under CC BY 4.0 through Zenodo
record 15420527; we cite the accompanying DM4CT paper.
The pretrained synchrotron diffusion checkpoint is distributed separately
under the MIT license.\footnote{\url{https://huggingface.co/jiayangshi/synchrotron_pixel_diffusion}}

\subsection{MRI with coil-map uncertainty}
\label{app:mri}

We estimate coil-map variation by perturbing the calibration measurements and recalibrating the sensitivity maps. The resulting covariance determines the Gaussian coefficient model used in the Ambient DPS likelihood.

\paragraph{Data and forward model.}
We use multi-coil brain measurements from the NYU fastMRI
Initiative~\citep{knoll2020fastmri,zbontar2019fastmriopendatasetbenchmarks}.
We take the middle slice of each volume, apply coil-noise prewhitening and intensity normalization, and use the resulting k-space measurements, sampling mask, and sensitivity maps.
Both versions use the released Ambient Diffusion checkpoint trained at acceleration 6, the released acceleration-8 masks, and the same coil-map calibration width and noise convention.

Reference images are reconstructed from full k-space data using the
released preparation procedure.
We evaluate 50 T2-weighted brain volumes selected at random without replacement from the fastMRI validation set, excluding six development volumes.

Here $x$ denotes the complex image. For sensitivity maps $s=(s_c)_c$, the response of coil $c$ is
\[
 \mathcal A_c(x,s)=M\mathcal F(s_c\odot x).
\]
Here $\mathcal F$ is the released orthonormal Fourier transform and $M$ is the sampling mask~\citep{aali2025adps}.

\paragraph{Nuisance statistics.}
We perturb the acquired calibration data at the measured noise level. Let $Y_{\rm ACS}$ contain the eight central calibration lines, and let $\mathcal E$ denote ESPIRiT calibration~\citep{uecker2014espirit}. The reference sensitivity maps are $s_0=\mathcal E(Y_{\rm ACS})$. For each scan, $\sigma_{\mathbb C}^2$ is the empirical complex variance of the pooled $30\times30$ upper-left corner patches of the prewhitened, normalized coil images. We generate $B=32$ calibration replicates,
\[
 \begin{aligned}
 Y_{\rm ACS}^{(j)}&=Y_{\rm ACS}+\eta_j,\\
 \eta_j&=\frac{\sigma_{\mathbb C}}{\sqrt2}(\xi_j+\mathrm i\zeta_j),\qquad
 \xi_j,\zeta_j\overset{\rm iid}{\sim}\mathcal N(0,I),
 \end{aligned}
\]
with noise added within the calibration lines. We recalibrate each replicate, $\widetilde s_j=\mathcal E(Y_{\rm ACS}^{(j)})$, and align its common coil phase at each pixel to $s_0$:
\[
 \rho_j(r)=\sum_c\overline{\widetilde s_{j,c}(r)}s_{c,0}(r),\qquad
 s_{j,c}(r)=\widetilde s_{j,c}(r)\frac{\rho_j(r)}{|\rho_j(r)|}.
\]
The phase factor is set to one where the overlap is negligible. Let $\mathcal R$ stack real and imaginary parts over all coils and pixels, and define
\[
 \delta_j=\mathcal R(s_j-s_0),\qquad
 \bar\delta=\frac1B\sum_{j=1}^{B}\delta_j,\qquad
 D=\frac{[\delta_1-\bar\delta,\ldots,\delta_B-\bar\delta]}{\sqrt{B-1}}.
\]
An SVD $D=V\operatorname{diag}(\varsigma_\ell)W^\T$ gives map modes $\mathcal R(b_\ell)=V_{:\ell}$ and variances $\lambda_\ell=\varsigma_\ell^2$ that factor the sample covariance $DD^\T$ of the aligned map perturbations. We retain all 32 returned columns, setting $P=32$; centering the replicates leaves at most 31 nonzero covariance modes. The likelihood mean uses the reference maps $s_0$; the centered replicates estimate coil-map variability under the calibration-noise model above.

The coefficients are modeled as $\psi\sim\mathcal N(0,\operatorname{diag}(\lambda_1,\ldots,\lambda_P))$, with complex map modes and real coefficients:
\begin{equation}
 s_c(\psi)=s_{c,0}+\sum_{\ell=1}^{P}b_{\ell,c}\psi_\ell,\qquad
 \mathcal A_c(x,\psi)=M\mathcal F(s_c(\psi)\odot x).
 \label{eq:mri-forward}
\end{equation}
We also apply $\mathcal R$ to the measurements and write $y$ for the resulting real vector. The predicted mean and covariance factor are
\[
 \begin{aligned}
 \bar y(x)&=\mathcal R\!\left(\{M\mathcal F(s_{c,0}\odot x)\}_c\right),\\
 U_{:\ell}(x)&=\sqrt{\lambda_\ell}\,\mathcal R\!\left(\{M\mathcal F(b_{\ell,c}\odot x)\}_c\right).
 \end{aligned}
\]
For proper complex observation noise, the variance per real component is $\sigma_{\rm R}^2=\sigma_{\mathbb C}^2/2$, giving $\Sigma_e=\sigma_{\rm R}^2I$. Integrating the Gaussian coefficient model gives
\[
 p_{\rm UAL}(y\mid x)=\mathcal N(y;\bar y(x),S(x)),\qquad
 S(x)=\sigma_{\rm R}^2I+U(x)U(x)^\T.
\]
At fixed $x$, the response is affine in $\psi$, so this Gaussian marginal requires no nuisance linearization.

\paragraph{Ambient DPS update.}
Ambient DPS evaluates its data term at a denoised image estimate~\citep{aali2025adps}. Let $v_b$ be the candidate used by the released loop for batch member $b$, and let $d_b=y_b-\bar y(v_b)$. With $Q$ from (\ref{eq:pal-relative-energy}), UAL substitutes
\begin{equation}
 E_b=E_{\rm MRI}(v_b;y_b)=\sigma_{\rm R}^2Q(v_b;y_b)
 \label{eq:mri-energy}
\end{equation}
for the squared-residual energy. The factor $\sigma_{\rm R}^2$ puts the energy in the units used by the published update; when coil-map covariance is zero, $E_b=\|d_b\|^2$. Both the energy derivative and the released guidance normalization use this substituted energy:
\[
 g=\nabla_{x_{\rm cur}}\sum_bE_b,\qquad
 x_{{\rm next},b}=\widehat x_b+h\,d_{{\rm cur},b}-\frac{l_{\rm ss}}{\sqrt{E_b}}g_b.
\]
Here $l_{\rm ss}$ is the guidance scale and $d_{\rm cur}$ the diffusion drift. The derivative includes the dependence of $\bar y(v_b)$ and $S(v_b)$ on the candidate image. Algorithm~\ref{alg:mri-adps} defines the diffusion quantities and evaluation points. Since $g_b/\sqrt{E_b}=2\nabla\sqrt{E_b}$, the guidance descends the square-root fidelity $\sqrt{E_b}=\sigma_{\rm R}\sqrt Q$, as in (\ref{eq:blind-fidelity}).

The pair uses 500 diffusion steps and zero churn.

\paragraph{Sampling loop.}
Algorithm~\ref{alg:mri-adps} gives the Ambient DPS sampling loop with UAL\@. The diffusion scaling and noise schedules are $s(t)$ and $\sigma(t)$.
 The data term is evaluated at $v_b=E_{\rm xstart,b}$, using the DPS or ALD definition in the algorithm. The training and inference operators are $\mathcal A_{\mathrm{train}}$ and $\mathcal A_{\mathrm{inf}}$. The returned reconstruction is $x_{\mathrm{cur}}$.

\begin{algorithm}[!t]
\small
\color{black}
\caption{Ambient DPS MRI sampling loop from the official release~\citep{aali2025adps}, with UAL replacement lines.}
\label{alg:mri-adps}

\PALlegend
\begin{algorithmic}[1]
\REQUIRE Released time grid $\{t_i\}_{i=0}^{N}$, $s,\sigma$, network $D_\theta$,
 $\mathcal A_{\mathrm{train}},\mathcal A_{\mathrm{inf}}$, corruption mask,
 measurement $y$, Gaussian latents, and the released churn and guidance settings
\STATE $x_{\mathrm{next}}\leftarrow
 \texttt{latents}\,\sigma(t_0)s(t_0)$
\FOR{$i=0,\ldots,N-1$}
\STATE $t_{\mathrm{cur}}\leftarrow t_i$,
 $t_{\mathrm{next}}\leftarrow t_{i+1}$,
 $x_{\mathrm{cur}}\leftarrow x_{\mathrm{next}}$
\STATE $\gamma\leftarrow
 \min(S_{\mathrm{churn}}/N,\sqrt2-1)$ if
 $S_{\min}\le\sigma(t_{\mathrm{cur}})\le S_{\max}$, else $0$
\STATE $\widehat t\leftarrow\sigma^{-1}
 \bigl(\texttt{round\_sigma}((1+\gamma)\sigma(t_{\mathrm{cur}}))\bigr)$
\STATE Draw $\epsilon\sim\mathcal N(0,I)$
\STATE $\widehat x\leftarrow
 \dfrac{s(\widehat t)}{s(t_{\mathrm{cur}})}x_{\mathrm{cur}}
 +s(\widehat t)S_{\mathrm{noise}}
 \sqrt{[\sigma(\widehat t)^2-\sigma(t_{\mathrm{cur}})^2]_+}\,\epsilon$
\STATE $h\leftarrow t_{\mathrm{next}}-\widehat t$
\IF{the network was trained on fully sampled data}
\STATE $\texttt{denoised}\leftarrow
 D_\theta(\widehat x/s(\widehat t),\sigma(\widehat t))$
\ELSE
\STATE $\texttt{masked}\leftarrow
 \mathcal A_{\mathrm{train}}^*\mathcal A_{\mathrm{train}}(\widehat x)$
\STATE Stack its real and imaginary channels and concatenate the corruption mask to form $\texttt{net\_input}$
\STATE $\texttt{denoised}\leftarrow
 D_\theta(\texttt{net\_input}/s(\widehat t),\sigma(\widehat t))$,
 retaining the released image-output channels
\ENDIF
\STATE $d_{\mathrm{cur}}\leftarrow
 \left(\dfrac{\dot\sigma(\widehat t)}{\sigma(\widehat t)}
       +\dfrac{\dot s(\widehat t)}{s(\widehat t)}\right)\widehat x
 -\dfrac{\dot\sigma(\widehat t)s(\widehat t)}{\sigma(\widehat t)}
  \texttt{denoised}$
\IF{$\texttt{l\_type}=\texttt{DPS}$}
\STATE $E_{\mathrm{xstart}}\leftarrow
 \dfrac{x_{\mathrm{cur}}+s(t_{\mathrm{cur}})^2
       (\texttt{denoised}-x_{\mathrm{cur}})}{s(t_{\mathrm{cur}})}$
\ELSE
\STATE $E_{\mathrm{xstart}}\leftarrow\texttt{denoised}$
\ENDIF
\STATE \PALoriginal{$E_b\leftarrow\|y_b-\mathcal A_{\mathrm{inf}}
 (E_{\mathrm{xstart},b})\|^2$ for each batch member $b$}
\PALalt{$E_b\leftarrow E_{\mathrm{MRI}}(E_{\mathrm{xstart},b};y_b)
 =\sigma_{\mathrm R}^2Q(E_{\mathrm{xstart},b};y_b)$
 for each batch member $b$ (\ref{eq:mri-energy})}
\STATE $g\leftarrow\nabla_{x_{\mathrm{cur}}}\sum_b E_b$
\STATE $x_{\mathrm{next},b}\leftarrow
 \widehat x_b+h\,d_{\mathrm{cur},b}
 -(l_{\mathrm{ss}}/\sqrt{E_b})g_b$
\STATE Detach $x_{\mathrm{next}}$ before the next iteration
\ENDFOR
\RETURN $x_{\mathrm{cur}}$
\end{algorithmic}
\end{algorithm}

\paragraph{Dataset acknowledgment.}

Data used in the preparation of this article were obtained from the NYU fastMRI Initiative database
(fastmri.med.nyu.edu)~\citep{knoll2020fastmri,zbontar2019fastmriopendatasetbenchmarks}. As such, NYU fastMRI investigators
provided data but did not participate in analysis or writing of this report. A listing of NYU fastMRI
investigators, subject to updates, can be found at fastmri.med.nyu.edu. The primary goal of fastMRI is
to test whether machine learning can aid in the reconstruction of medical images. 

\subsection{Ultrasound with sound-speed uncertainty} 

\label{app:ultrasound}

Differences between submitted and optimized sound speeds in the CUBDL acquisitions set the uncertainty scale. We propagate nonlinear sound-speed variation at a fixed pilot reconstruction to obtain the per-mode variances used by DRUS.
 
\paragraph{Data and forward model.}
We use the acquired PICMUS contrast and resolution phantoms provided with the DRUS release and its linked supplement, with 12 paired trajectories per phantom~\citep{liebgott2016picmus,10.1007/978-3-031-53767-7_19}.
Channel measurements $y_{\rm raw}$ and beamformed measurements $y_d$ are related by
\begin{equation}
 y_{\mathrm{raw}}=H({c_s})x+n,\qquad
 y_d=B_0y_{\mathrm{raw}}=B_0H({c_s})x+B_0n,\qquad B_0=B({c_{s,0}}).
 \label{eq:ultrasound-forward}
\end{equation}
UAL uses the beamformed measurement $y=y_d$ in (\ref{eq:model}).
The signal is reflectivity, the nuisance is $\psi={c_s}$, and
$\mathcal A(x,{c_s})=B_0H({c_s})x$.
The reference sound speed is ${c_{s,0}}=1540\,\mathrm{m/s}$.
When varying ${c_s}$, we keep the beamformer $B_0$, channel-data
layout, and measurement window fixed at their reference definitions.
Only $H({c_s})$ changes.

The reference beamformed operator is $H_d=B_0H({c_{s,0}})$, with singular value decomposition $H_d=U_dS_dV_d^\T$.
The published DRUS spectral step uses a scalar noise model; the UAL replacement uses per-mode variances.
The sound-speed uncertainty scale comes from the nine CUBDL 040GSE reference acquisitions~\citep{hyun2021cubdl}, MYO001--MYO006 and UFL003--UFL005 in the released workbook.\footnote{\url{https://gitlab.com/dongwoon.hyun/cubdl/-/blob/master/datasets/CUBDL_Phantom_Details.xlsx}}
For their submitted-minus-optimized speed offsets $\delta_j=1540-c_{s,j}^{\rm opt}$, we compute
\[
 \bar\delta=\frac19\sum_{j=1}^{9}\delta_j,\qquad
 \sigma_c^2=\frac18\sum_{j=1}^{9}(\delta_j-\bar\delta)^2.
\]
This gives $\sigma_c=26.9\,\mathrm{m/s}$. The offsets determine the spread; the Gaussian sound-speed model remains centered at the released value $c_{s,0}=1540\,\mathrm{m/s}$.

\paragraph{Nuisance propagation.}
For each acquired case, let $(t_k,w_k)$ be the nodes and weights of a 32-node Gauss--Hermite rule for a standard Gaussian. We retain the eight nodes with $|t_k|\le2$ and use
\[
 c_{s,k}=c_{s,0}+\sigma_ct_k,\qquad
 \pi_k=\frac{w_k}{\sum_{j:\,|t_j|\le2}w_j}.
\]
Thus $\sigma_c$ is the Gaussian scale before truncation. Sound speed changes the propagation delay $(d_{{\rm tx},p}+d_{{\rm rx},pe})/c_s$, where $d_{{\rm tx},p}$ and $d_{{\rm rx},pe}$ are the transmit and receive path lengths for pixel $p$ and receive element $e$. The pulse model converts these delays into $H(c_s)$; the beamformer, channel layout, and measurement window remain fixed. At each retained node, we evaluate the nonlinear perturbation
\begin{equation}
 \Delta_k=B_0[H(c_{s,k})-H(c_{s,0})]x_{\rm ref},\qquad
 \bar\Delta=\sum_k\pi_k\Delta_k.
 \label{eq:ultrasound-perturbation}
\end{equation}Here $x_{\rm ref}$ is a pilot reconstruction from the same beamformed measurement, produced by the released DRUS sampler and averaged over its three channels. The per-mode variances are computed once before sampling and use no ground truth.
The weighted covariance is
\[
 C_\Delta=\sum_k\pi_k(\Delta_k-\bar\Delta)(\Delta_k-\bar\Delta)^\T.
\]
These are quadrature moments, so the covariance uses the normalized weights directly. For column $u_i$ of $U_d$, the per-mode variance is
\begin{equation}
 \gamma_i^2=\gamma^2+u_i^\T C_\Delta u_i
 =\gamma^2+\sum_k\pi_k[u_i^\T(\Delta_k-\bar\Delta)]^2.
 \label{eq:ultrasound-variance}
\end{equation}
Here $\gamma$ is the released scalar noise standard deviation, including its $\sqrt3$ factor for the three-channel representation. The same spectral variances are used in each channel. The perturbations retain the nonlinear dependence on sound speed.

\paragraph{DRUS update.} 
 
For each image channel, write $\widetilde y=U_d^\T y_d$ and $z=V_d^\T x$, and let $s_i$ be the singular values of $H_d$. The spectral Gaussian data term is
\[
 \ell_{\rm US}(x;y_d)=\frac12\sum_i\left[
 \frac{(\widetilde y_i-s_i z_i)^2}{\gamma_i^2}+\log\gamma_i^2\right]+\mathrm{const}.
\]
DRUS retains the nominal mean response $H_dx$ and replaces its scalar noise level $\gamma$ with $\gamma_i$ in initialization, mode-selection thresholds, residual weights, and sampling variances (Algorithm~\ref{alg:drus-loop}). The diffusion prior and remaining sampling operations are unchanged.

\paragraph{Spectral sampling loop.}
Algorithm~\ref{alg:drus-loop} gives the DRUS spectral updates~\citep{10.1007/978-3-031-53767-7_19}. It uses $H_d=U_dS_dV_d^\T$, with singular values $s_l$. Its scalar noise level $\sigma_0$ corresponds to $\gamma$. Here $\eta_C$ controls noise mixing in the default spectral update, $\eta_A$ controls the residual/noise mixture on $I_A$, and $\eta_B$ controls the measurement blend and sampling variance on $I_B$. Unmeasured coordinates retain the published update, and independent Gaussian draws retain their published ordering.

\begin{algorithm}[!t]
\small
\color{black}
\caption{DRUS spectral sampling loop from the official release~\citep{10.1007/978-3-031-53767-7_19}, with UAL replacement lines.}
\label{alg:drus-loop}

\PALlegend
\begin{algorithmic}[1]
\REQUIRE Gaussian input $\xi$, released time sequence \texttt{seq},
 noise predictor, $\bar\alpha_t$, $U_d,S_d,V_d$, measurement $y_d$,
 $\sigma_0,\eta_A,\eta_B,\eta_C$
\STATE $\widetilde y\leftarrow U_d^\T y_d$,
 $\sigma_T\leftarrow\sqrt{(1-\bar\alpha_{\texttt{seq[-1]}})/
 \bar\alpha_{\texttt{seq[-1]}}}$
\STATE \PALoriginal{Set $a_l\leftarrow\widetilde y_l/s_l$ and $b_l\leftarrow\sigma_0/s_l$
 where $s_l\sigma_T>\sigma_0$; set both to zero elsewhere}
\PALalt{Set $a_l\leftarrow\widetilde y_l/s_l$ and $b_l\leftarrow\gamma_l/s_l$
 where $s_l\sigma_T>\gamma_l$; set both to zero elsewhere (\ref{eq:ultrasound-variance})}
\STATE $x\leftarrow V_d\bigl(a+\sqrt{[\sigma_T^2-b^2]_+}\odot\xi\bigr)/\sigma_T$
\STATE $\texttt{xs}\leftarrow[x]$; $\texttt{x0\_preds}\leftarrow[\,]$;
 $\texttt{seq\_next}\leftarrow[-1]+\texttt{seq[:-1]}$
\FOR{$(i,j)$ in \texttt{zip(reversed(seq),reversed(seq\_next))}}
\STATE $\epsilon\leftarrow$ released noise predictor at $(\texttt{xs[-1]},i)$,
 retaining its image-noise channels
\STATE $\widehat x_0\leftarrow
 (\texttt{xs[-1]}-\sqrt{1-\bar\alpha_i}\,\epsilon)/\sqrt{\bar\alpha_i}$
\STATE $\sigma_{\mathrm{next}}\leftarrow
 \sqrt{(1-\bar\alpha_j)/\bar\alpha_j}$, with $\bar\alpha_{-1}=1$
\STATE $v\leftarrow V_d^\T\widehat x_0$
\STATE \PALoriginal{$I_A\leftarrow\{l:s_l\sigma_{\mathrm{next}}<\sigma_0\}$,
 $I_B\leftarrow\{l:s_l\sigma_{\mathrm{next}}>\sigma_0\}$}
\PALalt{$I_A\leftarrow\{l:s_l\sigma_{\mathrm{next}}<\gamma_l\}$,
 $I_B\leftarrow\{l:s_l\sigma_{\mathrm{next}}>\gamma_l\}$ (\ref{eq:ultrasound-variance})}
\STATE Draw a full signal-coordinate Gaussian vector $\xi_C$
\STATE $\overline x_{\mathrm{next}}\leftarrow v+
 \sigma_{\mathrm{next}}\sqrt{1-\eta_C^2}\,V_d^\T\epsilon+
 \sigma_{\mathrm{next}}\eta_C\xi_C$
\STATE Draw $\xi_A$ on $I_A$ and overwrite, for $l\in I_A$,
\STATE \PALoriginal{$\displaystyle\quad
 \overline x_{\mathrm{next},l}\leftarrow v_l+
 \sigma_{\mathrm{next}}\sqrt{1-\eta_A^2}\,
       \frac{\widetilde y_l-s_lv_l}{\sigma_0}
 +\sigma_{\mathrm{next}}\eta_A\xi_{A,l}$}
\PALalt{$\displaystyle\quad
 \overline x_{\mathrm{next},l}\leftarrow v_l+
 \sigma_{\mathrm{next}}\sqrt{1-\eta_A^2}\,
       \frac{\widetilde y_l-s_lv_l}{\gamma_l}
 +\sigma_{\mathrm{next}}\eta_A\xi_{A,l}$ (\ref{eq:ultrasound-variance})}
\STATE Draw a full measurement-coordinate Gaussian vector $\xi_B$
 and overwrite, for $l\in I_B$,
\STATE \PALoriginal{$\displaystyle\quad
 \overline x_{\mathrm{next},l}\leftarrow
 \eta_B\frac{\widetilde y_l}{s_l}+(1-\eta_B)v_l+
 \sqrt{\sigma_{\mathrm{next}}^2-\eta_B^2\sigma_0^2/s_l^2}\,\xi_{B,l}$}
\PALalt{$\displaystyle\quad
 \overline x_{\mathrm{next},l}\leftarrow
 \eta_B\frac{\widetilde y_l}{s_l}+(1-\eta_B)v_l+
 \sqrt{\sigma_{\mathrm{next}}^2-\eta_B^2\gamma_l^2/s_l^2}\,\xi_{B,l}$ (\ref{eq:ultrasound-variance})}
\STATE $x_{\mathrm{next}}\leftarrow\sqrt{\bar\alpha_j}\,V_d\overline x_{\mathrm{next}}$
\STATE Append $\widehat x_0$ to \texttt{x0\_preds} and $x_{\mathrm{next}}$ to \texttt{xs}
\ENDFOR
\RETURN \texttt{xs}, \texttt{x0\_preds}; the released reconstruction is \texttt{xs[-1]}
\end{algorithmic}
\end{algorithm}

\paragraph{Dataset terms.}
The PICMUS organizers permit use of the data and code and request
acknowledgment of the challenge and citation of its proceedings
paper~\citep{liebgott2016picmus}.\footnote{\url{https://www.creatis.insa-lyon.fr/Challenge/IEEE_IUS_2016/node/48}}

\subsection{Fluorescence microscopy with pupil-phase uncertainty}
\label{app:microscopy}

We reconstruct a thin two-dimensional fluorescent object from a blurred
measurement. Five uncertain pupil-phase coefficients determine the
point-spread function. Their prescribed Gaussian distribution defines the
synthetic aberrations and the uncertainty propagated by UAL\@.

\paragraph{Reference images and measurements.}
We use four stacks from the CARE Planaria test data~\citep{weigert2018care}. From the 192 available axial sections, we
select 13, 13, 12, and 12 uniformly spaced sections from the four stacks
in filename order. We take a central $256\times256$ crop from each section.
Selection uses no reconstruction scores.
The crop is mapped to $[-1,1]$ using the minimum and maximum of its full
source section, then used as the reference signal $x$ in
\begin{equation}
 y=A_{k(\psi^\star)}x+e,\qquad
 e\sim\mathcal N(0,\sigma_e^2I),\qquad \sigma_e=0.02,
 \label{eq:microscopy-forward}
\end{equation}
where $A_k$ denotes circular convolution with kernel $k$, and $\psi^\star$ contains the realized pupil-phase coefficients.
The reference is an acquired CARE image section; the aberrated measurement
is generated by this forward model.

\paragraph{Optical model.}
The scalar pupil model uses wavelength $\lambda_{\rm em}=0.6\,\mu{\rm m}$,
numerical aperture $\mathrm{NA}=1.2$, refractive index $n_{\rm med}=1.33$,
and a fixed focal offset $z_f=1\,\mu{\rm m}$.
For normalized pupil coordinates $\xi=(u,v)$ in the unit disk, the phase is
\begin{equation}
 \phi(\xi;\psi)=\sum_{j=1}^{5}\psi_j Z_{\iota_j}(\xi)
 -2\pi z_f\left[
 \sqrt{\left(\frac{n_{\rm med}}{\lambda_{\rm em}}\right)^2
       -\left(\frac{{\rm NA}}{\lambda_{\rm em}}\right)^2\|\xi\|^2}
 -\frac{n_{\rm med}}{\lambda_{\rm em}}\right].
 \label{eq:microscopy-phase}
\end{equation}
Here $\iota=(5,6,7,8,11)$ selects unit-RMS Zernike modes in Noll
ordering~\citep{noll1976zernike}, and $\psi_j$ is measured in phase radians.
Writing $r^2=u^2+v^2$, the implemented modes are
$\sqrt6(2uv)$, $\sqrt6(u^2-v^2)$, $\sqrt8\,v(3r^2-2)$,
$\sqrt8\,u(3r^2-2)$, and $\sqrt5(6r^4-6r^2+1)$.
The focal offset is fixed; it is not one of the nuisance coordinates.

Let $\Omega$ be the points of a $64\times64$ midpoint grid inside the
unit disk. A positive-exponent matrix Fourier transform gives the field
at detector positions $\mathbf r$:
\begin{equation}
 h_\psi(\mathbf r)=
 \sum_{\xi\in\Omega}
 \exp\!\bigl(i\phi(\xi;\psi)\bigr)
 \exp\!\left(2\pi i\,\frac{{\rm NA}}{\lambda_{\rm em}}
                  \mathbf r^\T\xi\right),
 \qquad
 k(\psi)[\mathbf r]=
 \frac{|h_\psi(\mathbf r)|^2}{\sum_{\mathbf r'}|h_\psi(\mathbf r')|^2}.
 \label{eq:microscopy-psf}
\end{equation}
The detector grid has $65\times65$ samples with $0.1\,\mu{\rm m}$ spacing.
The discrete kernel is normalized on this grid before circular convolution.

\paragraph{\texorpdfstring{Aberration distribution and UAL.}{Aberration distribution and UAL.}}
The five pupil-phase coefficients follow the prescribed Gaussian distribution
\begin{equation}
 q=\mathcal N(\mu,\Sigma_\psi),\qquad
 \mu=(0.10,-0.10,0.05,0,0.10)^\T,\qquad
 \Sigma_\psi=0.15^2I_5.
 \label{eq:microscopy-nuisance}
\end{equation}
This distribution generates the synthetic aberrations and provides the mean and covariance used by UAL\@.
 We linearize at $\psi_0=\mu$, so the first-order mean correction vanishes. Independent observation noise is added to each measurement.
One draw $\psi^\star\sim q$ and one noisy measurement are fixed per section and shared by both sampler versions and all 12 trajectories.

To propagate phase uncertainty, we differentiate the optical intensity and its normalization. Define $I_\psi(\mathbf r)=|h_\psi(\mathbf r)|^2$ and $T_\psi=\sum_{\mathbf r}I_\psi(\mathbf r)$. The pupil-field derivative at $\mu$ is
\[
 h_j(\mathbf r)=\left.\partial_{\psi_j}h_\psi(\mathbf r)\right|_\mu
 =\mathrm i\sum_{\xi\in\Omega}Z_{\iota_j}(\xi)e^{\mathrm i\phi(\xi;\mu)}
 e^{2\pi\mathrm i({\rm NA}/\lambda_{\rm em})\mathbf r^\T\xi}.
\]
The intensity and normalized-kernel derivatives are
\[
 \begin{aligned}
 I_j(\mathbf r)&=2\operatorname{Re}\!\left[\overline{h_\mu(\mathbf r)}h_j(\mathbf r)\right],\\
 \kappa_j(\mathbf r)&=\frac{I_j(\mathbf r)-k(\mu)[\mathbf r]\sum_{\mathbf r'}I_j(\mathbf r')}{T_\mu},
 \qquad b_j=0.15\kappa_j.
 \end{aligned}
\]
Each $\kappa_j$ sums to zero, so it changes the PSF shape while preserving unit sum. For $a\sim\mathcal N(0,I_5)$, linearization gives $k(\mu+0.15a)\approx k(\mu)+\sum_j a_jb_j$. Under this linearization, the conditional measurement model is
\[
 y\mid x,a\sim\mathcal N\!\left(A_{k(\mu)}x+\sum_{j=1}^{5}a_jA_{b_j}x,\sigma_e^2I\right).
\]
Integrating over $a$ gives
\begin{equation}
 \bar y(x)=A_{k(\mu)}x,\qquad
 U(x)=[A_{b_1}x,\ldots,A_{b_5}x],\qquad
 S(x)=\sigma_e^2I+U(x)U(x)^\T.
 \label{eq:microscopy-pal}
\end{equation}
We prepare the reference kernel and its five scaled derivative kernels $b_j$ once.
The reduced system is $5\times5$ and is evaluated by Cholesky
factorization in double precision.

\paragraph{Released sampler and UAL update.}
Both versions use PnP-DM~\citep{wu2024pnpdm} with the ED-Diff CARE
checkpoint~\citep{luo2026eddiff} and the same mean kernel $k(\mu)$.
For candidate image $u$, with $d(u)=y-A_{k(\mu)}u$, the UAL likelihood step uses
\[
 \ell_{\rm UAL}(u;y)=\tfrac12d(u)^\T S(u)^{-1}d(u)
 +\tfrac12\log\det S(u)+\tfrac m2\log(2\pi),
\]
with $S(u)$ from (\ref{eq:microscopy-pal}). At outer iterate $x^{(k)}$, the conditional target is proportional to $\exp[-\ell_{\rm UAL}(u;y)-\|u-x^{(k)}\|^2/(2\rho_k^2)]$. Its Langevin gradient includes the covariance derivatives, as in Algorithm~\ref{alg:pnp-likelihood}.
The UAL likelihood step takes 100 steps of size $10^{-4}$ per likelihood update.

Algorithm~\ref{alg:pal-pnpdm} gives the shared outer sampler; the diffusion-prior update is unchanged from the released method.
Both versions use 100 outer iterations in total, including 40 burn-in iterations, and the variance-preserving SDE mode with 100 prior steps.
The coupling scale $\rho_k$ starts at 10, decays by a factor of $0.9$ per iteration,
and has minimum $0.3$.
For each seed, we retain the last sample returned by PnP-DM\@.
Twelve paired seeds, 0 through 11, give 600 reconstructions per version.
The realized aberration coefficients and true PSFs are used only in
measurement generation and evaluation.

\paragraph{Diffusion prior, scoring, and data access.}
The released ED-Diff model has one input channel, 64 base channels,
three residual blocks per resolution level, and a 1,000-step linear
diffusion schedule.
For each prior update, we map the requested noise level to the nearest level on this schedule and convert the pretrained noise prediction to a clean-signal estimate, clipped to $[-1,1]$.
No retraining is performed.
PSNR and SSIM use data range 2 and are computed per trajectory before
averaging within sections and across the 50 sections.
The reference sections come from CARE's original test stacks. The released ED-Diff prior is pretrained on high-SNR images~\citep{luo2026eddiff}.

CARE data are distributed through the Edmond repository,
\href{https://doi.org/10.17617/3.FDFZOF}{doi:10.17617/3.FDFZOF},
whose metadata specify CC0 1.0.
The ED-Diff code and checkpoint are obtained through the
\href{https://github.com/cbmi-group/ED-Diff}{authors' release}.
The examined release does not specify a separate redistribution license
for the checkpoint.

\subsection{Evaluation metrics and reporting}
\label{app:metrics}

Reported 95\% intervals use 20,000 paired bootstrap resamples of images for deblurring, detector rows within the same acquisition for CT, and scans for MRI, after computing each case's reconstruction and ensemble scores.
Ablation error bars are standard errors across the ablation set. The microscopy whiskers in Figure~\ref{fig:teaser} resample the 50 sections in pairs and describe variation within these sections. Per-stack gains provide a separate descriptive comparison. Ultrasound whiskers in Figure~\ref{fig:teaser} show one standard error of the paired differences across the 12 contrast-phantom trajectories. Following the DRUS evaluation protocol, we report contrast-to-noise ratio (CNR), generalized contrast-to-noise ratio (gCNR), and speckle SNR on the contrast phantom, and axial and lateral full width at half maximum (FWHM) on the resolution phantom.

For one case, let $\widehat x^{(1)},\ldots,\widehat x^{(R)}\in\mathbb R^n$
be the vectorized reconstructions in scoring coordinates with $R\ge2$, and let $x^\star$
be the reference. We distinguish the average trajectory score from the
score of the ensemble-mean reconstruction:
\begin{equation}
 \overline{\operatorname{metric}}
 =\frac1R\sum_{r=1}^R
 \operatorname{metric}(\widehat x^{(r)},x^\star),
 \qquad
 \operatorname{metric}_{\mathrm{ens}}
 =\operatorname{metric}(\widehat x_{\mathrm{ens}},x^\star).
 \label{eq:metric-aggregation}
\end{equation}
This distinction applies to both PSNR and SSIM\@. For the real-valued tasks,
$\widehat x_{\mathrm{ens}}=R^{-1}\sum_r\widehat x^{(r)}$.
For MRI, individual magnitude reconstructions are used for per-trajectory PSNR and SSIM and for the three fair ensemble scores. For ensemble-mean PSNR and SSIM, we first average the complex reconstructions and then take the magnitude.
Case-level scores are averaged over the reported cases.

For a reconstruction $z$, we use
\begin{equation}
 \operatorname{MSE}(z,x^\star)=\frac1n\|z-x^\star\|_2^2,
 \qquad
 \operatorname{PSNR}(z,x^\star)
 =10\log_{10}\frac{D^2}{\operatorname{MSE}(z,x^\star)},
 \label{eq:metric-psnr}
\end{equation}where $D$ is the task-specific PSNR peak.
If $\overline\Delta_{\mathrm{PSNR}}$ is the mean paired PSNR gain,
the geometric mean of the UAL-to-original MSE ratios is
$10^{-\overline\Delta_{\mathrm{PSNR}}/10}$.
We use this relation for the relative MSE reductions in the Abstract and Introduction.
For deblurring and microscopy, PSNR and SSIM use the data range 2 for intensities in $[-1,1]$.
CT uses the reference-image range for both; MRI uses the reference-image maximum for PSNR and its range for SSIM\@. SSIM uses the scikit-image default
window settings; color-image scores are averaged over channels.

The fair finite-ensemble estimators of the continuous ranked probability score (CRPS) and the energy score~\citep{gneiting2007strictly,ferro2014fair}
use off-diagonal ensemble pairs:
\begin{align}
 \operatorname{CRPS}_{\mathrm f}
 &=\frac1n\sum_{i=1}^n
 \left[
 \frac1R\sum_{r=1}^R|\widehat x_i^{(r)}-x_i^\star|
 -\frac1{R(R-1)}\sum_{r<s}
 |\widehat x_i^{(r)}-\widehat x_i^{(s)}|
 \right],
 \label{eq:metric-crps}\\
 \operatorname{ES}_{\mathrm f}
 &=\frac1R\sum_{r=1}^R\|\widehat x^{(r)}-x^\star\|_2
 -\frac1{R(R-1)}\sum_{r<s}
 \|\widehat x^{(r)}-\widehat x^{(s)}\|_2.
 \label{eq:metric-energy}
\end{align}
The energy score uses the Euclidean norm of the full image vector,
without an additional dimension normalization.

For the local variogram score, define the pair features
$v_{ij}^{(r)}=|\widehat x_i^{(r)}-\widehat x_j^{(r)}|^{1/2}$ and
$v_{ij}^\star=|x_i^\star-x_j^\star|^{1/2}$
using the variogram construction of \citet{scheuerer2015variogram}.
Let $\bar v_{ij}=R^{-1}\sum_r v_{ij}^{(r)}$. 
For each offset $g$, let $\mathcal P_g$ contain the valid within-channel
pixel pairs at that offset. We use lags 1, 2, 4, 8, and 16 in the
horizontal, vertical, and two diagonal directions, giving 20 offset groups.
We use the fair estimator
\begin{equation}
 \operatorname{LVS}_{\mathrm f}
 =\frac1{20}\sum_g\frac1{|\mathcal P_g|}
 \sum_{(i,j)\in\mathcal P_g}
 \left[
  (\bar v_{ij}-v_{ij}^\star)^2
  -\frac1{R(R-1)}
   \sum_{r=1}^R(v_{ij}^{(r)}-\bar v_{ij})^2
 \right].
 \label{eq:metric-variogram}
\end{equation}
Thus valid pairs and channels are averaged within each offset group,
and the groups receive equal weight. Lower values of all three
ensemble scores are better. Metrics are computed in double precision, without additional clipping or rescaling.

\begin{table}[htbp]
\centering
\small
\caption[Primary paired-gain intervals]{Pointwise 95\% paired-bootstrap confidence intervals
for the primary differences in Table~\ref{tab:exp-deblurring}.
PSNR gain is UAL minus Published for mean per-trajectory PSNR (dB);
CRPS reduction is Published minus UAL, scaled by $10^3$.}
\label{tab:primary-gain-intervals}
\begin{tabular}{@{}lll@{}}
\toprule
Sampler / task & PSNR gain 95\% CI & CRPS reduction 95\% CI $\times10^3$ \\
\midrule
\multicolumn{3}{l}{Gaussian blur: mean width 3 pixels, standard deviation 0.5 pixels} \\
DAPS & $[0.20,0.61]$ & $[-0.10,1.65]$ \\
PnP-DM & $[.26,.84]$ & $[1.79,3.49]$ \\
DiffPIR & $[.59,1.46]$ & $[2.35,5.31]$ \\
BlindDPS & $[.80,.99]$ & $[1.91,2.51]$ \\
\midrule
\multicolumn{3}{l}{Motion blur} \\
DAPS & $[10.48,11.85]$ & $[97.27,137.38]$ \\
PnP-DM & $[8.88,10.13]$ & $[95.12,133.09]$ \\
DiffPIR & $[2.50,3.83]$ & $[85.82,108.75]$ \\
BlindDPS & $[-0.27,-0.07]$ & $[0.80,2.70]$ \\
Fast Diffusion EM & $[5.71,6.94]$ & $[45.31,56.27]$ \\
\midrule
\multicolumn{3}{l}{Imaging applications} \\
CT & $[.173,.236]$ & $[.516,.671]$ \\
MRI & $[0.0097,0.0160]$ & $[0.0239,0.0435]$ \\
Microscopy & $[0.238,0.649]$ & $[0.671,2.426]$ \\
\bottomrule
\end{tabular}
\end{table}

\section{Additional experimental results}
\label{app:additional-results}
This appendix gives the supporting comparisons and full metric summaries for Section~\ref{sec:experiments}.

\subsection{Numerical likelihood reference}
\label{app:numerical-reference}
For Figure~\ref{fig:likelihood-approximation}, we integrate the nonlinear Gaussian-blur likelihood over two independent positive-truncated-normal width distributions, with mean width 3 pixels and standard deviations $\tau\in\{0.25,0.5,1,1.5\}$ pixels. At fixed candidate $x$ and measurement $y$, we center adaptive tensor-product Gauss--Hermite quadrature at a mode of $p(\psi\mid x,y)\propto q(\psi)p(y\mid x,\psi)$. We find the mode by L-BFGS-B with positive width bounds, starting from $(3,3)$ and $(3+\tau,3-\tau)$ pixels and retaining the solution with the lower negative log posterior. The Gaussian quadrature proposal $h$ has covariance given by the inverse of the symmetrized Hessian of $-\log[q(\psi)p(y\mid x,\psi)]$, computed by centered finite differences of the nuisance gradient at the mode. The transformed quadrature weights include $q(\psi)/h(\psi)$, with zero weight outside the positive-width support. The likelihood calculations use the same kernel grid, boundary convention, crop, and noise standard deviation $0.02$ as the paired reconstructions.

Writing $g=\nabla_x\log p(y\mid x)$ and $\ell=-\log p(y\mid x)$, the plotted errors are $100\|g-g_{\rm ref}\|_2/\|g_{\rm ref}\|_2$ percent and $|\ell-\ell_{\rm ref}|/m$ nats per measurement, with $m=3\times256\times256$. The reference score is the nuisance-posterior-weighted average of the conditional image scores. At each uncertainty level, all three likelihoods use the same candidate and observation.
 We average errors over trajectories within each image before computing the mean and SE across images.

We checked numerical convergence for all 768 final candidate--observation pairs in Figure~\ref{fig:likelihood-approximation}a,b by increasing the quadrature order from 5 to 9 per width coordinate.
 The largest observed changes were $5.82\times10^{-11}$ nats in energy and $5.77\times10^{-12}$ in the relative $\ell_2$ score change. These changes with quadrature order are far below the plotted approximation errors.
To compare likelihood accuracy during sampling, we reran the published PnP-DM sampler on the 16 ablation images with seed~0 at each $\tau$. We saved its prior-step outputs at zero-based outer iterations 0, 25, 50, 74, and 99 of 100 and evaluated all three likelihoods at these same candidates against the quadrature reference.
Relative to the original final reconstructions, the reruns change the measured score error by at most 0.0121 percentage points and the energy error by at most $6.82\times10^{-5}$ nats per measurement.
 Table~\ref{tab:figure2-trajectory} reports the relative score error; UAL also has the lowest mean energy error at every sampling state and uncertainty level. At the first saved state, UAL\textquoteright s mean relative image-score error ranges from 17.5\% to 19.9\% across uncertainty levels.

\begin{table}[htbp]
\centering
\small
\caption{Relative image-score error (percent) at five saved PnP-DM sampling states in Gaussian deblurring.
Entries are mean $\pm$ one SE across 16 images, with one trajectory per image at each $\tau$.
Candidates are prior-step outputs from Published; iteration indices are zero-based.
Figure~\ref{fig:likelihood-approximation}c labels the sampling states nominally as 1, 25, 50, 75, and 100\@.
Numerical marginalization in Appendix~\ref{app:numerical-reference} provides the reference.}
\label{tab:figure2-trajectory}
\begin{tabular}{@{}llrrrrr@{}}
\toprule
$\tau$ & Likelihood & iter.\ 0 & iter.\ 25 & iter.\ 50 & iter.\ 74 & iter.\ 99 \\
\midrule
0.25 & Published & $36\pm5$ & $8.1\pm0.9$ & $6.1\pm0.7$ & $6.0\pm0.6$ & $6.2\pm0.6$ \\
 & Isotropic & $33\pm5$ & $7.8\pm0.7$ & $6.9\pm0.5$ & $6.8\pm0.5$ & $6.9\pm0.5$ \\
 & UAL & $17\pm4$ & $0.35\pm0.06$ & $0.14\pm0.03$ & $0.11\pm0.01$ & $0.12\pm0.01$ \\
\addlinespace[2pt]
0.5 & Published & $37\pm5$ & $11\pm1$ & $8.2\pm0.9$ & $7.9\pm0.8$ & $8.3\pm0.8$ \\
 & Isotropic & $34\pm4$ & $16\pm1$ & $17\pm1$ & $17\pm1$ & $17\pm1$ \\
 & UAL & $18\pm4$ & $0.58\pm0.08$ & $0.24\pm0.04$ & $0.21\pm0.03$ & $0.22\pm0.03$ \\
\addlinespace[2pt]
1 & Published & $38\pm5$ & $15\pm2$ & $12\pm1$ & $11\pm1$ & $11\pm1$ \\
 & Isotropic & $58\pm2$ & $40\pm2$ & $42\pm3$ & $42\pm3$ & $42\pm3$ \\
 & UAL & $19\pm4$ & $1.2\pm0.2$ & $0.51\pm0.07$ & $0.46\pm0.07$ & $0.43\pm0.07$ \\
\addlinespace[2pt]
1.5 & Published & $39\pm5$ & $19\pm2$ & $15\pm2$ & $15\pm2$ & $15\pm2$ \\
 & Isotropic & $80\pm3$ & $58\pm3$ & $61\pm4$ & $62\pm4$ & $62\pm4$ \\
 & UAL & $20\pm4$ & $2.0\pm0.4$ & $0.97\pm0.16$ & $0.99\pm0.23$ & $0.91\pm0.16$ \\
\bottomrule
\end{tabular}
\end{table}
\subsection{\texorpdfstring{Deblurring comparisons}{Deblurring comparisons}}
At $(\bar s,\tau)=(3,0.5)$, UAL increases per-trajectory PSNR for all four
Gaussian-deblurring samplers (Table~\ref{tab:exp-deblurring}).
The gains range from $0.39$ dB for DAPS, with a 95\% interval of
$[0.20,0.61]$, to $0.99$ dB for DiffPIR, with interval $[0.59,1.46]$.
The fraction of images with higher mean per-trajectory PSNR ranges from 66\% for DAPS to 96\% for DiffPIR and BlindDPS\@.
Mean CRPS decreases for all four samplers; Table~\ref{tab:primary-gain-intervals} reports the paired confidence intervals.

In motion deblurring, UAL raises mean per-trajectory PSNR from 17.68 to 28.84 dB for DAPS, from 17.04 to 26.54 dB for PnP-DM, from 11.44 to 14.55 dB for DiffPIR, and from 22.76 to 29.10 dB for Fast Diffusion EM\@. With the true measurement kernel, the published PnP-DM sampler gives 27.65 dB (Table~\ref{tab:pnpdm-motion-oracle}).
 UAL closes 89.5\% of the per-trajectory PSNR gap from Published to this oracle reference, with 1.11 dB remaining. For ensemble-mean PSNR, UAL closes 90.3\% of the gap. For DAPS, the oracle gives 30.65 dB mean per-trajectory PSNR; UAL closes 86.0\% of the Published-to-Oracle PSNR gap, with 1.81 dB remaining. With the true kernel, the published DiffPIR sampler gives 30.90 dB mean per-trajectory PSNR; UAL closes 16.0\% of its Published-to-Oracle PSNR gap (Table~\ref{tab:pnpdm-motion-oracle}).

For BlindDPS in motion deblurring, UAL improves all three ensemble scores and raises ensemble-mean PSNR from 23.21 to 23.33 dB\@.
Mean per-trajectory PSNR is 21.95 dB versus 22.12 dB for Published, with corresponding SSIM values of 0.5449 and 0.5559\@.

For DAPS, PnP-DM, and DiffPIR in motion deblurring, UAL raises mean per-trajectory PSNR on every image.
 Figures~\ref{fig:appendix-gaussian} and~\ref{fig:appendix-motion} show selected individual reconstructions for the same three images. 

\begin{table}[htbp]
\centering
\small
\setlength{\tabcolsep}{3pt}
\caption[Motion deblurring with true-kernel references]{DAPS, PnP-DM, and DiffPIR motion deblurring on 50 images with 12 paired trajectories per image. Published uses the population-mean kernel; Oracle uses the true measurement kernel with the published sampler. UAL uses $P=256$ covariance modes. T/M denotes mean per-trajectory score / score of the ensemble mean. PSNR is in dB; LVS is scaled by $10^3$.}
\label{tab:pnpdm-motion-oracle}
\begin{tabular*}{\linewidth}{@{\extracolsep{\fill}}llrrrrr@{}}
\toprule
Sampler & Treatment & PSNR T/M $\uparrow$ & SSIM T/M $\uparrow$ & CRPS $\downarrow$ & Energy $\downarrow$ & LVS $\downarrow$ \\
\midrule
DAPS & Published & 17.68 / 18.10 & .3303 / .3779 & .15043 & 96.013 & 57.076 \\
 & UAL & 28.84 / 29.58 & .7103 / .7726 & .03502 & 23.388 & 12.774 \\
 & Oracle & 30.65 / 32.50 & .7698 / .8537 & .02265 & 15.048 & 8.783 \\
\addlinespace[2pt]
PnP-DM & Published & 17.04 / 17.90 & .2766 / .3897 & .14485 & 91.587 & 61.032 \\
 & UAL & 26.54 / 29.43 & .5552 / .7628 & .03236 & 20.943 & 14.959 \\
 & Oracle & 27.65 / 30.68 & .5752 / .7816 & .02878 & 17.908 & 13.728 \\
\addlinespace[2pt]
DiffPIR & Published & 11.44 / 12.17 & .0701 / .0877 & .34406 & 172.270 & 179.800 \\
 & UAL & 14.55 / 15.19 & .1440 / .1666 & .24688 & 130.857 & 100.688 \\
 & Oracle & 30.90 / 33.27 & .7936 / .8738 & .02041 & 13.380 & 7.447 \\
\bottomrule
\end{tabular*}
\end{table}
\paragraph{Comparison with blind samplers.} In motion deblurring, we compare DAPS, PnP-DM, and Fast Diffusion EM with UAL against the published blind samplers BlindDPS and Fast Diffusion EM on the same 50 observations, including the same noise realizations. The blind samplers sample or estimate a kernel for each image; the UAL versions use the population-mean kernel and the propagated kernel covariance. Table~\ref{tab:ual-vs-blind} reports paired image-level bootstrap intervals with 20,000 resamples. All intervals favor UAL except per-trajectory SSIM for PnP-DM with UAL against BlindDPS, whose interval contains zero.
\begin{table}[!ht]
\centering
\scriptsize
\setlength{\tabcolsep}{2pt}
\caption{Paired improvements over published blind samplers in motion deblurring, 50 images with 12 trajectories each: mean [95\% paired bootstrap interval] and number of images improved. Values are UAL minus the blind baseline for PSNR and SSIM, and the blind baseline minus UAL for CRPS, energy, and LVS, so positive values denote improvement; CRPS and LVS are scaled by $10^3$. The first method in each pair uses UAL.}
\label{tab:ual-vs-blind}
\resizebox{\linewidth}{!}{%
\begin{tabular}{@{}llllllll@{}}
\toprule
Pair & PSNR M (dB) & PSNR T (dB) & SSIM M & SSIM T & CRPS & Energy & LVS \\
\midrule
DAPS with UAL vs BlindDPS & +6.37 [5.81, 6.95] 50 & +6.72 [6.25, 7.19] 50 & +0.1463 [0.1330, 0.1606] 50 & +0.1544 [0.1385, 0.1713] 50 & +28.90 [25.12, 32.81] 50 & +23.24 [20.35, 26.31] 50 & +5.52 [4.51, 6.58] 48 \\
DAPS with UAL vs Fast Diffusion EM & +6.31 [5.69, 6.87] 48 & +6.08 [5.53, 6.60] 49 & +0.2201 [0.1960, 0.2435] 48 & +0.1868 [0.1621, 0.2114] 48 & +42.87 [37.43, 48.27] 48 & +28.48 [24.98, 31.91] 48 & +11.50 [10.15, 12.83] 49 \\
PnP-DM with UAL vs BlindDPS & +6.22 [5.69, 6.75] 50 & +4.42 [3.98, 4.85] 49 & +0.1364 [0.1237, 0.1500] 50 & $-$0.0007 [$-$0.0199, 0.0190] 23 & +31.55 [28.30, 34.94] 50 & +25.69 [23.11, 28.33] 50 & +3.33 [2.17, 4.45] 42 \\
PnP-DM with UAL vs Fast Diffusion EM & +6.16 [5.58, 6.70] 50 & +3.78 [3.25, 4.29] 47 & +0.2103 [0.1859, 0.2338] 49 & +0.0316 [0.0052, 0.0582] 31 & +45.53 [40.55, 50.38] 50 & +30.92 [27.58, 34.11] 50 & +9.31 [7.73, 10.88] 47 \\
Fast Diffusion EM with UAL vs BlindDPS & +7.52 [6.97, 8.05] 50 & +6.98 [6.46, 7.51] 50 & +0.1845 [0.1717, 0.1981] 50 & +0.2076 [0.1941, 0.2215] 50 & +36.84 [33.50, 40.45] 50 & +28.30 [25.60, 31.16] 50 & +8.66 [7.68, 9.74] 50 \\
Fast Diffusion EM with UAL vs Fast Diffusion EM & +7.45 [6.83, 8.02] 50 & +6.34 [5.71, 6.93] 48 & +0.2583 [0.2329, 0.2832] 50 & +0.2399 [0.2150, 0.2640] 49 & +50.81 [45.32, 56.30] 50 & +33.53 [30.09, 36.86] 50 & +14.64 [13.21, 16.05] 50 \\
\bottomrule
\end{tabular}}
\end{table}

\begin{figure}[p]
\centering
\includegraphics[width=0.95\linewidth]{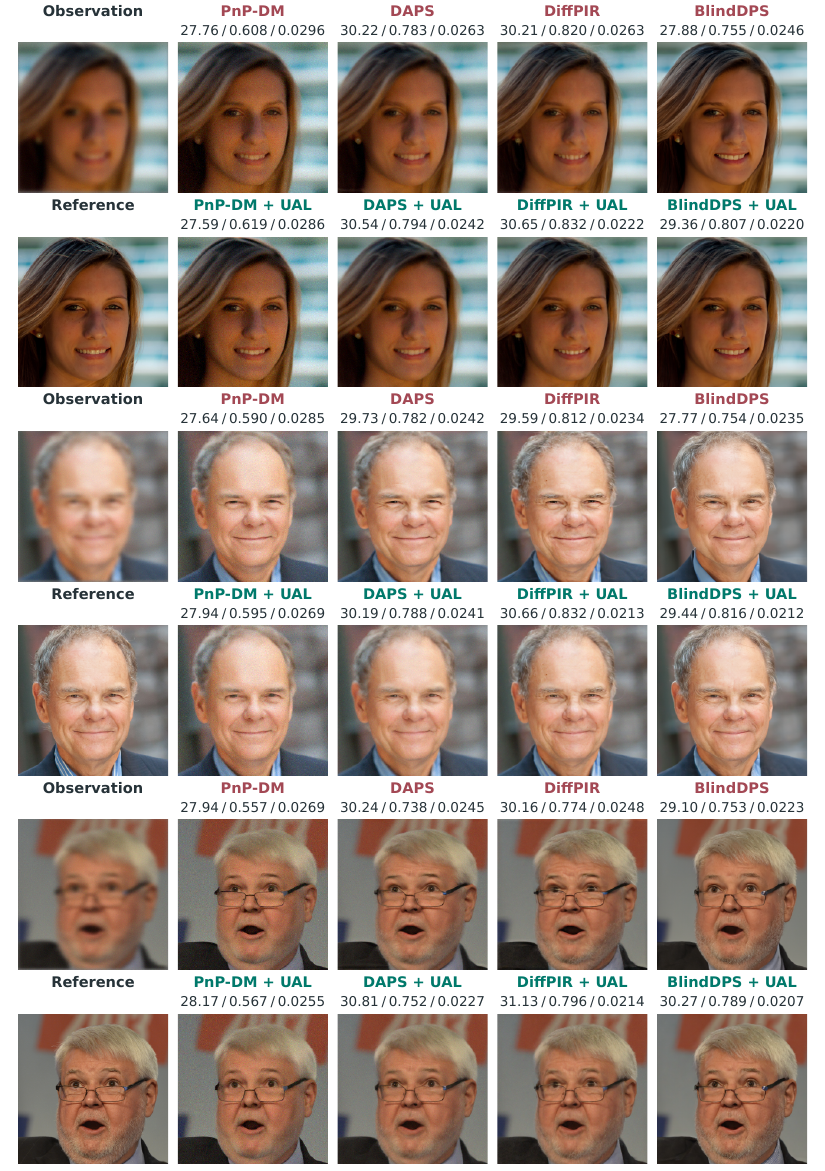}
\caption[Selected Gaussian-deblurring reconstructions]{Selected Gaussian-deblurring reconstructions at $(\bar s,\tau)=(3,0.5)$ pixels. Each row pair shows the observation and reference at left, followed by Published (upper) and UAL (lower) reconstructions from four samplers. Published and UAL use matched random seeds within each sampler. Numbers give PSNR (dB), SSIM, and CRPS, in that order. PSNR and SSIM describe the displayed trajectory; CRPS uses all 12 trajectories.}
\label{fig:appendix-gaussian}
\end{figure}

\begin{figure}[p]
\centering
\includegraphics[width=0.95\linewidth]{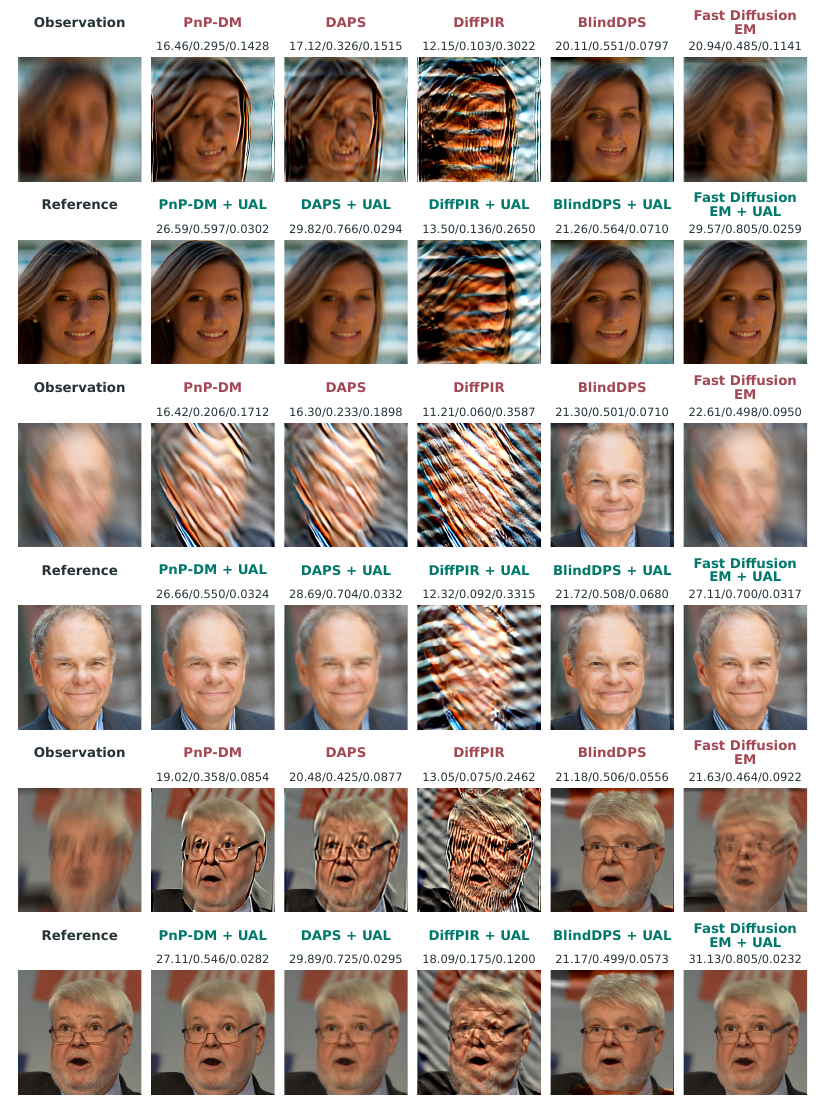}
\caption[Selected motion-deblurring reconstructions]{Selected motion-deblurring reconstructions for the same three images as Figure~\ref{fig:appendix-gaussian}, with $P=256$ covariance modes for UAL\@. Each row pair shows the observation and reference at left, followed by Published (upper) and UAL (lower) reconstructions. For each sampler, the displayed reconstructions are paired by random seed. Numbers give PSNR (dB), SSIM, and CRPS, in that order. PSNR and SSIM describe the displayed trajectory; CRPS uses all 12 trajectories.}
\label{fig:appendix-motion}
\end{figure}

\subsection{Application details}

In the main CT comparison, UAL gains $0.204$ dB in per-trajectory PSNR, with a 95\% interval of
$[0.173,0.236]$ dB, and improves SSIM and all three ensemble scores. The gain persists when both versions use the same mean flat and dark reference frames: UAL raises per-trajectory PSNR from 25.942 to 26.145 dB and lowers CRPS from 0.021955 to 0.021352\@. The paired PSNR gain is 0.2034 dB with a 95\% interval of $[0.1735,0.2356]$, and all 50 detector rows improve. Using those mean reference frames without UAL changes per-trajectory SSIM from 0.42478 to 0.42434, with a 95\% paired interval of $[-0.00070,-0.00021]$ for the change. Figure~\ref{fig:appendix-ct} shows paired reconstructions and differences from the reference.

\begin{figure}[!htbp]
\centering
\setlength{\abovecaptionskip}{4pt}
\includegraphics[width=\linewidth]{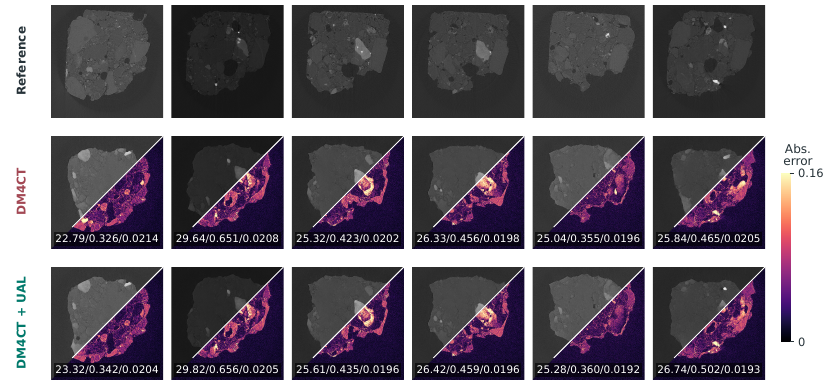}
\caption[Selected CT reconstructions]{Selected CT cases, with reconstruction draws paired by seed. Columns show cases.  Each method panel combines the reconstruction in the upper-left half with its absolute difference from the reference in the lower-right half. Difference maps share one color scale. Numbers give PSNR (dB), SSIM, and CRPS, in that order. PSNR and SSIM describe the displayed draw; CRPS uses all 12 trajectories. DM4CT denotes Published.}
\label{fig:appendix-ct}
\end{figure}

For MRI, mean per-trajectory PSNR improves on 47 of 50 scans, with a mean paired gain of $0.0127$ dB and a 95\% paired-bootstrap confidence interval of $[0.0097,0.0160]$ dB (paired $t=7.83$). The corresponding mean SSIM gain is $0.00031$. All seven reported metrics favor UAL, with their pointwise 95\% paired-bootstrap intervals excluding zero (Table~\ref{tab:mri-gain-intervals}). Figure~\ref{fig:appendix-mri} shows paired magnitude reconstructions.

\begin{table}[htbp]
\centering
\small
\caption{MRI gains and pointwise 95\% paired-bootstrap confidence intervals across 50 scans, each with 12 paired trajectories. Gains are UAL minus Published for PSNR and SSIM, and Published minus UAL for the ensemble scores. T/M follows Table~\ref{tab:exp-deblurring}.}
\label{tab:mri-gain-intervals}
\begin{tabular}{@{}lrr@{}}
\toprule
Metric & Mean gain & 95\% interval \\
\midrule
PSNR T (dB) & 0.0127 & $[0.0097,0.0160]$ \\
PSNR M (dB) & 0.0141 & $[0.0104,0.0183]$ \\
SSIM T & 0.000313 & $[0.000213,0.000420]$ \\
SSIM M & 0.000119 & $[0.000026,0.000231]$ \\
CRPS & 0.0000331 & $[0.0000239,0.0000435]$ \\
Energy & 0.0220 & $[0.0160,0.0290]$ \\
LVS ($\times10^3$) & 0.0116 & $[0.0086,0.0149]$ \\
\bottomrule
\end{tabular}
\end{table}

\begin{figure}[p]
\centering
\setlength{\abovecaptionskip}{4pt}
\includegraphics[width=\linewidth]{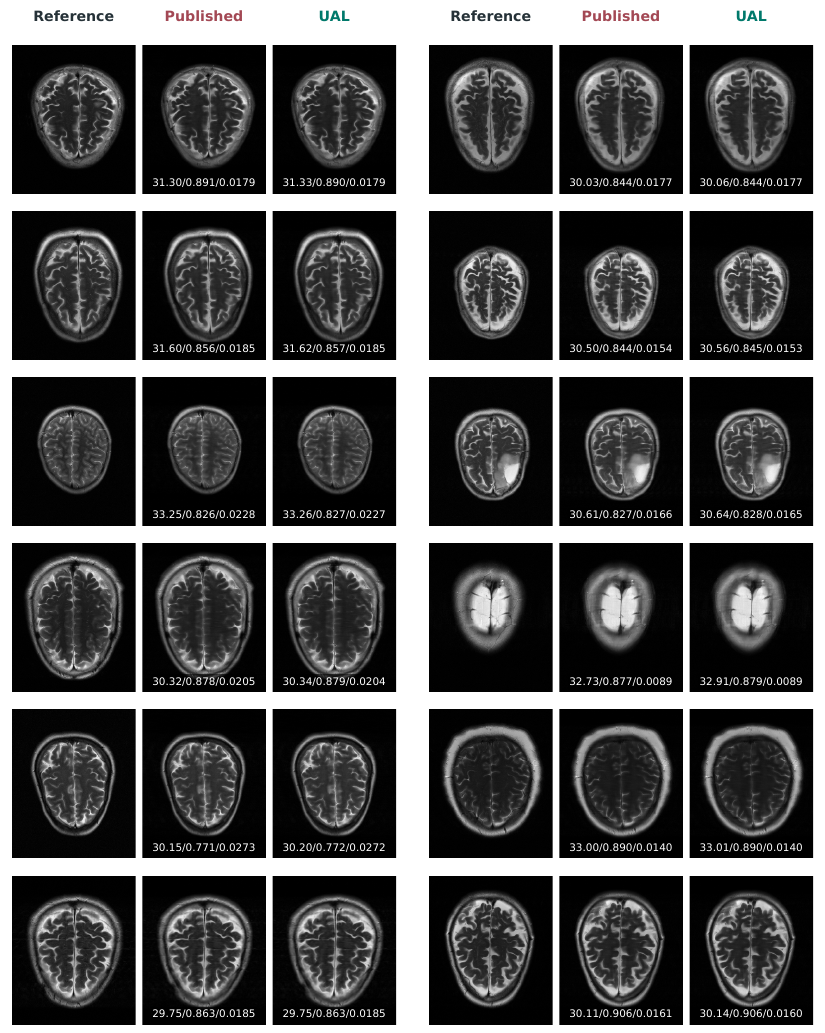}
\caption[Selected MRI reconstructions]{Magnitude reconstructions for selected MRI scans. Each triplet shows the reference, Published Ambient DPS, and Ambient DPS with UAL\@. Both methods use the same random seed for each scan. Numbers give PSNR (dB), SSIM, and CRPS, in that order. PSNR and SSIM describe the displayed draw; CRPS uses all 12 trajectories.}
\label{fig:appendix-mri}
\end{figure}

\begin{samepage}
In microscopy with five uncertain pupil-phase coefficients, UAL increases per-trajectory PSNR by $0.425$ dB and SSIM by $0.0153$.
The mean PSNR gain is positive in each of the four source stacks, ranging
from $0.251$ to $0.724$ dB\@.
Using the ensemble mean, PSNR rises from 33.67 to 34.22 dB\@.
The CRPS, energy, and local variogram scores decrease by
0.00148, 0.901, and 0.000881, respectively.
The uncertainty model and synthetic measurement construction are specified
in Appendix~\ref{app:microscopy}. Figures~\ref{fig:appendix-microscopy} and~\ref{fig:appendix-microscopy-uncertainty} show paired reconstructions and ensemble spread. For the section in Figure~\ref{fig:appendix-microscopy-uncertainty}, UAL reduces the root-mean-square error of the ensemble mean from 0.085 to 0.052, while the root-mean-square sample standard deviation changes from 0.063 to 0.059\@.

\par
\end{samepage}

\begin{figure}[H]
\centering
\setlength{\abovecaptionskip}{4pt}
\includegraphics[width=\linewidth]{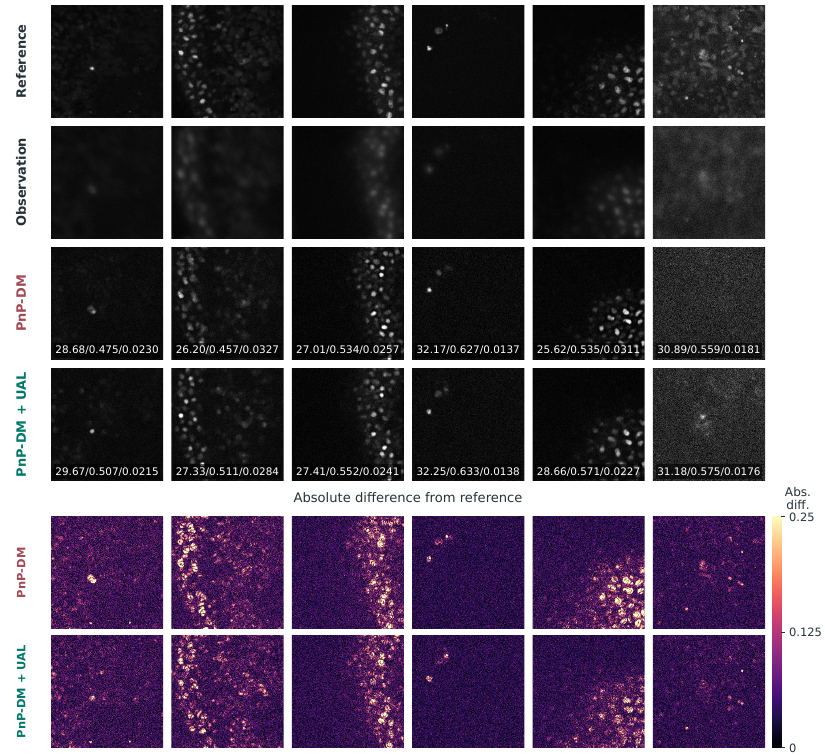}
\caption[Selected fluorescence-microscopy reconstructions]{Selected fluorescence-microscopy sections. Reconstructions are paired by seed; difference maps share one scale. Numbers give PSNR (dB), SSIM, and CRPS\@. PSNR and SSIM use the displayed draw; CRPS uses all 12 trajectories.}
\label{fig:appendix-microscopy}
\end{figure}

\begin{table}[H]
\centering
\small
\caption{Native ultrasound contrast and resolution metrics, averaged over 12 trajectories per phantom. CNR, gCNR, and speckle SNR use the contrast phantom; axial and lateral FWHM use the resolution phantom. FWHM is in mm and lower values are better. Speckle SNR is compared with the reference $1.91$. Bold marks better values.} 
\label{tab:ultrasound-results}
\begin{tabular}{@{}lrr@{}}
\toprule
 Metric & Published & UAL \\
\midrule
CNR (dB), higher & 9.9292 & \textbf{12.3625} \\
gCNR & 0.9141 & \textbf{0.9743} \\
Speckle SNR, target $1.91$ & 2.0616 & \textbf{2.0184} \\
Axial FWHM (mm) & 0.2506 & \textbf{0.2497} \\
Lateral FWHM (mm) & 0.5729 & \textbf{0.5722} \\
\bottomrule
\end{tabular}
\end{table}

On the acquired PICMUS contrast phantom, UAL increases CNR and gCNR\@. On the resolution phantom, both mean FWHM changes are below 0.01 mm (Table~\ref{tab:ultrasound-results}). Figure~\ref{fig:appendix-ultrasound} shows all 12 paired contrast-phantom reconstructions.

\begin{figure}[p]
\centering
\setlength{\abovecaptionskip}{4pt}
\includegraphics[width=\linewidth]{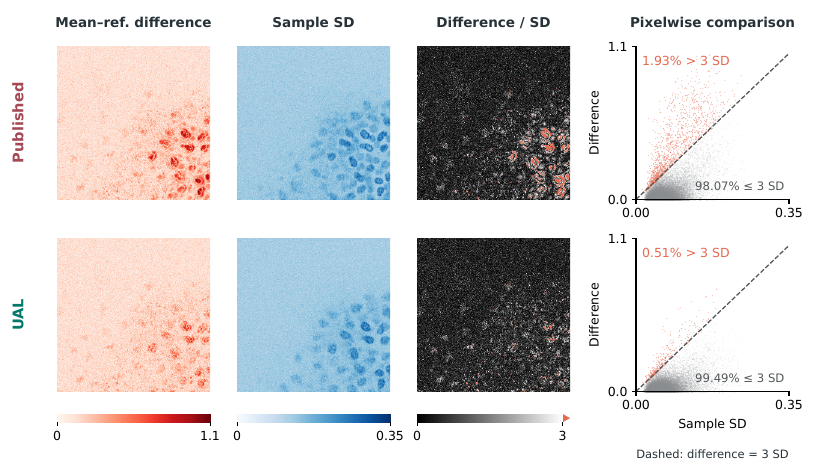}
\caption[Microscopy ensemble spread and reference differences]{
Pixelwise statistics for the microscopy section in Figure~\ref{fig:appendix-microscopy},
using all 12 paired random seeds.
The left columns show the absolute difference between the ensemble mean and reference
and the sample standard deviation (SD).
The right columns show their ratio and a scatter plot of difference against SD\@.
Red marks pixels above the 3-SD threshold, indicated by the dashed line; percentages
give the pixel fractions on either side.
UAL reduces the fraction above this threshold from 1.93\% to 0.51\%, so large
reference differences relative to the ensemble spread are less frequent in this section.
}
\label{fig:appendix-microscopy-uncertainty}
\par\vspace{8pt}
\includegraphics[width=\linewidth]{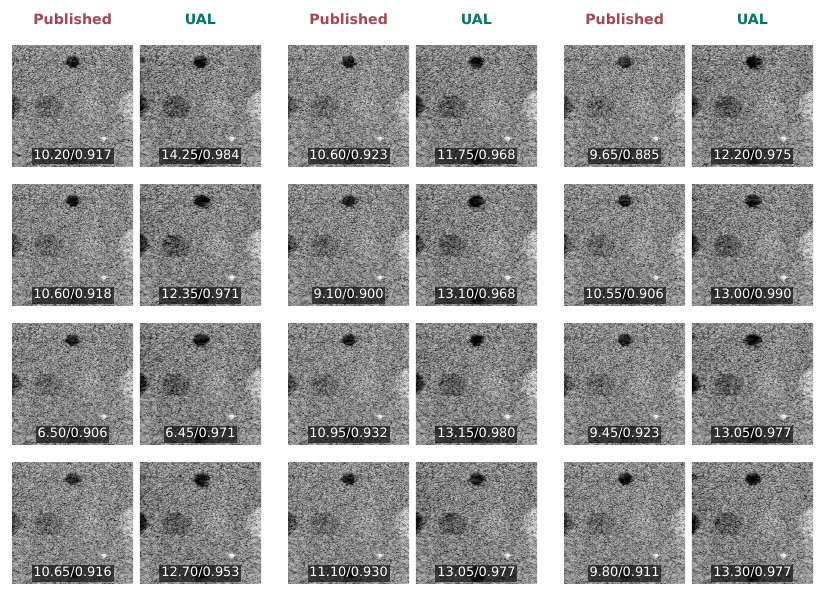}
\caption[All paired ultrasound contrast-phantom reconstructions]{Paired DRUS reconstructions of the acquired PICMUS contrast phantom. Each pair shows Published on the left and UAL on the right, ordered by seed from left to right and top to bottom. Numbers give CNR (dB) and gCNR for the displayed reconstruction. All panels use the same $[-60,0]$ dB B-mode display range.}
\label{fig:appendix-ultrasound}
\end{figure}

\clearpage
\subsection{Covariance-structure control and secondary metrics}\label{app:covariance-control}
For Gaussian deblurring with PnP-DM at $(\bar s,\tau)=(3,0.5)$, structured UAL improves reconstruction over an isotropic likelihood with the same whitened covariance trace.
The isotropic covariance is
\begin{equation}
 S_{\rm iso}(x)=[1+c(x)]\Sigma_e,\qquad
 c(x)=\frac{1}{m}\operatorname{tr}
 \bigl(\Sigma_e^{-1}J_\psi(x)\Sigma_\psi J_\psi(x)^\T\bigr),
 \label{eq:exp-isotropic}
\end{equation}
where $m$ is the measurement dimension. Here $c(x)$ is the average added variance after noise whitening, so $1+c(x)=v_\star(x)$.
Both likelihoods include Gaussian normalization and differentiate their
signal-dependent covariance.
Structured UAL gains $0.455\pm0.197$ dB over the isotropic control in
per-trajectory PSNR and reduces CRPS by $0.00260\pm0.00056$
(mean $\pm$ standard error).
The isotropic control produces little change from the published sampler
(Table~\ref{tab:exp-ablations-full}).
Structured UAL improves per-trajectory PSNR over isotropic inflation on 13 of 16 Gaussian-deblurring images.

In motion deblurring, isotropic inflation improves mean
per-trajectory PSNR by $6.212\pm0.398$ dB over the published sampler.
Structured UAL adds $2.557\pm0.317$ dB in per-trajectory PSNR and reduces CRPS by $0.02706\pm0.00336$ relative to the isotropic control. The corresponding ensemble-mean PSNR margin is $4.431\pm0.325$ dB\@.
Structured UAL improves per-trajectory PSNR over isotropic inflation on 15 of 16 motion-deblurring images.
These results support using the nuisance-induced covariance directions to weight the residual.

In Gaussian deblurring with PnP-DM, recomputing covariance preserves the reconstruction gain when its derivatives are stopped.
The stopped-derivative control removes both covariance terms in (\ref{eq:pal-score}), so it tests their combined effect. Mean per-trajectory PSNR is 26.944 dB for the published sampler, 27.506 dB with covariance derivatives stopped, and 27.475 dB with UAL\@. Fixing covariance at the published sampler's ensemble mean gives 26.964 dB, with UAL higher by $0.512\pm0.113$ dB; uncertainties are standard errors across the 16 images.
 With fixed covariance, ensemble-mean PSNR is 29.088 dB and CRPS is 0.03484, compared with 29.388 dB and 0.03253 for the published sampler. Refreshing covariance with its derivatives stopped gives 30.304 dB and 0.02976, close to 30.259 dB and 0.02981 for UAL\@. 

We also select a fixed noise-variance multiplier by mean per-trajectory PSNR\@.
For the scalar noise-variance multipliers $v=1,1.5,2,3,5$, mean per-trajectory PSNR is 26.944, 26.991, 27.007, 27.029, and 26.995 dB, respectively.
 
 UAL exceeds the selected scalar control by $0.446\pm0.123$ dB, with higher PSNR on 15 of 16 images; the uncertainty is one standard error of the paired image differences. In motion deblurring, the best tested multiplier, $v=48$, gives 23.369 dB mean per-trajectory PSNR and CRPS 0.05773, compared with 26.641 dB and 0.03149 for UAL (Table~\ref{tab:scalar-motion}). UAL gives higher mean per-trajectory PSNR on 15 of 16 images. For DAPS, multiplying the data-term variance $\tau_{\rm D}^2$ by $v=48$ gives the best tested value, 24.413 dB, compared with 28.631 dB for UAL; UAL is higher by $4.22\pm0.27$ dB (mean $\pm$ SE across the 16 images), with higher PSNR on every image.

\begin{table}[htbp]
\centering
\small
\setlength{\tabcolsep}{4pt}
\caption{Fixed data-term variance multipliers in motion deblurring: mean per-trajectory PSNR (dB) on the ablation set. For PnP-DM, $v$ scales the observation-noise variance; for DAPS, it scales the data-fitting variance $\tau_{\rm D}^2$. $v=1$ is the published sampler; UAL gives 26.641 dB for PnP-DM and 28.631 dB for DAPS.}
\label{tab:scalar-motion}
\begin{tabular}{@{}lrrrrrrrrr@{}}
\toprule
$v$ & 1 & 3 & 6 & 12 & 24 & 48 & 96 & 192 & 384 \\
\midrule
PnP-DM & 17.872 & 20.429 & 21.652 & 22.569 & 23.243 & 23.369 & 23.210 & 22.503 & 21.419 \\
DAPS & 18.441 & 21.098 & 22.431 & 23.448 & 24.113 & 24.413 & 24.360 & 23.965 & 23.245 \\
\bottomrule
\end{tabular}
\end{table}

\begin{table}[htbp]
\centering
\small
\caption{Gaussian-deblurring covariance-structure control. T/M follows Table~\ref{tab:exp-deblurring}; PSNR is in dB, and CRPS and LVS are scaled by $10^3$. Lower ensemble scores are better.}
\label{tab:exp-ablations-full}
\begin{tabular}{@{}lrrrrr@{}}
\toprule
Likelihood & PSNR T & PSNR M & CRPS & Energy & LVS \\
\midrule
Published & 26.944 & 29.388 & 32.53 & 21.327 & 15.22 \\
Isotropic & 27.020 & 29.472 & 32.42 & 21.078 & 15.12 \\
UAL & \textbf{27.475} & \textbf{30.259} & \textbf{29.81} & \textbf{19.052} & \textbf{14.18} \\
\bottomrule
\end{tabular}
\end{table}

For PnP-DM in motion deblurring, including the log determinant improves all seven metric means over the quadratic-only control (Table~\ref{tab:normalization-motion}).
 This control uses $P=256$ and the ablation set, as the covariance-structure comparison. We retain $\tfrac12d(x)^\T S(x)^{-1}d(x)$, recomputing and differentiating its covariance at every candidate. This is the penalized profile energy described in Appendix~\ref{app:schur}. UAL has higher mean per-trajectory PSNR on 15 of 16 images, with a paired advantage of $0.0257\pm0.0066$ dB (mean $\pm$ SE across images).

\begin{table}[!ht]
\centering
\small
\caption[Likelihood normalization in motion deblurring]{Likelihood normalization in motion deblurring with PnP-DM at $P=256$. Values are means $\pm$ one SE across the ablation set. Quadratic only omits the log determinant while retaining covariance recomputation and differentiation. T/M follows Table~\ref{tab:exp-deblurring}.}
\label{tab:normalization-motion}
\begin{tabular}{@{}lrrr@{}}
\toprule
Metric & Published & Quadratic only & UAL \\
\midrule
PSNR T (dB) $\uparrow$ & $17.8723\pm0.7269$ & $26.6157\pm0.3040$ & $26.6414\pm0.3013$ \\
PSNR M (dB) $\uparrow$ & $18.7901\pm0.7827$ & $29.5099\pm0.3551$ & $29.5335\pm0.3564$ \\
SSIM T $\uparrow$ & $0.30950\pm0.02466$ & $0.55399\pm0.00805$ & $0.55412\pm0.00806$ \\
SSIM M $\uparrow$ & $0.43340\pm0.03894$ & $0.76189\pm0.01441$ & $0.76191\pm0.01444$ \\
CRPS ($\times10^3$) $\downarrow$ & $131.533\pm19.396$ & $31.534\pm1.491$ & $31.489\pm1.496$ \\
Energy $\downarrow$ & $83.874\pm9.441$ & $20.406\pm0.882$ & $20.351\pm0.884$ \\
LVS ($\times10^3$) $\downarrow$ & $54.395\pm7.343$ & $14.790\pm0.572$ & $14.766\pm0.564$ \\
\bottomrule
\end{tabular}
\end{table}

For Fast Diffusion EM in motion deblurring, UAL improves every reported metric on all 50 images relative to covariance weighting alone (Table~\ref{tab:fdem-control}).

\begin{table}[!ht]
\centering
\small
\caption[Fast Diffusion EM covariance control]{Motion deblurring with Fast Diffusion EM on 50 images with 12 paired trajectories per image and $P=256$. All three adapted arms use the population-mean kernel and propagated covariance; covariance weighting holds the covariance fixed during differentiation and omits the log determinant, and the quadratic-only arm recomputes and differentiates the covariance but omits the log determinant. T/M follows Table~\ref{tab:exp-deblurring}.}
\label{tab:fdem-control}
\begin{tabular}{@{}lrrrr@{}}
\toprule
Metric & Published & Covariance weighting & Quadratic only & UAL \\
\midrule
PSNR T (dB) $\uparrow$ & 22.76 & 23.80 & 28.49 & 29.10 \\
PSNR M (dB) $\uparrow$ & 23.27 & 25.30 & 30.64 & 30.72 \\
SSIM T $\uparrow$ & .5236 & .5663 & .7537 & .7635 \\
SSIM M $\uparrow$ & .5525 & .6243 & .8131 & .8108 \\
CRPS $\downarrow$ & .07789 & .05472 & .02731 & .02708 \\
Energy $\downarrow$ & 51.865 & 37.077 & 18.360 & 18.332 \\
LVS ($\times10^3$) $\downarrow$ & 24.271 & 17.248 & 9.449 & 9.633 \\
\bottomrule
\end{tabular}
\end{table}

\subsection{\texorpdfstring{Covariance spectra at final reconstructions}{Covariance spectra at final reconstructions}}
\label{app:covariance-diagnostics}

We evaluate UAL's covariance at the final Published, Isotropic, and UAL reconstructions, using the same parameter mean and covariance for all three. 
With $W=\Sigma_e^{-1/2}U$ and $G=I_r+W^\T W$, we report the average added variance after noise whitening, $c(x)=(\tr G-r)/m$, and the two leading squared singular values $s_1^2$ and $s_2^2$ of $W$.
The trace-matched variance multiplier is $1+c(x)$.

Each quantity is evaluated separately for each trajectory and averaged within
an image; the table reports means and standard errors across images.
Singular values are ordered within each trajectory before aggregation.

\begin{table}[htbp]
\centering
\small
\caption{Final-state covariance diagnostics, reported as mean $\pm$ standard error across images. The leading squared singular values are scaled by $10^{-4}$.}
\label{tab:covariance-diagnostics}
\begin{tabular}{@{}llrrr@{}}
\toprule
Task & Candidate & $c(x)$ & $s_1^2/10^4$ & $s_2^2/10^4$ \\
\midrule
Gaussian blur & Published & $0.189\pm0.015$ & $2.617\pm0.266$ & $1.098\pm0.048$ \\
 & Isotropic & $0.185\pm0.014$ & $2.553\pm0.248$ & $1.088\pm0.047$ \\
 & UAL & $0.197\pm0.010$ & $2.516\pm0.134$ & $1.357\pm0.065$ \\
\midrule
Motion blur & Published & $20.085\pm3.532$ & $176.389\pm44.334$ & $57.841\pm8.578$ \\
 & Isotropic & $3.244\pm0.328$ & $34.666\pm4.084$ & $12.922\pm1.438$ \\
 & UAL & $4.626\pm0.370$ & $38.296\pm3.533$ & $19.090\pm1.942$ \\
\bottomrule
\end{tabular}
\end{table}

The Gaussian covariance concentrates its added variance in two directions, so a small average variance increase can accompany large directional variances. In motion deblurring, the added covariance trace differs across the final candidate reconstructions. These diagnostics use measurement-noise whitening; the reconstruction-risk Gram matrix $\Lambda$ in Appendix~\ref{app:linear-risk} uses predictive-covariance whitening.

\subsection{Covariance scale and matched uncertainty}
\begin{figure}[!ht]
\centering
\includegraphics[width=\linewidth]{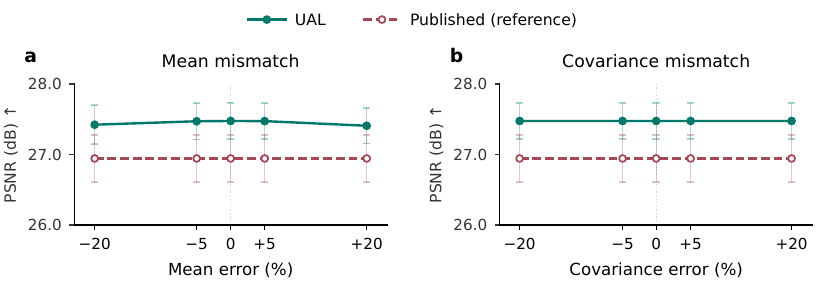}
\caption{PnP-DM sensitivity to nuisance moments in Gaussian deblurring. (a) Mean-width error with correct covariance. (b) Covariance-scale error with correct mean. Observations stay fixed; zero denotes the correct moment. Curves show mean per-trajectory PSNR $\pm$ one SE across the ablation set. The dashed Published reference uses the correct mean in both panels.  Mean PSNR changes little over the tested ranges.}
\label{fig:covariance-scale}
\end{figure}
\paragraph{\texorpdfstring{Mean and covariance mismatch.}{Mean and covariance mismatch.}}
With observations fixed and covariance correct, shifting the mean width by approximately $\pm5\%$ and $\pm20\%$ gives mean per-trajectory PSNR from 27.408 to 27.475 dB (Figure~\ref{fig:covariance-scale}a).
CRPS is 0.03099 at mean width 2.4 pixels and 0.02981 at the correct mean of 3 pixels.
We next perturb the covariance used by UAL while keeping the generating covariance and observations fixed. Gaussian deblurring uses the reference $\Sigma_{\psi,\rm ref}=0.25I_2$ in squared-pixel units. Each perturbed covariance is a scalar multiple of its task's reference. Table~\ref{tab:covariance-scale} reports paired changes from UAL with zero covariance error. We also test approximately doubled nuisance covariance.
Across the nonzero Gaussian covariance settings, mean per-trajectory PSNR remains 27.475 dB to three decimals. For motion deblurring, removing the nuisance covariance gives the published sampler's PSNR of 17.872 dB; doubling it gives 26.644 dB, compared with 26.641 dB at the reference covariance.

Over the approximately $\pm20\%$ covariance-error range, mean per-trajectory PSNR changes by less than $0.001$ dB in Gaussian deblurring and at most $0.0641$ dB in motion deblurring. The largest absolute relative CRPS changes are $0.0011\%$ and $0.0572\%$, respectively. One trajectory accounts for most of the motion PSNR change; all trajectories remain in the paired summaries.

\begin{table}[htbp]
\centering
\small
\setlength{\tabcolsep}{5pt}
\caption{Covariance-scale mismatch with fixed observations. Each task uses the ablation set. Signed changes (perturbed minus UAL at zero covariance error) are reported as mean $\pm$ SE across images. PSNR is in dB\@. Relative CRPS changes divide the paired change and its SE by the reference mean CRPS\@.}
\label{tab:covariance-scale}
\begin{tabular}{@{}llrrr@{}}
\toprule
Task & Covariance error & $\Delta$PSNR T & $\Delta$PSNR M & $\Delta$CRPS (\%) \\
\midrule
Gaussian & $-20.1\%$ & $+0.0002\pm0.0002$ & $+0.0001\pm0.0002$ & $-0.0003\pm0.0008$ \\
Gaussian & $-4.9\%$ & $+0.0002\pm0.0002$ & $+0.0002\pm0.0001$ & $-0.0006\pm0.0007$ \\
Gaussian & $+5.1\%$ & $+0.0002\pm0.0001$ & $+0.0002\pm0.0001$ & $-0.0010\pm0.0006$ \\
Gaussian & $+19.9\%$ & $+0.0001\pm0.0001$ & $+0.0002\pm0.0002$ & $-0.0011\pm0.0008$ \\
Gaussian & $+99.9\%$ & $+0.0001\pm0.0001$ & $+0.0002\pm0.0002$ & $-0.0009\pm0.0010$ \\
\midrule
Motion & $-20.0\%$ & $+0.0609\pm0.0633$ & $+0.0257\pm0.0291$ & $+0.0572\pm0.0392$ \\
Motion & $-5.0\%$ & $-0.0005\pm0.0002$ & $-0.0006\pm0.0002$ & $+0.0040\pm0.0024$ \\
Motion & $+5.0\%$ & $+0.0004\pm0.0002$ & $+0.0005\pm0.0003$ & $-0.0043\pm0.0026$ \\
Motion & $+20.0\%$ & $+0.0641\pm0.0634$ & $+0.0308\pm0.0292$ & $+0.0252\pm0.0366$ \\
Motion & $+100.0\%$ & $+0.0030\pm0.0034$ & $+0.0108\pm0.0078$ & $-0.0495\pm0.0404$ \\
\bottomrule
\end{tabular}
\end{table}

\paragraph{Matched uncertainty levels.}
\label{app:matched-uncertainty}
We next vary the generating width uncertainty and the uncertainty supplied
to UAL together, using $\mu=(3,3)$ and $\Sigma_\psi=\tau^2I_2$.
Image identities, width percentiles, and noise realizations are coupled across levels; the observations change with the blur widths.
Relative to Published, UAL's mean per-trajectory PSNR gain is $0.043\pm0.067$ dB at $\tau=0.25$.
It grows to $1.746\pm0.397$ dB at $\tau=1.5$, while the reduction in
per-trajectory MSE grows from $0.000099$ to $0.007808$
(Table~\ref{tab:matched-uncertainty}).
At the largest tested uncertainty, UAL reduces average
per-trajectory MSE by approximately 38\%.
Table~\ref{tab:matched-uncertainty} also lists the trace-matched isotropic control of (\ref{eq:exp-isotropic}). UAL gives lower mean MSE, higher mean per-trajectory PSNR, and lower CRPS than isotropic inflation at every tested level.

The comparison shows how the empirical benefit changes when generating uncertainty and the uncertainty used by UAL vary together; the conditional crossover in Appendix~\ref{app:linear-risk} instead holds the nuisance model \(q\) fixed and varies the nuisance realization.

\begin{table}[!ht]
\centering
\small
\setlength{\tabcolsep}{2.5pt}
\caption{Matched generating and UAL width uncertainty, on the ablation set. MSE and PSNR (dB) are averaged per trajectory; CRPS is scaled by $10^3$. Isotropic is the trace-matched control of (\ref{eq:exp-isotropic}), run at each level.}
\label{tab:matched-uncertainty}
\begin{tabular*}{\linewidth}{@{\extracolsep{\fill}}rrrrrrrrrr@{}}
\toprule
 & \multicolumn{3}{c}{MSE $\downarrow$} & \multicolumn{3}{c}{PSNR $\uparrow$} & \multicolumn{3}{c}{CRPS $\downarrow$} \\
\cmidrule(lr){2-4}\cmidrule(lr){5-7}\cmidrule(lr){8-10}
$\tau$ (pixels) & Published & Isotropic & UAL & Published & Isotropic & UAL & Published & Isotropic & UAL \\
\midrule
.25 & .007301 & .007312 & .007202 & 27.508 & 27.502 & 27.551 & 29.93 & 29.97 & 29.23 \\
.50 & .008459 & .008263 & .007348 & 26.944 & 27.020 & 27.475 & 32.53 & 32.42 & 29.81 \\
1.00 & .014066 & .010826 & .008800 & 25.368 & 26.018 & 26.818 & 40.95 & 38.24 & 33.26 \\
1.50 & .020463 & .012810 & .012655 & 24.043 & 25.318 & 25.789 & 49.58 & 43.24 & 39.90 \\
\bottomrule
\end{tabular*}
\end{table}

\subsection{Full Gaussian-width results}
\label{app:width-results}
Across the three Gaussian-width settings, UAL gives higher mean per-trajectory PSNR and lower CRPS than Published for all four samplers (Table~\ref{tab:gaussian-widths}). These comparisons use all 50 images and 12 paired trajectories per image; Table~\ref{tab:exp-deblurring} gives all seven scores for the main setting; Table~\ref{tab:width-full-second} gives them for the other two.

\begin{table}[!ht]
\centering
\small
\setlength{\tabcolsep}{3.5pt}
\caption{Gaussian-width effects on reconstruction and ensemble prediction on 50 images, with 12 paired trajectories per image. Entries are per-trajectory PSNR (dB) and CRPS ($\times10^3$) for the published sampler and UAL; $(\bar s,\tau)$ is in pixels. Bold marks better UAL values at the displayed precision.}
\label{tab:gaussian-widths}
\begin{tabular}{@{}llrrrrrrrr@{}}
\toprule
 & & \multicolumn{2}{c}{DAPS} & \multicolumn{2}{c}{PnP-DM} & \multicolumn{2}{c}{DiffPIR} & \multicolumn{2}{c}{BlindDPS} \\
\cmidrule(lr){3-4}\cmidrule(lr){5-6}\cmidrule(lr){7-8}\cmidrule(lr){9-10}
$(\bar s,\tau)$ & Version & PSNR $\uparrow$ & CRPS $\downarrow$ & PSNR $\uparrow$ & CRPS $\downarrow$ & PSNR $\uparrow$ & CRPS $\downarrow$ & PSNR $\uparrow$ & CRPS $\downarrow$ \\
\midrule
 & Published & 29.49 & 28.71 & 26.97 & 32.09 & 29.13 & 29.48 & 28.31 & 26.98 \\
\rowcolor{black!6}
\cellcolor{white}\raisebox{0.48\normalbaselineskip}[0pt][0pt]{$(3,0.5)$} & UAL & \textbf{29.88} & \textbf{27.96} & \textbf{27.50} & \textbf{29.56} & \textbf{30.12} & \textbf{25.83} & \textbf{29.21} & \textbf{24.76} \\
\addlinespace[2pt]
 & Published & 28.16 & 32.89 & 26.35 & 35.58 & 27.52 & 32.67 & 26.96 & 32.38 \\
\rowcolor{black!6}
\cellcolor{white}\raisebox{0.48\normalbaselineskip}[0pt][0pt]{$(5,0.5)$} & UAL & \textbf{28.65} & \textbf{32.22} & \textbf{26.84} & \textbf{33.93} & \textbf{28.60} & \textbf{28.64} & \textbf{27.37} & \textbf{29.79} \\
\addlinespace[2pt]
 & Published & 28.05 & 35.42 & 25.80 & 38.56 & 26.70 & 43.55 & 28.43 & 27.01 \\
\rowcolor{black!6}
\cellcolor{white}\raisebox{0.48\normalbaselineskip}[0pt][0pt]{$(3,1)$} & UAL & \textbf{29.56} & \textbf{30.13} & \textbf{27.16} & \textbf{31.03} & \textbf{29.63} & \textbf{29.07} & \textbf{29.00} & \textbf{25.73} \\
\bottomrule
\end{tabular}
\end{table}

\begin{table}[!ht]
\centering
\small
\setlength{\tabcolsep}{4pt}
\caption{Full Gaussian-width results for $(5,0.5)$ and $(3,1)$ pixels on 50 images, with 12 paired trajectories per image. T/M follows Table~\ref{tab:exp-deblurring}; PSNR is in dB and LVS is scaled by $10^3$. Bold marks better UAL values at the displayed precision.}
\label{tab:width-full-second}
\begin{tabular}{@{}llrrrrr@{}}
\toprule
Sampler & Version & PSNR T/M $\uparrow$ & SSIM T/M $\uparrow$ & CRPS $\downarrow$ & Energy $\downarrow$ & LVS $\downarrow$ \\
\midrule
\multicolumn{7}{l}{$(5,0.5)$ pixels} \\
DAPS & Published & 28.16 / 29.32 & .6888 / .7685 & .03289 & 22.999 & 12.581 \\
\rowcolor{black!6}
 & UAL & \textbf{28.65} / \textbf{29.67} & \textbf{.6949} / \textbf{.7702} & \textbf{.03222} & \textbf{22.512} & 12.631 \\
\addlinespace[1pt]
PnP-DM & Published & 26.35 / 28.55 & .5325 / .7290 & .03558 & 23.459 & 16.258 \\
\rowcolor{black!6}
 & UAL & \textbf{26.84} / \textbf{29.12} & \textbf{.5366} / \textbf{.7338} & \textbf{.03393} & \textbf{22.008} & \textbf{15.759} \\
\addlinespace[1pt]
DiffPIR & Published & 27.52 / 29.11 & .7028 / .7712 & .03267 & 22.846 & 11.400 \\
\rowcolor{black!6}
 & UAL & \textbf{28.60} / \textbf{30.07} & \textbf{.7278} / \textbf{.7891} & \textbf{.02864} & \textbf{20.396} & \textbf{10.456} \\
\addlinespace[1pt]
BlindDPS & Published & 26.96 / 28.78 & .6732 / .7436 & .03238 & 23.167 & 11.609 \\
\rowcolor{black!6}
 & UAL & \textbf{27.37} / \textbf{29.45} & \textbf{.6955} / \textbf{.7714} & \textbf{.02979} & \textbf{21.120} & \textbf{10.365} \\
\addlinespace[1pt]
\midrule
\multicolumn{7}{l}{$(3,1)$ pixels} \\
DAPS & Published & 28.05 / 29.02 & .7113 / .7839 & .03542 & 25.410 & 13.712 \\
\rowcolor{black!6}
 & UAL & \textbf{29.56} / \textbf{30.56} & \textbf{.7347} / \textbf{.8041} & \textbf{.03013} & \textbf{20.679} & \textbf{11.525} \\
\addlinespace[1pt]
PnP-DM & Published & 25.80 / 28.13 & .5388 / .7370 & .03856 & 25.832 & 18.788 \\
\rowcolor{black!6}
 & UAL & \textbf{27.16} / \textbf{29.96} & \textbf{.5639} / \textbf{.7683} & \textbf{.03103} & \textbf{19.965} & \textbf{14.864} \\
\addlinespace[1pt]
DiffPIR & Published & 26.70 / 27.78 & .7114 / .7667 & .04355 & 31.323 & 18.119 \\
\rowcolor{black!6}
 & UAL & \textbf{29.63} / \textbf{30.67} & \textbf{.7684} / \textbf{.8130} & \textbf{.02907} & \textbf{20.319} & \textbf{10.210} \\
\addlinespace[1pt]
BlindDPS & Published & 28.43 / 30.37 & .7381 / .8015 & .02701 & 19.286 & 9.725 \\
\rowcolor{black!6}
 & UAL & \textbf{29.00} / \textbf{30.88} & \textbf{.7622} / \textbf{.8202} & \textbf{.02573} & \textbf{18.312} & \textbf{9.155} \\
\addlinespace[1pt]
\bottomrule
\end{tabular}
\end{table}

\subsection{\texorpdfstring{Covariance rank and computational cost}{Covariance rank and computational cost}}

\label{app:runtime}

Table~\ref{tab:runtime} reports mean wall time per trajectory on NVIDIA L40S GPUs, with one reconstruction job per GPU\@. DAPS deblurring times cover the sampling call, excluding setup. DiffPIR includes process initialization and output writing; DRUS excludes model loading, singular value decomposition, and covariance preparation.
 For Figure~\ref{fig:rank-tradeoff}, reconstruction scores use the ablation set.
Table~\ref{tab:motion-rank-comparison} compares Published and UAL at $P=256$ and $P=2048$ on the ablation set to assess sensitivity to rank. At both ranks, UAL improves mean per-trajectory PSNR for the three non-blind samplers and lowers the mean CRPS, energy, and local variogram scores for all four samplers.

\begin{table}[!ht]
\centering
\small
\setlength{\tabcolsep}{3pt}
\caption[Motion-deblurring covariance ranks]{Motion deblurring with the published samplers and UAL\@. All treatments use the ablation set. T/M follows Table~\ref{tab:exp-deblurring}; PSNR is in dB\@. CRPS, energy, and LVS are fair ensemble scores, with LVS scaled by $10^3$.}
\label{tab:motion-rank-comparison}
\begin{tabular*}{\linewidth}{@{\extracolsep{\fill}}llrrrrr@{}}
\toprule
Sampler & Treatment & PSNR T/M $\uparrow$ & SSIM T/M $\uparrow$ & CRPS $\downarrow$ & Energy $\downarrow$ & LVS $\downarrow$ \\
\midrule
DAPS & Published & 18.44 / 18.89 & .3677 / .4190 & 0.13913 & 88.578 & 52.203 \\
 & UAL, $P=256$ & 28.63 / 29.34 & .7030 / .7660 & 0.03572 & 24.046 & 13.079 \\
 & UAL, $P=2048$ & 28.56 / 29.29 & .7002 / .7659 & 0.03568 & 23.962 & 12.990 \\
\addlinespace[2pt]
PnP-DM & Published & 17.87 / 18.79 & .3095 / .4334 & 0.13153 & 83.874 & 54.395 \\
 & UAL, $P=256$ & 26.64 / 29.53 & .5541 / .7619 & 0.03149 & 20.351 & 14.766 \\
 & UAL, $P=2048$ & 26.25 / 29.31 & .5420 / .7517 & 0.03209 & 20.788 & 14.997 \\
\addlinespace[2pt]
DiffPIR & Published & 11.72 / 12.50 & .0853 / .1057 & 0.32936 & 165.774 & 170.151 \\
 & UAL, $P=256$ & 15.19 / 15.87 & .1681 / .1930 & 0.22664 & 120.672 & 90.039 \\
 & UAL, $P=2048$ & 16.78 / 17.52 & .2184 / .2507 & 0.18415 & 101.031 & 68.612 \\
\addlinespace[2pt]
BlindDPS & Published & 22.50 / 23.61 & .5747 / .6418 & 0.06109 & 44.357 & 17.561 \\
 & UAL, $P=256$ & 22.20 / 23.62 & .5593 / .6391 & 0.06048 & 42.963 & 16.822 \\
 & UAL, $P=2048$ & 22.38 / 23.78 & .5616 / .6392 & 0.05955 & 42.368 & 16.759 \\
\bottomrule
\end{tabular*}
\end{table}

\begin{table}[htbp]
\centering
\small
\caption{Mean wall time per trajectory on NVIDIA L40S GPUs with one reconstruction job per GPU.} 
\label{tab:runtime}
\begin{tabular}{@{}lrr@{}}
\toprule
Sampler / task & Published (s) & UAL (s) \\
\midrule
\multicolumn{3}{l}{Gaussian deblurring} \\
PnP-DM & 77.3 & 76.9 \\
DAPS & 39.4 & 116.6 \\
DiffPIR & 6.0 & 6.1 \\
BlindDPS & 129.6 & 371.8 \\
\midrule
\multicolumn{3}{l}{Motion deblurring} \\
PnP-DM & 31.8 & 89.5 \\
DAPS & 21.8 & 155.4 \\
DiffPIR & 6.1 & 6.4 \\
BlindDPS & 126.4 & 70.7 \\
\midrule
\multicolumn{3}{l}{Imaging applications} \\
DAPS (CT) & 429.8 & 436.2 \\
Ambient DPS (MRI) & 52.9 & 103.2 \\
DRUS (ultrasound) & 74.3 & 73.9 \\
PnP-DM (microscopy) & 30.5 & 83.5 \\
\bottomrule
\end{tabular}
\end{table}

\end{document}